\documentclass[11pt]{article}

\ifdefined\ISSUBMISSION
\newcommand{\cond}[2]{#1}
\else
\newcommand{\cond}[2]{#2}
\fi

\usepackage[colorlinks=true,linkcolor=blue,urlcolor=black,citecolor=blue,breaklinks]{hyperref}
\usepackage{color}
\usepackage{enumerate}
\usepackage{graphicx}
\usepackage{algorithm, algcompatible}
\usepackage{varwidth}
\usepackage{mathrsfs}
\usepackage{mathtools}
\usepackage[labelfont=bf,font={small}]{caption}
\usepackage[labelfont=bf]{subcaption}
\usepackage{wrapfig}
\usepackage{overpic}
\usepackage{multirow}

\usepackage{natbib}
\bibpunct{(}{)}{;}{a}{,}{,}

\renewcommand{\geq}{\geqslant}

\renewcommand{\leq}{\leqslant}
\renewcommand{\preceq}{\preccurlyeq}
\renewcommand{\succeq}{\succcurlyeq}

\usepackage{fullpage}
\usepackage{authblk}

\usepackage{amsmath,amsthm,amssymb,colonequals,etoolbox}
\usepackage{thmtools}
\usepackage{thm-restate}

\newtheorem{theorem}{Theorem}[section]
\newtheorem{assumption}[theorem]{Assumption}
\newtheorem{proposition}[theorem]{Proposition}
\newtheorem{lemma}[theorem]{Lemma}
\newtheorem{definition}[theorem]{Definition}
\newtheorem{corollary}[theorem]{Corollary}

\DeclareMathOperator*{\Tr}{\mathrm{tr}}

\newcommand{\R}{\ensuremath{\mathbb{R}}}

\newcommand{\N}{\ensuremath{\mathbb{N}}}

\newcommand{\norm}[1]{\lVert #1 \rVert}
\newcommand{\bignorm}[1]{\left\lVert #1 \right\rVert}

\newcommand{\ip}[2]{\ensuremath{\langle #1, #2 \rangle}}

\newcommand{\E}{\mathbb{E}}
\newcommand{\abs}[1]{\ensuremath{| #1 |}}
\newcommand{\bigabs}[1]{\ensuremath{\left| #1 \right|}}
\newcommand{\floor}[1]{\lfloor #1 \rfloor}

\renewcommand{\vec}{\mathrm{vec}}

\newcommand{\T}{\mathsf{T}}

\newcommand{\calC}{\mathcal{C}}
\newcommand{\calD}{\mathcal{D}}

\newcommand{\calF}{\mathcal{F}}

\renewcommand{\vec}{\mathbf}

\numberwithin{equation}{section}

\title{Regret Bounds for Adaptive Nonlinear Control}

\author[1]{Nicholas M.\ Boffi\thanks{
Both authors contributed equally.}}
\author[2]{Stephen Tu$^*$}
\author[3,2]{Jean-Jacques E.\ Slotine}

\affil[1]{John A.\ Paulson School of Engineering and Applied Sciences, Harvard University}
\affil[2]{Google Brain Robotics}
\affil[3]{Nonlinear Systems Laboratory, Massachusetts Institute of Technology}

\date{\today}

\usepackage{etoc}

\begin{document}

\maketitle


\begin{abstract}%
We study the problem of adaptively controlling a known discrete-time nonlinear system
subject to unmodeled disturbances. 
We prove the first finite-time regret
bounds for adaptive nonlinear control with matched uncertainty in the stochastic setting, 
showing that the regret
suffered by certainty equivalence adaptive control, compared to an oracle controller with perfect knowledge
of the unmodeled disturbances, is upper bounded by $\widetilde{O}(\sqrt{T})$ in expectation.
Furthermore, we show that when the input is subject to a $k$ timestep delay, the regret
degrades to $\widetilde{O}(k \sqrt{T})$.
Our analysis draws connections between classical stability notions in nonlinear control theory (Lyapunov stability and contraction theory) and modern regret analysis from online
convex optimization.
The use of stability theory allows us to analyze the challenging
infinite-horizon single trajectory setting.
\end{abstract}

\section{Introduction}
\label{sec:intro}

The goal of adaptive nonlinear control
\citep{slotine91book,ioannou96book,fradkov99book}
is to control a continuous-time dynamical system in the presence of unknown dynamics; it is the study of concurrent learning and control of dynamical systems. There is a rich body of literature analyzing the stability and convergence
properties of classical adaptive control algorithms. Under suitable assumptions (e.g.,\ Lyapunov stability of the known part of the system), typical results guarantee asymptotic convergence of the unknown system to a fixed point or desired trajectory.

On the other hand, due to recent successes of reinforcement learning (RL) in
the control of physical systems~\citep{yang19legged,openai2019rubiks,hwangbo19legged,williams17mpc,levine16endtoend}, there has been a flurry of research in online RL algorithms for continuous control. In contrast
to the classical setting of adaptive nonlinear control, online RL algorithms operate in discrete-time, and often
come with finite-time regret bounds~\citep{wang19optimism,kakade20nonlinear,jin20linear,cao20adaptive,cai20oppo,agarwal20boosting}. 
These bounds provide a quantitative rate at which the control performance of the online
algorithm approaches the performance of an oracle equipped with hindsight knowledge of the uncertainty.

In this work, we revisit the analysis of adaptive nonlinear control algorithms through the lens of modern reinforcement learning. 
Specifically, we show how to systematically port 
matched uncertainty adaptive control algorithms to discrete-time, and we use
the machinery of online convex optimization~\citep{hazan16oco} to prove finite-time
regret bounds. Our analysis uses the notions of contraction and
incremental stability~\citep{lohmiller98contraction,angeli02incrementalstability}
to draw a connection between control regret, the quantity we are interested in, and function prediction regret, the quantity online convex optimization enables us to bound.

We present two main sets of results.
First, we provide a discrete-time analysis of \emph{velocity gradient}
adaptation~\citep{fradkov99book}, a broad framework which encompasses e.g.,\ classic adaptive sliding control~\citep{slotine86sliding}.
We prove that in the deterministic setting, if a Lyapunov function describing the nominal system is strongly convex in the state, then the corresponding velocity gradient algorithm achieves constant regret with respect to a baseline controller having full knowledge of the system.
Our second line of results considers the use of online least-squares \emph{gradient based optimization} for the parameters. Under an incremental input-to-state stability assumption,
we prove $\widetilde{O}(\sqrt{T})$ regret bounds in the presence of stochastic process noise.
We further show that when the input is delayed by $k$ timesteps, the regret degrades to $\widetilde{O}(k\sqrt{T})$.
Importantly, our bounds hold for the challenging single trajectory infinite horizon setting, 
rather than the finite-horizon episodic setting more frequently studied in reinforcement learning.
We conclude with simulations showing the efficacy of
our proposed discrete-time algorithms 
in quickly adapting to unmodeled disturbances.
\cond{Proofs and more details can be found in the full version of the paper~\citep{boffi20fullpaper}.}{}
\section{Related Work}
\label{sec:related}

There has been a renewed focus on the continuous state and action space
setting in the reinforcement learning (RL) literature.
The most well-studied problem for continuous control in RL
is the 
Linear Quadratic Regulator (LQR) problem with unknown dynamics.
For LQR, both upper and lower bounds achieving $\sqrt{T}$ regret
are available~\citep{abbasiyadkori11regret,agarwal19onlinecontrol,mania19CE,cohen19sqrtT,simchowitz20naiveexploration,hazan20nonstochastic}, for stochastic and adversarial noise processes.
Furthermore, in certain settings it is even possible to obtain
logarithmic regret~\citep{agarwal19logarithmic,cassel20logarithmic,foster20logarithmic}.

Results that extend beyond the classic LQR problem are
less complete, but are rapidly growing.
Recently, \cite{kakade20nonlinear}
showed $\sqrt{T}$ regret bounds in the finite horizon episodic setting
for dynamics of the form $x_{t+1} = A \phi(x_t, u_t) + w_t$
where $A$ is an unknown operator and $\phi$ is a known feature map, though their algorithm
is generally not tractable to implement.
\cite{mania20nonlinear} show how to actively recover the parameter matrix $A$
using trajectory optimization.
\cite{azizzadenesheli18exploration,jin20linear,yang20linearmdp,zanette20frequentist} show $\sqrt{T}$ regret bounds for \emph{linear MDPs},
which implies that the associated $Q$-function is linear after a known feature transformation.
\cite{wang19optimism} extend this model to allow for generalized linear model $Q$-functions.
Unlike the stability notions considered in this work, we are unaware of any algorithmic method of verifying the
linear MDP assumption.
Furthermore, the aforementioned regret bounds
are for the finite-horizon episodic setting; we study the infinite-horizon single trajectory setting 
without resets.

Very few results categorizing regret bounds for adaptive nonlinear control exist; one recent example is~\citet{gaudio19connections},
who highlight that simple model reference adaptive controllers obtain constant regret in the continuous-time deterministic setting. 
In contrast, our work simultaneously tackles the
issues of more general models, discrete-time systems, and stochastic noise.
We note that several authors have ported various adaptive controllers
into discrete-time~\citep{pieper96discrete,bartolini95discrete,loukianov18discrete,munoz00discrete,kanellakopoulos94discrete,ordonez06discrete}.
These results, however, are mostly concerned with asymptotic stability of the closed-loop system,
as opposed to finite-time regret bounds.

\section{Problem Statement}
\label{sec:setup}

In this work, we focus on the following discrete-time\footnote{
Discrete-time systems may arise as a modeling decision, or due to finite sampling rates for the input, e.g.,\ a continuous-time controller implemented on a computer. In \cond{the full paper}{Appendix~\ref{sec:app:c2d}}, we study the latter situation,
giving bounds on the rate for which a continuous-time controller must be sampled
such that discrete-time closed-loop stability holds.}, time-varying, and nonlinear dynamical system with linearly parameterized unknown in the matched uncertainty setting:
\begin{align}
    x_{t+1} = f(x_t, t) + B(x_t, t)(u_t - Y(x_t, t)\alpha) + w_t \:. \label{eq:system}
\end{align}
Here $x_t \in \R^n$, $u_t \in \R^d$, $f:\R^n\times\N\rightarrow \R^n$ is a known nominal dynamics model, $B:\R^n\times\N \rightarrow \R^{n\times d}$ is a known input matrix, $Y: \R^n\times\N \rightarrow \R^{d\times p}$ is a matrix of known basis functions, and $\alpha \in \R^p$ is a vector of unknown parameters. The sequence of noise vectors $\{w_t\} \subseteq \R^n$ is assumed to satisfy the distributional requirements $\E[w_t] = 0$, $\norm{w_t} \leq W$ almost surely,
and that $w_s$ is independent of $w_t$ for all $s \neq t$.
We further assume that $\alpha \in \calC := \{ \alpha \in \R^p : \norm{\alpha} \leq D \}$,
and that an upper bound for $D$ is known.
Without loss of generality, we set the origin to be a fixed-point of the nominal
dynamics, so that $f(0, t) = 0$ for all $t$. Because the nominal dynamics is time-varying, this formalism captures the classic setting of nonlinear adaptive control, which considers the problem of tracking a time-varying desired trajectory $x^d_t$\footnote{To see this, consider a system $y_{t+1} = g(y_t, t) + B(y_t, t)\left(u_t - Y(y_t, t)\alpha\right)$
and a desired trajectory $y_t^d$ satisfying $y_{t+1}^d = g(y_t^d, t)$.
Define the new variable $x_t := y_t - y_t^d$. Then $x_{t+1} = g(x_t + y_t^d, t) - g(y_t^d, t) + B(x_t + y_t^d, t)\left(u_t - Y(x_t + y_t^d, t)\alpha\right)$, so that the nominal dynamics $f(x_t, t) = g(x_t + y_t^d, t) - g(y_t^d, t)$ satisfies $f(0, t) = 0$ for all $t$. If the original $y_t$ system is non-autonomous, the time-dependent desired trajectory will introduce a time-dependent nominal dynamics in the $x_t$ system.}. 

We study \emph{certainty equivalence} controllers.
In particular, we maintain a parameter estimate $\hat{\alpha}_t \in \calC$
and play the input $u_t = Y(x_t, t) \hat{\alpha}_t$. Our goal is to design a learning algorithm that updates $\hat{\alpha}_t$ to cancel the unknown and which provides a guarantee of fast convergence to the performance of an ideal \emph{comparator}. The comparator that we will study is an oracle that plays the ideal control $u_t = Y(x_t, t) \alpha$ at every timestep, leading to the dynamics $x_{t+1} = f(x_t, t) + w_t$.
To measure the rate of convergence to this comparator,
we study the following notion of \emph{control regret}:
\begin{align}
    \mathsf{Regret}(T) := \E_{\{w_t\}}\left[ \sum_{t=0}^{T-1} \norm{x_t^a}^2 - \norm{x_t^c}^2 \right] \:.
\end{align}
Here, the trajectory $\{x_t^a\}$ is generated by an adaptive control algorithm, while the trajectory $\{x_t^c\}$ is generated by the oracle with access to the true parameters $\alpha$. Our notation for $x_t^a$ and $x_t^c$ suppresses the dependence of the trajectory
on the noise sequence $\{w_t\}$. 
Our goal will be to design algorithms that exhibit \emph{sub-linear} regret,
i.e.,\ $\mathsf{Regret}(T) = o(T)$, which ensures that the time-averaged regret asymptotically converges to zero.
For ease of exposition,
in the sequel we define $Y_t := Y(x_t^a, t)$ and $B_t := B(x_t^a, t)$, and we use the symbol $\tilde{\alpha}_t$ to denote the parameter estimation error $\ \hat{\alpha}_t - \alpha$.

\subsection{Parameter Update Algorithms}

We study two primary classes of parameter update algorithms inspired by online convex optimization~\citep{hazan16oco}.
The first is the family of \emph{velocity gradient algorithms}~\citep{fradkov99book},
which perform online gradient-based optimization on a Lyapunov function for the nominal system.
The second obviates the need for a known Lyapunov function, and directly performs online optimization on the least-squares prediction error. Here we discuss the discrete-time formulation, but a self-contained introduction to these algorithms in continuous-time can be found in \cond{the full paper}{Appendix~\ref{sec:app:vel_grad_cont}}.

\subsubsection{Velocity gradient algorithms}
\label{ssec:vg}
Velocity gradient algorithms exploit access to a known Lyapunov function for the nominal dynamics. Specifically, assume the existence of a non-negative function $Q(x, t):\R^n\times\N \rightarrow \R_{\geq 0}$, which is differentiable in its first argument, and a constant $\rho \in (0, 1)$
such that for all $x, t$:
\begin{align}
    Q(f(x, t), t+1) \leq Q(x, t) - \rho \norm{x}^2 \:. \label{eq:lyapunov_stability}
\end{align}
Given such a $Q(x, t)$, velocity gradient methods update the parameters according to the iteration
\begin{align}
    \hat{\alpha}_{t+1} = \Pi_{\calC}[\hat{\alpha}_t - \eta_t Y(x_t, t)^\T B(x_t, t)^\T \nabla Q(x_{t+1}, t+1)] \:, \:\: \Pi_{\calC}[x] := \arg\min_{y \in \calC} \norm{x-y} \:, \label{eq:speed_gradient}
\end{align}
which can alternatively be viewed as projected gradient descent with respect to the parameters after noting that $Y(x_t, t)^\T B(x_t, t)^\T \nabla Q(x_{t+1}, t+1) = \nabla_{\hat{\alpha}_t}Q(x_{t+1}, t+1)$.
As we will demonstrate, the use of $\nabla Q(x_{t+1}, t+1)$ instead of
$\nabla Q(x_t, t)$ in \eqref{eq:speed_gradient} is key to unlocking a sublinear regret bound.

\subsubsection{Online least-squares}
\label{ssec:ols}

Online least-squares algorithms are motivated by minimizing the approximation error directly rather than through stability considerations. For each time $t$, define the prediction error loss function 
%
\begin{align}
    f_t(\hat{\alpha}) := \frac{1}{2} \norm{B(x_t, t) Y(x_t, t) (\hat{\alpha} - \alpha) + w_t}^2 \:. \label{eq:prediction_loss}
\end{align}
Unlike in the usual optimization setting, the loss at time $t$ is unknown to the controller, due to its dependence on the unknown parameters $\alpha$. However, its gradient $\nabla f_t(\hat{\alpha}_t)$ can be implemented
after observing $x_{t+1}$ through a discrete-time analogue of Luenberger's well-known approach for reduced-order observer design~\citep{luenberger79book}\footnote{We note that implementing this gradient update rule in continuous-time is substantially more involved; see Appendix~\ref{sec:app:vel_grad_cont} for a discussion.}:
\begin{align}
    \nabla f_t(\hat{\alpha}_t) = Y(x_t, t)^\T B(x_t, t)^\T (x_{t+1} - f(x_t, t)) \:.
\end{align}
The simplest update rule that uses the gradient $\nabla f_t(\hat{\alpha}_t)$ is online gradient descent:
\begin{align}
    \hat{\alpha}_{t+1} = \Pi_{\calC}[ \hat{\alpha}_t - \eta_t \nabla f_t(\hat{\alpha}_t) ] \:, \label{eq:online_gd}
\end{align}
while a more sophisticated update rule is the online Newton method:
\begin{align}
    \hat{\alpha}_{t+1} = \Pi_{\calC,t}[ \hat{\alpha}_t - \eta A_t^{-1} \nabla f_t(\hat{\alpha}_t) ] \:, \:\: A_t = \lambda I + \sum_{s=0}^{t} M_s^\T M_s \:, \:\: M_s = B(x_s, s) Y(x_s, s) \:. \label{eq:online_newton}
\end{align}
Above, the operator $\Pi_{\calC,t}[\cdot]$
denotes projection w.r.t.\ the $A_t$-norm: $\Pi_{\calC,t}[x] := \arg\min_{y \in \calC} \norm{x - y}_{A_t}$.
\section{Regret Bounds for Velocity Gradient Algorithms}
\label{sec:speed_gradient}

In this section, we provide a regret analysis for the velocity gradient algorithm. Here, we will assume a deterministic system, so that $w_t \equiv 0$.
Unrolling the Lyapunov stability assumption \eqref{eq:lyapunov_stability}
and using the non-negativity of $Q(x, t)$ yields
$\sum_{t=0}^{T-1} \norm{x_t^c}^2 \leq \frac{Q(x_0, 0)}{\rho}$,
which shows that the contribution of
$\sum_{t=0}^{T-1} \norm{x_t^c}^2$ to the regret is $O(1)$.
Therefore, it suffices to bound $\sum_{t=0}^{T-1} \norm{x_t^a}^2$ directly. The key assumption that enables application of the velocity gradient method in discrete-time is strong convexity of the Lyapunov function $Q(x, t)$ with respect to $x$.
Recall that a $C^1$ function $h(x)$ is $\mu$-strongly convex
if for all $x$ and $y$,
$h(y) \geq h(x) + \ip{\nabla h(x)}{y-x} + \frac{\mu}{2}\norm{y-x}^2$.
Our first result is a data-dependent regret bound for the velocity gradient algorithm.
\begin{restatable}{theorem}{speedgradientmain}
\label{thm:speed_grad_data_dependent_regret}
Fix a $\lambda > 0$.
Consider the velocity gradient update \eqref{eq:speed_gradient}
with $\hat{\alpha}_0 \in \calC$ and learning rate
$\eta_t = \frac{D}{\sqrt{\lambda + \sum_{i=0}^{t} \norm{Y_i^\T B_i^\T \nabla Q (x_{i+1}^a, i+1)}^2}}$.
Assume that the Lyapunov stability condition \eqref{eq:lyapunov_stability} is verified,
and that for every $t$, the map $x \mapsto Q(x, t)$ is $\mu$-strongly convex.
Then
for any $T \geq 1$:
\begin{align*}
    \sum_{t=0}^{T-1} \norm{x_t^a}^2 + \frac{\mu}{2\rho} \sum_{t=0}^{T-1} \norm{B_t Y_t \tilde{\alpha}_t}^2 \leq \frac{Q(x_0, 0)}{\rho} + \frac{5\sqrt{\lambda}D}{\rho} + \frac{3D}{\rho} \sqrt{ \sum_{t=0}^{T-1} \norm{Y_t^\T B_t^\T \nabla Q(x_{t+1}^a, t+1)}^2 } \:.
\end{align*}
\end{restatable}

By Theorem~\ref{thm:speed_grad_data_dependent_regret}, a
bound on $\sum_{t=0}^{T-1} \norm{Y_t^\T B_t^\T \nabla Q(x_{t+1}^a, t+1)}^2$ ensures a bound on the control regret.
One way to obtain a bound is to assume that 
$\norm{Y_t^\T B_t^\T \nabla Q(x_{t+1}^a, t+1)} \leq G$ for all $t$,
in which case Theorem~\ref{thm:speed_grad_data_dependent_regret}
yields the sublinear guarantee $\mathsf{Regret}(T) \leq O(\sqrt{T})$.
However, this can be strengthened by assuming that both $\nabla Q(x, t)$ and $f(x, t)$ are Lipschitz
continuous.

\begin{restatable}{theorem}{speedgradlipschitz}
\label{thm:speed_grad_lipschitz}
Suppose that for every $x$ and $t$, $\norm{\nabla Q(x, t)} \leq L_Q \norm{x}$ and $\norm{f(x, t)} \leq L_f \norm{x}$. Further assume that $\sup_{x,t} \norm{B(x, t)} \leq M$ and
$\sup_{x,t} \norm{Y(x, t)} \leq M$.
Then, under the hypotheses of Theorem~\ref{thm:speed_grad_data_dependent_regret},
for any $T \geq 1$:
\begin{align*}
    \sum_{t=0}^{T-1} \norm{x_t^a}^2 + \frac{\mu}{2\rho} \sum_{t=0}^{T-1} \norm{B_t Y_t \tilde{\alpha}_t}^2 \leq \frac{3}{2}\left( \frac{Q(x_0, 0)}{\rho} + \frac{5\sqrt{\lambda}D}{\rho}  \right) + \frac{27 D^2}{\rho^2} M^4 L_Q^2 \max\left\{L_f^2, \frac{2\rho}{\mu}\right\} \:.
\end{align*}
\end{restatable}

Theorem~\ref{thm:speed_grad_lipschitz} yields the constant bound 
$\mathsf{Regret}(T) \leq O(1)$, which mirrors an earlier result
in the continuous-time deterministic setting due to~\citet{gaudio19connections}.
\section{Regret Bounds for Online Least-Squares Algorithms}
\label{sec:online_ls}

In this section we study the use of online least-squares algorithms for
adaptive control in the stochastic setting.
A core challenge in this setting is that neither $\E \sum_{t=0}^{T-1} \norm{x_t^a}^2$ nor $\E \sum_{t=0}^{T-1} \norm{x_t^c}^2$ converges to a constant, but rather each grows as $\Omega(T)$. Any analysis yielding a sublinear regret bound must therefore consider the behavior of the trajectory $x_t^a$ \emph{together with} the trajectory $x_t^c$, and cannot bound the two terms independently.
Our approach couples the trajectories together with the same noise realization $\{w_t\}$, and then utilizes incremental stability to compare trajectories of the comparator and the adaptation algorithm. We first provide a brief introduction to contraction and incremental stability, and then we discuss our results.

\subsection{Contraction and Incremental Stability}
\label{sec:online_ls:incr_stability}
To prove regret bounds for our least-squares algorithms, we use the following generalization of input-to-state stability, which allows for a direct comparison between two trajectories of the system in terms of the strength of past inputs.
\begin{definition}[cf.\ \cite{angeli02incrementalstability}]
\label{def:e_delta_iss}
Let constants $\beta, \gamma$ be positive and $\rho \in (0, 1)$.
The discrete-time dynamical system $f(x, t)$ is
called $(\beta, \rho, \gamma)$-\emph{exponentially-incrementally-input-to-state-stable} (E-$\delta$ISS) for a pair of initial conditions $(x_0, y_0)$ and signal $u_t$ (which is possibly adapted to the history $\{x_s\}_{s \leq t}$)
if the trajectories
$x_{t+1} = f(x_t, t) + u_t$ and $y_{t+1} = f(y_t, t)$
satisfy for all $t \geq 0$:
\begin{align}
    \norm{x_t - y_t} \leq \beta \rho^{t} \norm{x_0 - y_0} + \gamma \sum_{k=0}^{t-1} \rho^{t-1-k} \norm{u_k} \:. \label{eq:e_delta_iss_ineq}
\end{align}
A system is $(\beta,\rho,\gamma)$-E-$\delta$ISS if it
is $(\beta,\rho,\gamma)$-E-$\delta$ISS for all initial conditions
$(x_0, y_0)$ and signals $u_t$.
\end{definition}

Definition~\ref{def:e_delta_iss} can be verified by 
checking if the system $f(x, t)$ is \emph{contracting}.
\begin{definition}[cf.\ \cite{lohmiller98contraction}]
The discrete-time dynamical system $f(x, t)$
is contracting with rate $\gamma \in (0, 1)$
in the metric $M(x, t)$
if
for all $x$ and $t$:
\begin{align*}
    \frac{\partial f}{\partial x}(x, t)^\T M(f(x, t), t+1) \frac{\partial f}{\partial x}(x, t) \preceq \gamma M(x, t) \:.
\end{align*}
\end{definition}

\begin{restatable}{proposition}{contractionincrstability}
\label{prop:contraction_implies_e_delta_iss}
Let $f(x, t)$ be contracting with rate $\gamma \in (0, 1)$ in the
metric $M(x, t)$.
Assume that for all $x, t$
we have $0 \prec \mu I \preceq M(x, t) \preceq L I$.
Then $f(x, t)$ is
$(\sqrt{L/\mu}, \sqrt{\gamma}, \sqrt{L/\mu})$-E-$\delta$ISS.
\end{restatable}
Furthermore, contraction is robust to small perturbations -- if the dynamics
$f(x, t)$ are contracting, so are the dynamics
$f(x, t) + w_t$ for small enough $w_t$.
\begin{restatable}{proposition}{contractionwithnoise}
\label{prop:contraction_with_noise}
Let $\{w_t\}$ be a fixed sequence
satisfying $\sup_{t \geq 0} \norm{w_t} \leq W$.
Suppose that $f(x, t)$ is contracting with rate $\gamma$
in the metric $M(x, t)$ with $M(x, t) \succeq \mu I$.
Define the perturbed dynamics $g(x, t) := f(x, t) + w_t$.
Suppose that for all $t$, the function $x \mapsto M(x, t)$ is $L_M$-Lipschitz.
Furthermore, suppose that $\sup_{x, t} \norm{\frac{\partial f}{\partial x}(x, t)} \leq L_f$.
Then as long as
$W \leq \frac{\mu(1-\gamma)}{L_f^2 L_M}$,
we have that $g(x, t)$ is contracting with rate $\gamma + \frac{L_f^2 L_M W}{\mu}$
in the metric $M(x, t)$.
\end{restatable}
Note that if the metric is state independent (i.e.,\ $M(x, t) = M(t)$), then we can
take $L_M = 0$ and hence the perturbed system $g(x, t)$ is contracting at
rate $\gamma$ for all realizations $\{w_t\}$.

\subsection{Main Results}
\label{sec:online_ls:results}

Our analysis proceeds by assuming that for almost all noise realizations $\{w_t\}$,
the perturbed nominal system $f(x, t) + w_t$ is incrementally stable
(E-$\delta$ISS). 
We apply incremental stability to bound
the control regret directly in terms of the prediction regret, $\mathsf{Regret}(T) \leq O(\sqrt{T} \sqrt{\sum_{t=0}^{T-1} \E \norm{B_t Y_t \tilde{\alpha}_t}^2} )$.
Because online convex optimization methods provide explicit guarantees on the prediction regret, we can apply existing results from the online optimization literature to generate a bound on the control regret.
To see this, recall that the sequence of prediction error functions
$\{f_t\}$ from \eqref{eq:prediction_loss} has the form $f_t(\hat{\alpha}) = \frac{1}{2} \norm{B_t Y_t(\hat{\alpha} - \alpha) + w_t}^2$.
Hence:
\begin{align*}
    \frac{1}{2}\E \sum_{t=0}^{T-1} \norm{B_t Y_t \tilde{\alpha}_t}^2 = \E\left[\sum_{t=0}^{T-1} f_t(\hat{\alpha}_t) - f_t(\alpha)\right] \leq \E\left[\sup_{\alpha \in \calC} \sum_{t=0}^{T-1} f_t(\hat{\alpha}_t) - f_t(\alpha)\right] \:.
\end{align*}

In this section, we make the following assumption regarding the
dynamics.
\begin{assumption}
\label{assumption:online_ls}
The perturbed system $g(x_t, t) := f(x_t, t) + w_t$
is $(\beta, \rho, \gamma)$-E-$\delta$ISS for all realizations
$\{w_t\}$ satisfying $\sup_{t} \norm{w_t} \leq W$.
Also
$\sup_{x, t} \norm{B(x, t)} \leq M$ and
$\sup_{x, t} \norm{Y(x, t)} \leq M$.
\end{assumption}

We define the constant $B_x := \beta \norm{x_0} + \frac{\gamma(2 D M^2 + W)}{1-\rho}$
and $G := M^2 (2DM^2 + W)$.
A key result, which relates control regret to prediction regret,
is given in the following theorem.
\begin{restatable}{theorem}{regretlsreduction}
\label{thm:regret_ls_reduction}
Consider any adaptive update rule $\{\hat{\alpha}_t\}$.
Under Assumption~\ref{assumption:online_ls}, for all $T \geq 1$:
\begin{align*}
    \E\left[\sum_{t=0}^{T-1} \norm{x_t^a}^2 - \norm{x_t^c}^2 \right] \leq \frac{2 B_x \gamma}{1-\rho} \sqrt{T} \sqrt{\sum_{t=0}^{T-1} \E\norm{B_t Y_t \tilde{\alpha}_t}^2} \:.
\end{align*}
\end{restatable}

We can immediately specialize Theorem~\ref{thm:regret_ls_reduction}
to both online gradient descent and online Newton.
Both corollaries are a direct consequence of applying well-known regret bounds
in online convex optimization to Theorem~\ref{thm:regret_ls_reduction} \cond{ (cf.\ \cite{hazan16oco})}{ (cf.\ Proposition~\ref{prop:online_gd_bounded_case}
and Proposition~\ref{prop:online_newton_bounded_case} in Appendix~\ref{sec:app:online_ls})}.
Our first corollary shows that online gradient descent
achieves a $O(T^{3/4})$ control regret bound.
\begin{corollary}
\label{cor:online_GD}
Suppose we use
online gradient descent \eqref{eq:online_gd} to update the parameters, setting the learning rate $\eta_t = \frac{D}{G\sqrt{t+1}}$. Under Assumption~\ref{assumption:online_ls},
for all $T \geq 1$:
\begin{align*}
    \E\left[\sum_{t=0}^{T-1} \norm{x_t^a}^2 - \norm{x_t^c}^2 \right] \leq 2\sqrt{6} B_x \frac{\gamma}{1-\rho} \sqrt{GD} T^{3/4} \:.
\end{align*}
\end{corollary}
This result immediately generalizes to the case of mirror descent, where dimension-dependence implicit in $G$ and $D$ can be reduced, and where recent implicit regularization results apply~\citep{boffi2020implicit}. Next, the regret can be improved to $O(\sqrt{T \log{T}})$ by using online Newton.
\begin{corollary}
\label{cor:online_newton}
Suppose we use the online Newton method \eqref{eq:online_newton} to update the parameters, setting $\eta = 1$. Suppose furthermore that $M \geq 1$. Under Assumption~\ref{assumption:online_ls},
for all $T \geq 1$: 
\begin{align*}
    \E\left[\sum_{t=0}^{T-1} \norm{x_t^a}^2 - \norm{x_t^c}^2 \right] \leq \frac{2B_x \gamma}{1-\rho} \sqrt{T} \sqrt{ 4D^2(\lambda + M^4) + p G^2 \log(1 + M^4 T/\lambda) } \:.
\end{align*}
\end{corollary}

We also note that in the deterministic setting,
online gradient descent to update the parameters achieves $O(1)$ prediction and control regret,
which is consistent with the results in Section~\ref{sec:speed_gradient}
and with the results in \cite{gaudio19connections}.
We give a self-contained proof of this in \cond{the full paper}{Appendix~\ref{sec:app:vel_grad_cont}}.

\subsection{Input Delay Results}
Motivated by \emph{extended matching} conditions commonly considered in  continuous-time adaptive control~\citep{krstic95adaptivebook}, we now extend our previous results to a setting where the input is time-delayed
by $k$ steps.
Specifically, we consider the modified system:
\begin{align}
    x_{t+1} = f(x_t, t) + B(x_t, t) (\xi_t - Y(t) \alpha) + w_t \:, \:\:
    \xi_{t} = u_{t-k} \:. \label{eq:delayed_system}
\end{align}
Here, we simplify part of the model~\eqref{eq:system} by assuming that
the matrix $Y(t)$ is state-independent. With this simplification,
the certainty equivalence controller is given by $u_t = Y(t+k) \hat{\alpha}_t$. The baseline we compare to in the definition of regret is 
the nominal system $x_{t+1}^c = f(x_t^c, t) + w_t$,
which is equivalent to playing the input $u_t = Y(t+k)\alpha$.
Note that the gradient $\nabla f_t(\hat{\alpha}_t)$ can be implemented by the controller as
$\nabla f_t(\hat{\alpha}_t) = Y_t^\T B_t^\T ( x_{t+1} - f(x_t, t) - B_t (\xi_t - Y_t \hat{\alpha}_t) )$.

Folk wisdom and basic intuition suggest that nonlinear adaptive control algorithms for the extended matching setting will perform worse than their matched counterparts; however, standard asymptotic guarantees do not distinguish between the performance of these two classes of algorithms. Here we show that the control regret rigorously captures this gap in performance. We begin with online gradient descent, which provides a regret bound of $O(T^{3/4} + k\sqrt{T})$.
\begin{restatable}{theorem}{delayresult}
\label{thm:regret_ls_delay}
Consider the online gradient descent update~\eqref{eq:online_gd}
for the $k$-step delayed system~\eqref{eq:delayed_system}
with step size $\eta_t = \frac{D}{G\sqrt{t+1}}$.
Under Assumption~\ref{assumption:online_ls} and with
state-independent $Y_t$, for all $T \geq k$:
\begin{align*}
    \E\left[\sum_{t=0}^{T-1} \norm{x_t^a}^2 - \norm{x_t^c}^2\right] \leq k B_x^2 + \frac{2B_x M^2 D \gamma}{(1-\rho)^2} + \frac{2 \sqrt{6} B_x \gamma \sqrt{GD}}{1-\rho} T^{3/4} + \frac{4 B_x \gamma M^2 D}{1-\rho} k \sqrt{T}  \:.
\end{align*}
\end{restatable}

Furthermore, the regret improves to $O(k \sqrt{T \log{T}})$ when we use
the online Newton method.
\begin{restatable}{theorem}{delayresultnewton}
\label{thm:regret_ls_delay_newton}
Consider the online Newton update~\eqref{eq:online_newton}
for the $k$-step delayed system~\eqref{eq:delayed_system} with $\eta = 1$.
Suppose $M \geq 1$.
Under Assumption~\ref{assumption:online_ls} and with state-independent $Y_t$, for all $T \geq k$:
\begin{align*}
\E\left[\sum_{t=0}^{T-1} \norm{x_t^a}^2 - \norm{x_t^c}^2\right] &\leq k B_x^2 + \frac{2B_x M^2 D \gamma}{(1-\rho)^2} + \frac{2 B_x \gamma G k}{1-\rho} \sqrt{\frac{p T}{\lambda} \log(1 + M^2 T/\lambda)}  \\
    &\qquad +\frac{2B_x\gamma}{1-\rho} \sqrt{T} \sqrt{4D^2(\lambda + M^4) + p G^2 \log(1 + M^4 T/\lambda) } \:.
\end{align*}
\end{restatable}

\subsection{Is Incremental Stability Necessary?}

The results in this section have crucially relied on 
incremental input-to-state stability 
(Definition~\ref{def:e_delta_iss}). A natural question to ask is
if it possible to relax this assumption to
input-to-state stability~\citep{sontag08iss}, while still retaining regret guarantees.
In \cond{the full paper}{the appendix}, we provide a partial answer to this question, which we outline here.
We build on the observation of \cite{ruffer13convergent},
who show that a convergent system is incrementally stable
over a compact set
(cf.\ Theorem 8 of \cite{ruffer13convergent}).
However, their analysis does not preserve rates of convergence, e.g.,\
it does not show that an exponentially convergent system is also
exponentially incrementally stable on a compact set.

In \cond{the full paper}{Appendix~\ref{sec:appendix:stability}}, we
show\cond{}{ in Lemma~\ref{lemma:stability_to_incremental_stability}}
that if a system is exponentially input-to-state stable \cond{(which we define similarly to Definition~\ref{def:e_delta_iss}, but in reference to a single trajectory)}{(cf.\ Definition~\ref{def:e_iss})}, 
then it is E-$\delta$ISS
on a compact set of initial conditions, but only
for certain \emph{admissible} inputs.
Next, we prove that
under a persistence of excitation condition, 
the disturbances $\{B_t Y_t \tilde{\alpha}_t\}$
due to parameter mismatch yield an admissible sequence of inputs with
high probability.
Combining these results, we show a $\sqrt{T} \log{T}$ regret bound that holds with \emph{constant} probability\cond{}{ (cf.\ Theorem~\ref{thm:stability_regret_bound})}.
We are currently unable to recover a high probability regret bound since
the $(\beta,\rho,\gamma)$ constants for our E-$\delta$ISS reduction
depend \emph{exponentially} on the original problem constants
and the size of the compact set.
We leave resolving this issue, in addition to removing the persistence of excitation
condition, to future work.
\section{Simulations}
\label{sec:simulations}

\subsection{Velocity Gradient Adaptation}
We consider the cartpole stabilization problem,
where we assume the true parameters are unknown. 
Let $q$ be the cart position, $\theta$ the pole angle, and $u$ the force applied to the cart.
The dynamics are:
\begin{align*}
    \ddot{q} &= \frac{1}{m_c + m_p \sin^2{\theta}} \left( u + m_p \sin{\theta}( \ell \dot{\theta}^2 + g \cos{\theta})\right) \:, \\
    \ddot{\theta} &= \frac{1}{\ell(m_c + m_p \sin^2{\theta})} \left( - u \cos{\theta} - m_p \ell \dot{\theta}^2 \cos{\theta}\sin{\theta} - (m_c + m_p) g \sin{\theta} \right) \:.
\end{align*}
We discretize the dynamics via the Runge-Kutta method with timestep $\Delta t = .01$. The true (unknown) parameters
are the cart mass $m_c = 1$g, the pole mass $m_p = 1$g, and pole length $\ell = 1$m. 
Let the state $x = (q, \dot{q}, \theta, \dot{\theta})$.
We solve a discrete-time infinite-horizon LQR problem (with $Q=I_{4}$ and $R=.5$) for
the linearization at $x_{\mathrm{eq}} := (0, 0, \pi, 0)$, 
using the \textit{wrong parameters} $m_c = .45$g, $m_p = .45$g, $\ell = .8$m.
This represents a simplified model of uncertainty in the system or a simulation-to-reality gap.
The solution to the discrete-time LQR problem yields a Lyapunov function $Q(x) = \frac{1}{2} (x-x_{\mathrm{eq}})^\T P (x-x_{\mathrm{eq}})$, and a control law $u_t = -K(x_t - x_{\mathrm{eq}})$ that would locally stabilize the system around $x_{\mathrm{eq}}$ if the parameters were correct. 

We use adaptive control to bootstrap our control policy computed with incorrect parameters to a stabilizing law for the true system. 
Specifically, we run the velocity gradient adaptive law \eqref{eq:speed_gradient} on the LQR Lyapunov function $Q(x)$ with basis functions $Y(x, t) \in \R^{1\times 400}$ given by random Gaussian features $\cos(\omega^\T x + b)$ with $\omega \sim N(0, 1)$ and $b \sim \mathsf{Unif}(0, 2\pi)$~(cf.~\cite{rahimi07randomfeatures}). 
We rollout $500$ trajectories initialized 
uniformly at random in an $\ell_\infty$ ball of radius $\frac{1}{2}$ around $x_{\mathrm{eq}}$, and measure the performance of the system both with and without adaptation through the average control regret $\frac{1}{T}\sum_{t=1}^T \norm{x_t - x_{\mathrm{eq}}}^2$.
The results are shown in the bottom-right pane of Figure~\ref{fig:results}.
Without adaptation, every trajectory diverges, and an example is shown in the left inset. 
On the other hand, adaptation is often able to
successfully stabilize the system.
One example trajectory with adaptation is shown in the body of the pane. The right inset shows the empirical CDF of the average control cost with adaptation, indicating that $\sim 60\%$ of trajectories with adaptation have an average control regret less than $0.1$, and $\sim 80\%$ less than $1$. 
More generally, our approach of improving the quality of a controller through online adaptation with expressive, unstructured basis functions could be used as an additional layer on top of existing adaptive control algorithms to correct for errors in the structured, physical basis functions originating from the dynamics model.

\begin{figure}[t]
{\begin{tabular}{cc}
    \includegraphics[width=.48\textwidth]{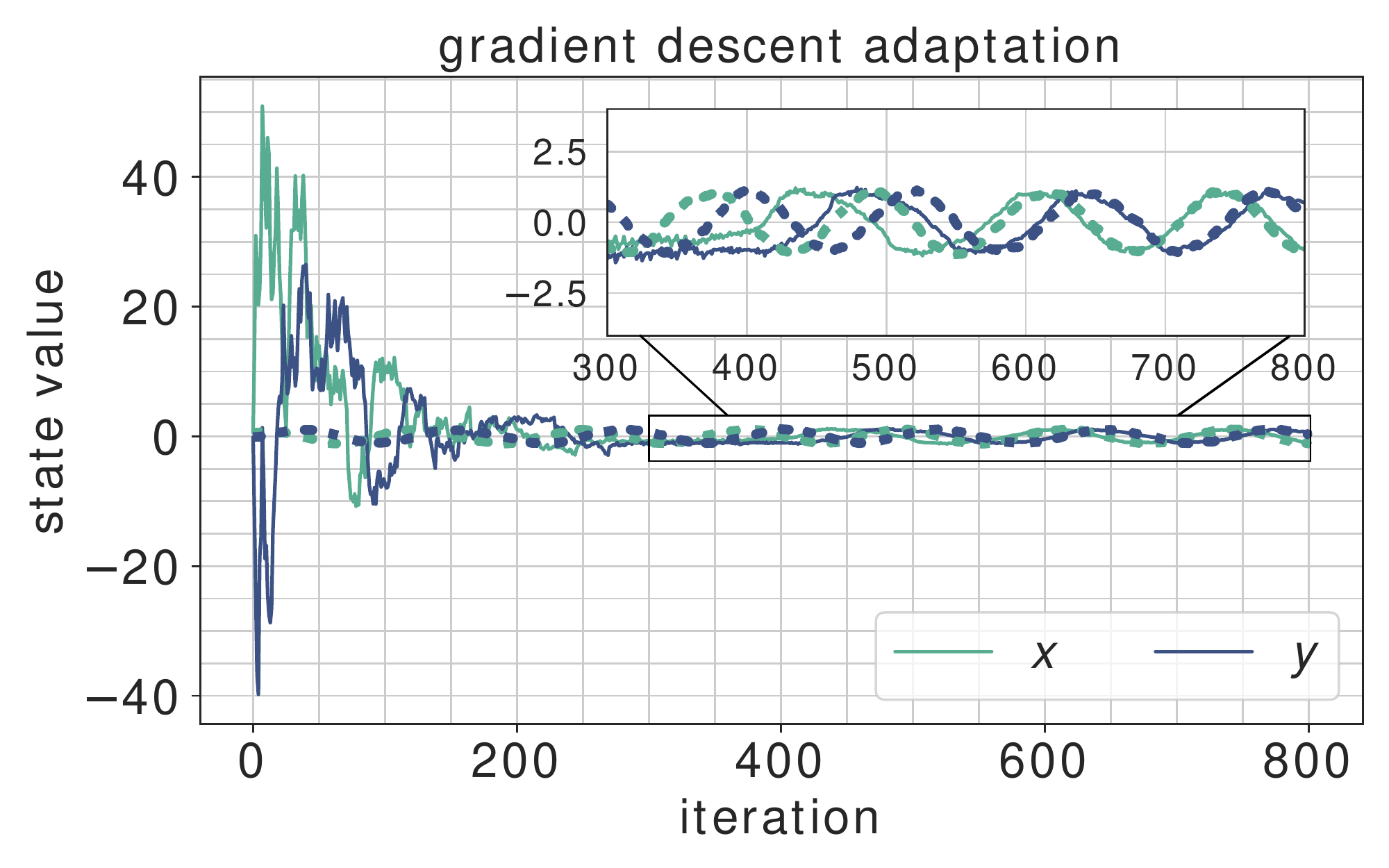} & 
    \includegraphics[width=.48\textwidth]{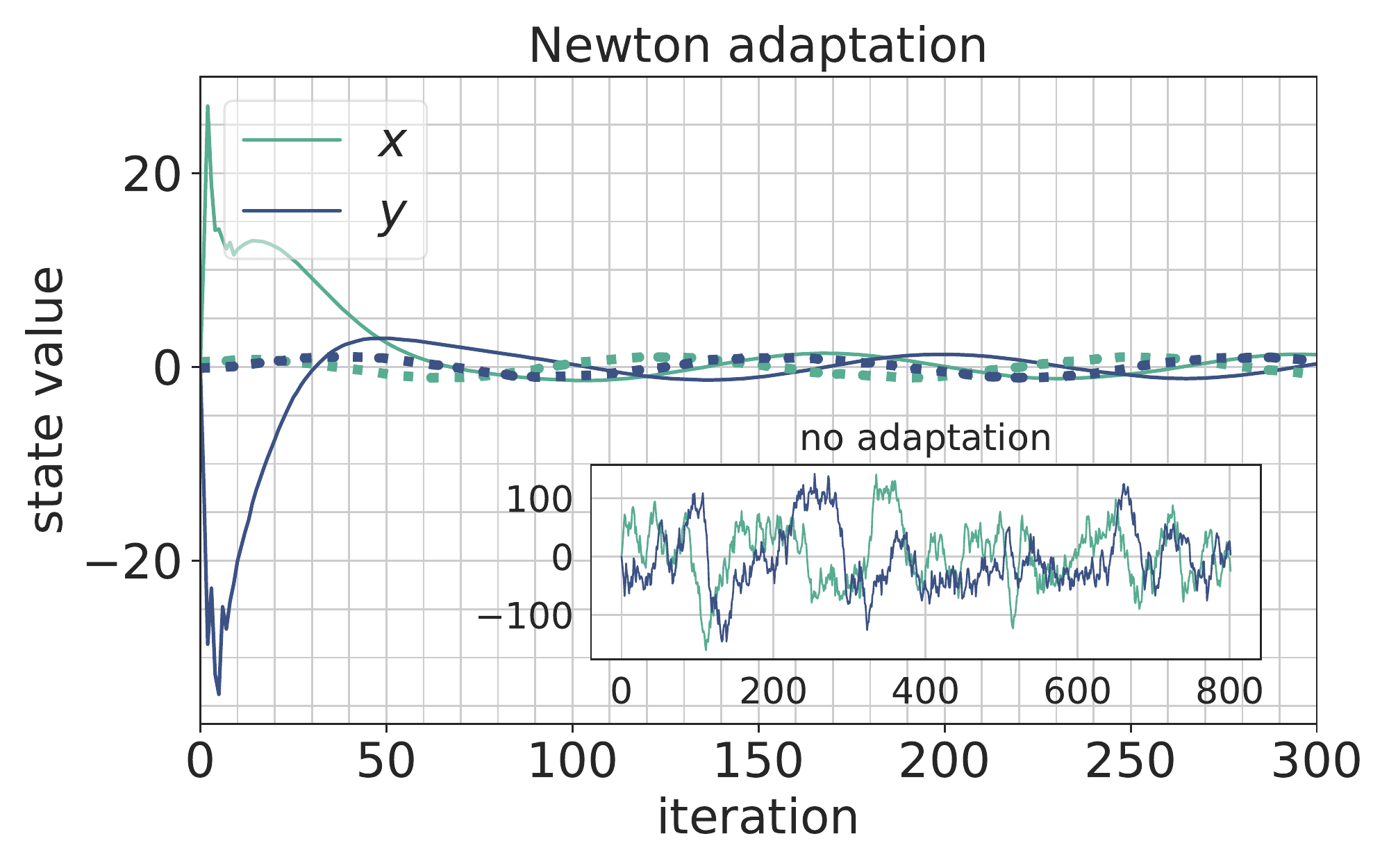}\\
    \includegraphics[width=.48\textwidth]{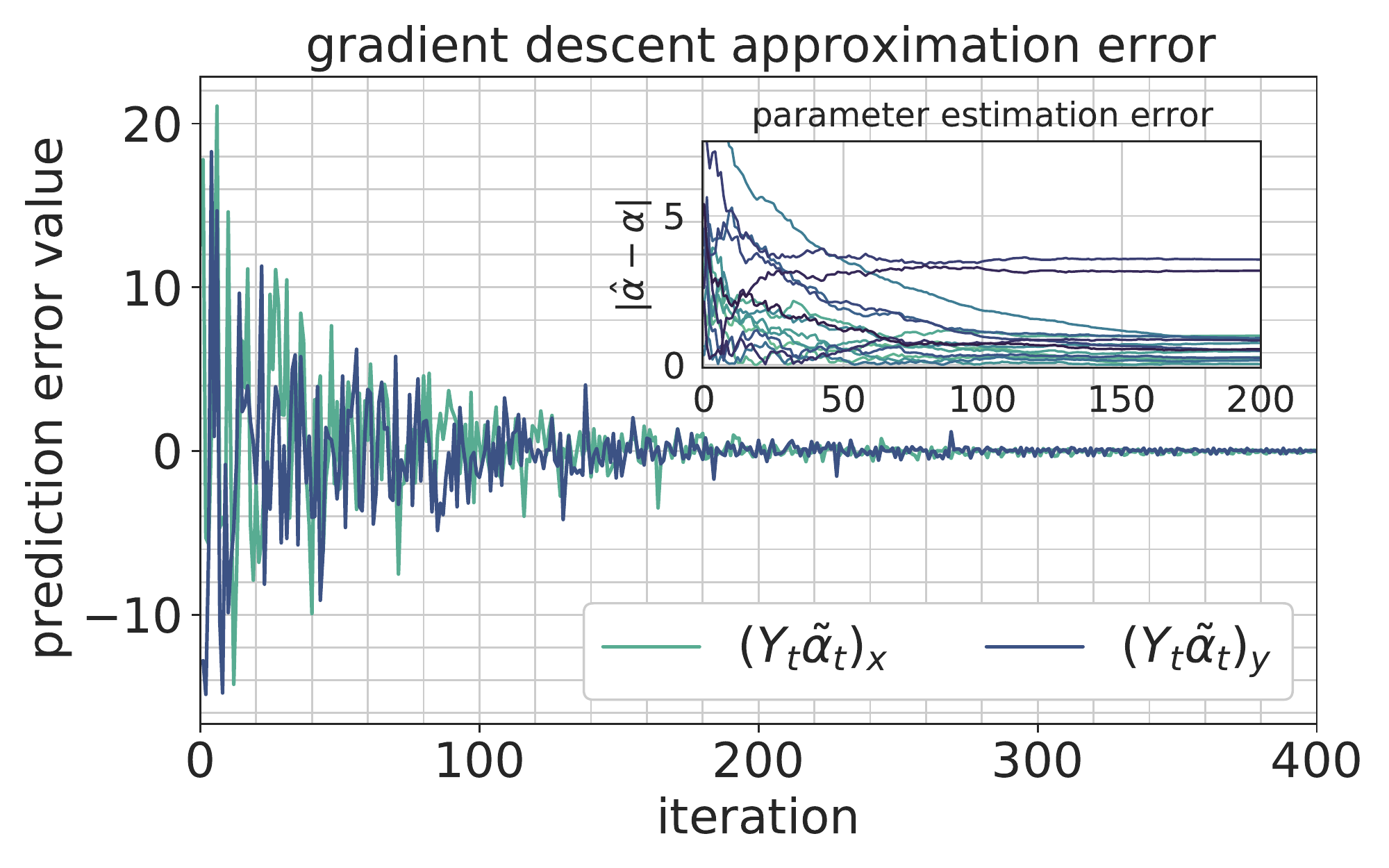} &
    \includegraphics[width=.48\textwidth]{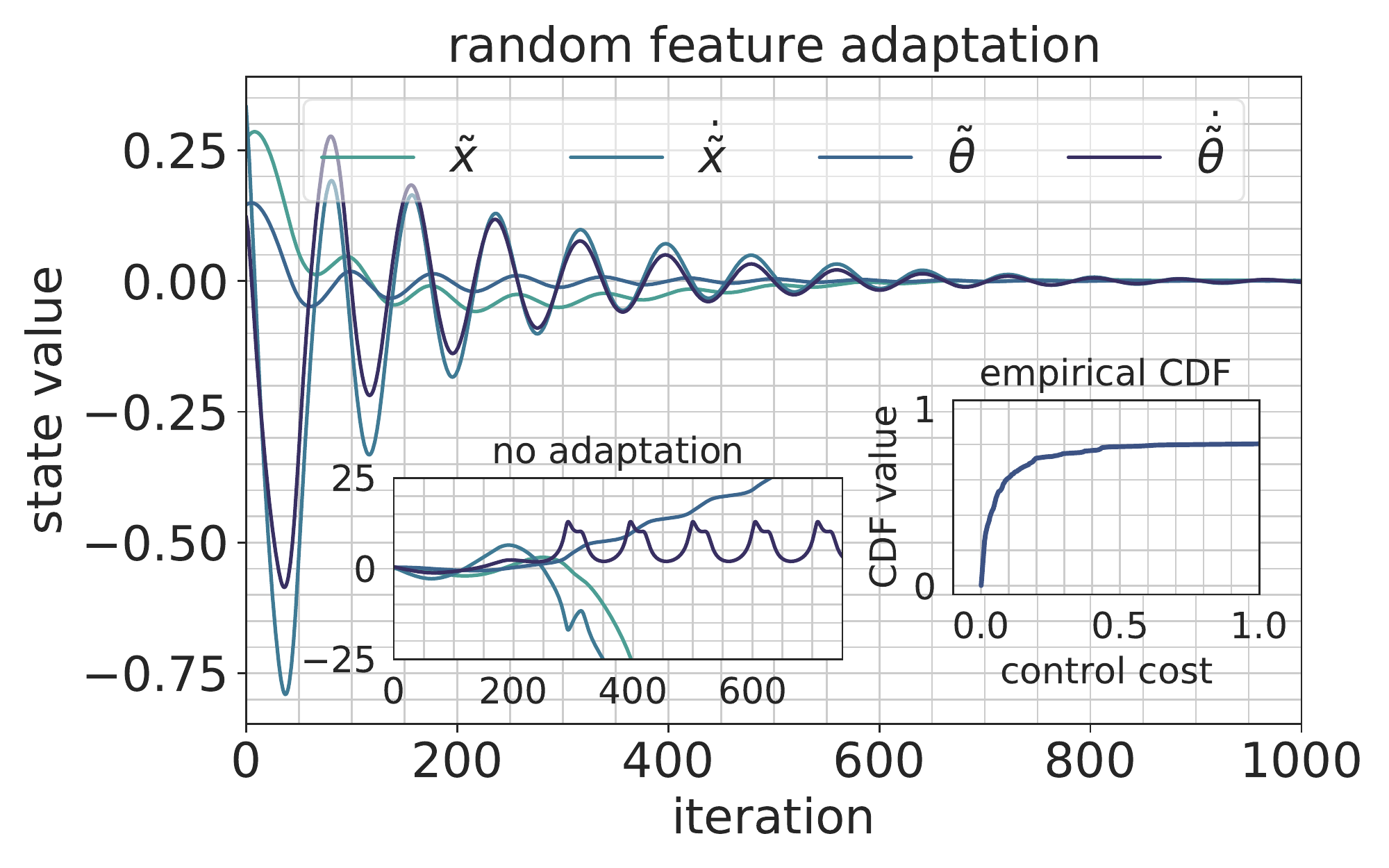} 
\end{tabular}
}
\caption{\textbf{(Top left)} Sample trajectory for online gradient descent (solid) and the comparator (dotted). Inset shows a close-up view near convergence. \textbf{(Top right)} Sample trajectory for online gradient descent (solid) and the comparator (dotted). Inset shows poor performance of the system \textit{without} adaptation.  \textbf{(Bottom left)} Prediction error for gradient descent (main figure) and parameter estimation error (inset). The parameters do not converge due to a lack of persistent excitation, but the prediction error still tends to zero. \textbf{(Bottom right)} LQR experiment with random features. Main figure shows the performance of one trajectory with adaptation. The right inset shows the empirical CDF of average control performance with random feature adaptation. The left inset shows divergent behavior of one trajectory without adaptation.}
\label{fig:results}
\end{figure}

\subsection{Online Convex Optimization Adaptation}
To demonstrate the applicability of our OCO-inspired discrete-time adaptation laws, we study the following discrete-time nonlinear system
\begin{equation}
    \label{eqn:sim_sys}
    \begin{aligned}
    x_{t+1} &= x_t + \tau\left(-y_t + \frac{x_t}{\sqrt{x_t^2+y_t^2}} - x_t + Y_x(x_t, t)^\T\tilde{\alpha}_t\right) + \sqrt{\tau}\sigma w_{t,1} \:, \\
    y_{t+1} &= y_t + \tau\left(x_t + \frac{y_t}{\sqrt{x_t^2 + y_t^2}} - y_t + Y_y(y_t, t)^\T\tilde{\alpha}_t\right) + \sqrt{\tau}\sigma w_{t,2}
    \end{aligned}
\end{equation}
for $\tau = 0.05$, $\sigma = 0.1$, and $w_{t,i} \sim N(0, 1)$. The nominal system for \eqref{eqn:sim_sys} is a forward-Euler discretization of the continuous-time system $\dot{x} = -y + \frac{x}{\sqrt{x^2 + y^2}} - x$, $\dot{y} = x + \frac{y}{\sqrt{x^2 + y^2}} - y$. In polar coordinates, the nominal system reads $\dot{r} = -(r-1)$, $\dot{\theta} = 1$, which is contracting in the Euclidean metric towards the limit cycle $\dot{\theta} = 1$ on the unit circle. This shows that the system in Euclidean coordinates is contracting in the radial direction in the metric $M(x, y) = \frac{\partial g}{\partial x}(x, y)^\T\frac{\partial g}{\partial x}(x, y)$, where $g$ is the nonlinear mapping $(x, y) \mapsto (r, \theta)$. The basis functions are taken to be $Y_z(z_t, t)^\T = \sin(\omega (z_t + \sin(t)))$ where $z \in \{x,y\}$, the outer $\sin$ is taken element-wise, and $\omega \in \R^p$ is a vector of frequencies sampled uniformly between $0$ and $2\pi$. The estimated parameters $\hat{\alpha}_t$ are updated according to the OCO-inspired adaptive laws \eqref{eq:online_gd} or \eqref{eq:online_newton} analyzed in Section~\ref{sec:online_ls:results}.

Results are shown in Figure~\ref{fig:results}. In the top-left pane, convergence of a sample trajectory towards the limit cycle is shown for gradient descent in solid, with the limit cycle itself plotted in dots. The inset displays a close-up view of convergence. In the top-right pane, convergence is shown for the online Newton method, which converges significantly faster and has a smoother trajectory than gradient descent. The inset displays a failure to converge without adaptation, demonstrating improved performance of the two adaptation algorithms in comparison to the system without adaptation. The bottom-left pane shows convergence of the two components of the prediction error $Y_t\tilde{\alpha}_t$ for gradient descent in the main figure, and shows parameter error trajectories in the inset. Note that the parameters do not converge to the true values due to a lack of persistent excitation.

\section{Conclusion and Future Work}
\label{sec:conclusion}

We present the first finite-time regret bounds for nonlinear adaptive control in discrete-time. Our work opens up many future directions of research. 
One direction is the possibility of logarithmic regret in our setting, given that it is achievable in various LQR problems~\citep{agarwal19logarithmic,cassel20logarithmic,foster20logarithmic}.
A second question is
handling state-dependent $Y(x, t)$ matrices in the $k$ timestep delay setting,
or more broadly, studying
the extended matching conditions of \cite{kanellakopoulos89extendedmatching,krstic95adaptivebook}
for which timestep delays are a special case.
Another direction concerns proving regret bounds for the velocity gradient algorithm in a stochastic setting. 
Furthermore,
in the spirit of \cite{agarwal19onlinecontrol,hazan20nonstochastic}, 
an extension of our analysis to handle more general cost functions and adversarial
noise sequences would be quite impactful.
Finally, understanding if sublinear regret guarantees are possible for a non-exponentially incrementally stable system would be interesting.

\section*{Acknowledgements}
The authors thank Naman Agarwal, Vikas Sindhwani, and Sumeet Singh for helpful feedback. 

\bibliography{paper}

\newpage
\tableofcontents
\newpage

\clearpage
\appendix

\section{Velocity Gradient Algorithms in Continuous-Time}
\label{sec:app:vel_grad_cont}
In this section, we provide a brief introduction to the continuous-time formulation of velocity gradient algorithms, and show how the continuum limit of the online convex optimization-inspired algorithms from Section~\ref{ssec:ols} can be seen as a particular case. 
A comprehensive treatment of velocity gradient algorithms
in continuous-time can be found in~\citet{fradkov99book}, Chapter 3.

In this section, we study the nonlinear dynamics with matched uncertainty
\begin{equation}
    \label{eqn:gen_dyn_cont}
    \dot{x} = f(x, t) + B(x, t)\left(u - Y(x, t)\alpha\right),
\end{equation}
with $f(x, t)$ a known nominal dynamics satisfying $f(0, t) = 0$ for all $t$, $B(x, t)$ and $Y(x, t)$ known matrix-valued functions, and $\alpha$ an unknown vector of parameters. As in the main text, we consider the certainty equivalence control input $u = Y(x, t)\hat{\alpha}$. We assume that $f(x, t)$, $B(x, t)$, and $Y(x, t)$ are continuous in $x$ and $t$.

The first result from \citet{fradkov99book} we describe
concerns the class of ``local'' velocity gradient algorithms, which use a Lyapunov function for the nominal system to adapt to unknown disturbances.
\begin{theorem}
Consider the system dynamics \eqref{eqn:gen_dyn_cont}. Suppose $f(x, t)$ admits a twice continuously differentiable Lyapunov function $Q(x, t)$ satisfying for some positive $\rho, \mu$:
\begin{enumerate}
    \item $Q(0, t) = 0$ and $Q(x, t) \geq \mu \norm{x}^2$ for all $x, t$.
    \item For all $x, t$,
        $\ip{\nabla_x Q(x, t)}{f(x, t)} + \frac{\partial Q}{\partial t}(x, t)\leq -\rho Q(x, t)$. 
\end{enumerate}
Define 
\begin{align*}
    \omega(x, \hat{\alpha}, t) := \ip{\nabla_x Q(x, t)}{f(x, t) + B(x, t)Y(x, t)(\hat{\alpha} - \alpha)} + \frac{\partial Q}{\partial t}(x, t) \:.
\end{align*}
Then the adaptation law
\begin{equation}
    \label{eqn:vg_update_cont}
    \dot{\hat{\alpha}} = -\nabla_{\hat{\alpha}}\omega(x(t), \hat{\alpha}(t), t)
\end{equation}
ensures that:
\begin{enumerate}
    \item The solution $(x(t), \hat{\alpha}(t))$ exists and is unique for all $t \geq 0$.
    \item The solution $(x(t), \hat{\alpha}(t))$ satisfies
    \begin{align*}
        \int_0^\infty \norm{x(t)}^2 \: dt \leq \frac{1}{\rho\mu}\left(Q(x(0), 0) + \frac{1}{2}\norm{\hat{\alpha}(0) - \alpha}^2\right) \:.
    \end{align*}
    \item The solution $(x(t), \hat{\alpha}(t))$ satisfies $x(t) \rightarrow 0$.
\end{enumerate}
\label{thm:vg}
\end{theorem}
\begin{proof}
By our continuity assumptions, we have that the 
closed-loop dynamics
\begin{align*}
    \dot{x} &= f(x, t) + B(x, t) Y(x, t) (\hat{\alpha} - \alpha) \:, \\
    \dot{\hat{\alpha}} &= - \nabla_{\hat{\alpha}} \omega(x, \hat{\alpha}, t) \:,
\end{align*}
is continuous in $t$ and locally Lipschitz in $(x, \alpha)$. Therefore, there exists a maximal interval $I(x(0), \hat{\alpha}(0)) \subseteq \R_{\geq 0}$
for which the solution $(x(t), \hat{\alpha}(t))$ exists and is unique.
Consider the Lyapunov-like function
\begin{equation*}
    V(x(t), \hat{\alpha}(t), t) = Q(x(t), t) + \frac{1}{2}\norm{\tilde{\alpha}(t)}^2.
\end{equation*}
It is simple to show that $V$ has time derivative
\begin{equation*}
    \dot{V}(x(t), \hat{\alpha}(t), t) = \omega(x(t), \hat{\alpha}(t), t) - \ip{\tilde{\alpha}(t)}{\nabla_{\hat{\alpha}}\omega(x(t), \hat{\alpha}(t), t)}.
\end{equation*}
Hence,
\begin{equation*}
    \dot{V}(x(t), \hat{\alpha}(t), t) = \omega(x(t), \alpha, t) = \ip{\nabla_{x}Q(x(t), t)}{f(x, t)} + \frac{\partial Q}{\partial t}(x, t) \leq -\rho Q(x(t), t) \leq -\rho\mu \norm{x(t)}^2,
\end{equation*}
which shows that $x(t)$ and $\hat{\alpha}(t)$ remain uniformly bounded
for all $t \in I(x(0), \hat{\alpha}(0))$.
This in turn implies that the solution $(x(t), \hat{\alpha}(t))$ exists
and is unique for all $t \geq 0$ (see e.g., Theorem 3.3 of \cite{khalilbook}).

Integrating both sides of the above differential inequality shows that
\begin{equation*}
    \int_0^\infty \norm{x(t)}^2 \: dt \leq \frac{1}{\rho\mu}\left(Q(x(0), 0) + \frac{1}{2}\norm{\hat{\alpha}(0) - \alpha}^2\right).
\end{equation*} 
By the assumption that $f$, $B$, and $Y$ are continuous
and that $Q$ is twice continuously differentiable, it is straightforward to check that
$\sup_{t \geq 0} \abs{\ddot{V}(x(t), \hat{\alpha}(t), t)} < \infty$.
We have therefore shown that
$\lim_{t \to \infty} V(x(t), \hat{\alpha}(t), t)$ exists and is finite,
and also that $\dot{V}(x(t), \hat{\alpha}(t), t)$ is uniformly continuous in $t$.
Applying Barbalat's lemma (see e.g.,\ Section 4.5.2 of~\cite{slotine91book}) yields the conclusion that
$\lim_{t \to \infty} \dot{V}(x(t), \hat{\alpha}(t), t) = 0$.
But this implies that:
\begin{align*}
    0 = \lim_{t \to \infty} \dot{V}(x(t), \hat{\alpha}(t), t) \leq \limsup_{t \to \infty} \left[-\rho \mu \norm{x(t)}^2\right] \leq 0 \:.
\end{align*}
Hence $x(t) \rightarrow 0$ as $t \rightarrow \infty$.
\end{proof}

In general, the proof of Theorem~\ref{thm:vg} works as long as
$\omega(x, \hat{\alpha}, t)$ is convex in $\hat{\alpha}$.
In this case, one has that the inequality $\dot{V}(x(t), \hat{\alpha}(t), t) \leq \omega(x(t), \alpha, t)$ holds. 

The continuous-time formulation \eqref{eqn:vg_update_cont} gives justification for the name ``velocity gradient''; $\omega(x(t), \hat{\alpha}(t), t)$ is the time derivative (velocity) of $Q(x(t), t)$ along the flow of the disturbed system. The adaptation algorithm is then derived by taking the gradient with respect to the parameters of this velocity. Moreover, \eqref{eqn:vg_update_cont} provides an explanation for the discrete-time requirement that $\nabla_x Q(x, t)$ be evaluated at $x_{t+1}$. In continuous-time, the instantaneous time derivative of $Q(x(t), t)$ provides information about the current function approximation error $B(x(t), t)Y(x(t), t)\tilde{\alpha}(t)$, which is only contained in $x_{t+1}$ in discrete-time.

A second class of ``integral'' velocity gradient algorithms from \citet{fradkov99book} can be obtained under a different set of assumptions, as shown next. These algorithms proceed by updating the parameters along the gradient of an instantaneous loss function $R(x, \hat{\alpha}, t)$. They then provide guarantees on the integral of $R(x, \hat{\alpha}, t)$ along trajectories of the system. In general, such a guarantee does not imply boundedness of the state, which must be shown independently.
\begin{theorem}
\label{thm:integral}
Let $R(x, \hat{\alpha}, t)$ denote a non-negative function that is convex in $\hat{\alpha}$ for all $x, t$. Let $\mu(t)$ denote a non-negative function such that $\int_0^\infty \mu(t) \: dt < \infty$ and $\lim_{t\rightarrow\infty} \mu(t) = 0$. Assume there exists some vector of parameters $\alpha$ satisfying $R(x, \alpha, t) \leq \mu(t)$ for all $x,t$. Then the adaptation law
\begin{equation}
    \label{eqn:vg_update_int}
    \dot{\hat{\alpha}} = -\nabla_{\hat{\alpha}}R(x(t), \hat{\alpha}(t), t)
\end{equation}
ensures that 
\begin{equation*}
    \int_0^t R(x(t'), \hat{\alpha}(t'), t') \: dt' \leq \frac{1}{2}\norm{\hat{\alpha}(0) - \alpha}^2 + \int_0^\infty \mu(t) \: dt.
\end{equation*}
for any $t \geq 0$ in the maximal interval of existence $I(x(0), \hat{\alpha}(0))$.
\end{theorem}
\begin{proof}
Consider the Lyapunov-like function
\begin{equation*}
    V(x(t), \hat{\alpha}(t), t) = \int_0^t R(x(t'), \hat{\alpha}(t'), t') \: dt' + \frac{1}{2}\norm{\tilde{\alpha}(t)}^2 + \int_t^\infty \mu(t') \: dt'.
\end{equation*}
Note that $V(x(t), \hat{\alpha}(t), t)$ has its time derivative given by
\begin{equation*}
    \dot{V}(x(t), \hat{\alpha}(t), t) = R(x(t), \hat{\alpha}(t), t) - \ip{\tilde{\alpha}(t)}{\nabla_{\hat{\alpha}}R(x(t), \hat{\alpha}(t), t)} - \mu(t).
\end{equation*}
By convexity of $R(x, \hat{\alpha}, t)$ in $\hat{\alpha}$, we have
\begin{equation*}
    \dot{V}(x(t), \hat{\alpha}(t), t) \leq R(x(t), \alpha, t) - \mu(t) \leq 0.
\end{equation*}
Because $\dot{V}(x(t), \hat{\alpha}(t), t) \leq 0$, and because each term in $V(x(t), \hat{\alpha}(t), t)$ is positive,
\begin{equation*}
    \int_0^t R(x(t'), \hat{\alpha}(t'), t') \: dt' \leq V(x(t), \hat{\alpha}(t), t) \leq V(x(0), \hat{\alpha}(0), 0) =  \frac{1}{2}\norm{\hat{\alpha}(0) - \alpha}^2 + \int_0^\infty \mu(t) \: dt.
\end{equation*}
\end{proof}
An important case for Theorem~\ref{thm:integral} is when $R(x(t), \hat{\alpha}(t), t)$ is the squared prediction error, i.e.,
\begin{equation}
    \label{eqn:R_ols}
    R(x, \hat{\alpha}, t) = \frac{1}{2}\norm{B(x, t)Y(x, t)\tilde{\alpha}}^2.
\end{equation}
In this case, $R(x, \alpha, t) = 0$, so that $\mu(t)$ can be taken to be zero. 
With the choice of $R$ given in \eqref{eqn:R_ols},
the resulting adaptation law \eqref{eqn:vg_update_int} becomes the gradient flow dynamics
\begin{equation}
    \label{eqn:gflow}
    \dot{\hat{\alpha}} = -Y(x, t)^\T B(x, t)^\T B(x, t) Y(x, t) \tilde{\alpha} \:.
\end{equation}
Furthermore, 
Theorem~\ref{thm:integral} states that for any $t$ in the maximal interval of existence,
\begin{align*}
    \int_0^t \norm{B(x(s), s) Y(x(s), s) \tilde{\alpha}(s)}^2 \: ds \leq \norm{\hat{\alpha}(0) - \alpha}^2\:.
\end{align*}
In this sense, the least-squares algorithms in Section~\ref{ssec:ols}
can be seen as an instance of the integral form of velocity gradient. Because we consider the deterministic setting here, we can state a stronger result: by Barbalat's Lemma, this $O(1)$ guarantee on the prediction regret also implies that the function approximation error $B(x(t), t)Y(x(t), t)\tilde{\alpha}(t) \rightarrow 0$.
Furthermore,
the next proposition shows how to turn this $O(1)$ prediction regret bound into an
$O(1)$ bound on the control regret.
\begin{proposition}
Suppose $f(x, t)$  admits a continuously differentiable Lyapunov function $Q(x, t)$ satisfying for some positive $\rho, \mu, L_Q$:
\begin{enumerate}
    \item $Q(0, t) = 0$ and $Q(x, t) \geq \mu\norm{x}^2$ for all $x,t$.
    \item $x \mapsto \nabla_x Q(x, t)$ is $L_Q$-Lipschitz for all $t$.
    \item For all $x, t$, $\ip{\nabla_x Q(x, t)}{f(x, t)} + \frac{\partial Q}{\partial t}(x, t)\leq - \rho Q(x, t)$.
\end{enumerate}
Let $u(x, \xi, t)$ and $g(x, \xi, t)$ be continuous functions, and consider the dynamics:
\begin{align*}
    \dot{x}(t) &= f(x, t) + u(x, \xi, t) \:, \\
    \dot{\xi}(t) &= g(x, \xi, t) \:.
\end{align*}
For every $t \geq 0$ in the maximal interval of existence $I(x(0), \xi(0))$, we have:
\begin{align}
    \norm{x(t)} \leq \sqrt{\frac{Q(x(0), 0)}{\mu} + \frac{L_Q^2}{4\mu(1-\gamma)\rho} \int_0^t \norm{u(x(s), \xi(s), s)}^2 \: ds} \:. \label{eq:xt_norm_bound_l2}
\end{align}
Furthermore, for all $T \in I(x(0), \xi(0))$, we have:
\begin{align}
    \int_0^T \norm{x(t)}^2 \: dt \leq \frac{Q(x(0), 0)}{\mu \gamma \rho} + \frac{L_Q^2}{4 \mu^2(1-\gamma) \gamma \rho^2} \int_0^T \norm{u(x(t), \xi(t), t)}^2 \: dt \:. \label{eq:regret_bound_l2}
\end{align}
Finally, suppose that for all $t \in I(x(0), \xi(0))$ the following inequality holds:
\begin{align*}
    \int_0^t \norm{u(x(s), \xi(s), s)}^2 \: ds \leq B_0 \:.
\end{align*}
Then, the solution $(x(t), \xi(t))$ exists for all $t \geq 0$, and therefore:
\begin{align}
    \int_0^\infty \norm{x(t)}^2 \: dt \leq \frac{Q(x(0), 0)}{\mu \gamma \rho} + \frac{L_Q^2 B_0}{4 \mu^2(1-\gamma) \gamma \rho^2} \:. \label{eq:regret_bound_l2_inf}
\end{align}
\end{proposition}
\begin{proof}
Since zero is a global minimum of the map $x \mapsto Q(x, t)$ for all $t$, we have that
$\nabla_x Q(0, t) = 0$ for all $t$.
Therefore, for any $\varepsilon > 0$:
\begin{align*}
    \frac{d}{dt}Q(x, t) &= \ip{\nabla_x Q(x, t)}{f(x, t) + u(x, t)} + \frac{\partial Q}{\partial t}(x, t) \\
    &=  \ip{\nabla_x Q(x, t)}{f(x, t)} + \frac{\partial Q}{\partial t}(x, t) + \ip{\nabla_x Q(x, t)}{u(x, t)} \\
    &\leq - \rho Q(x, t) + \norm{\nabla_x Q(x, t)}\norm{u(x, t)} \\
    &\leq - \rho Q(x, t) + L_Q \norm{x(t)} \norm{u(x, t)} \\
    &\leq - \rho Q(x, t) + \frac{\varepsilon L_Q^2}{2} \norm{x(t)}^2 + \frac{1}{2\varepsilon}\norm{u(x, t)}^2 \\
    &\leq - \gamma\rho Q(x, t) + \left[ - (1-\gamma) \rho + \frac{\varepsilon L_Q^2}{2 \mu} \right] Q(x, t) + \frac{1}{2\varepsilon}\norm{u(x, t)}^2 \:.
\end{align*}
Setting $\varepsilon = 2\mu (1-\gamma)\rho/L_Q^2$,
\begin{align*}
    \frac{d}{dt} Q(x, t) \leq -\gamma\rho Q(x, t) + \frac{L_Q^2}{4 \mu (1-\gamma) \rho} \norm{u(x, t)}^2 \:.
\end{align*}
By the comparison lemma,
\begin{align*}
    \mu\norm{x(t)}^2 &\leq Q(x(t), t) \leq e^{-\gamma \rho t} Q(x(0), 0) + \frac{L_Q^2}{4\mu(1-\gamma)\rho} \int_0^t e^{-\gamma \rho(t-s)} \norm{u(x(s), s)}^2 \: ds \:.
\end{align*}
This establishes \eqref{eq:xt_norm_bound_l2}.
Furthermore, integrating the above inequality from zero to $T$,
\begin{align*}
    \int_0^T \norm{x(t)}^2 \: dt &\leq \frac{Q(x(0), 0)}{\mu} \int_0^T e^{- \gamma \rho t} \: dt + \frac{L_Q^2}{4\mu^2(1-\gamma)\rho} \int_0^T \int_0^t e^{-\gamma\rho(t-s)}\norm{u(x(s), s)}^2 \: ds \: dt \\
    &= \frac{Q(x(0), 0)}{\mu} \int_0^T e^{- \gamma \rho t} \: dt + \frac{L_Q^2}{4\mu^2(1-\gamma)\rho} \int_0^T \left[ \int_0^{T-t} e^{-\gamma\rho s} \: ds \right]\norm{u(x(t), t)}^2 \: dt \\
    &\leq \frac{Q(x(0), 0)}{\mu \gamma \rho} + \frac{L_Q^2}{4 \mu^2(1-\gamma) \gamma \rho^2} \int_0^T \norm{u(x(t), t)}^2 \: dt \:.
\end{align*}
This establishes \eqref{eq:regret_bound_l2}.
The claim \eqref{eq:regret_bound_l2_inf} follows from
\eqref{eq:xt_norm_bound_l2}, \eqref{eq:regret_bound_l2}, and
Theorem 3.3 of \cite{khalilbook}.
\end{proof}

We conclude this section by noting that \eqref{eqn:gflow} cannot be directly implemented
due to the dependence on $\tilde{\alpha}(t)$. In discrete-time, this can be remedied as described in Section~\ref{ssec:ols}. In continuous-time, additional structural requirements are needed, which we briefly describe. Because the quantity $B(x(t), t) Y(x(t), t) \tilde{\alpha}(t)$ is contained in $\dot{x}$, the update \eqref{eqn:gflow} can be implemented through the proportional-integral construction (see e.g., \cite{astolfi03immersion,boffi2020implicit})
\begin{align*}
    \hat{\alpha}(t) &= \bar{\alpha}(t) - Y^\T(x(t), t)B^\T(x(t), t)x(t) + \psi(x(t)) \:, \\
    \dot{\bar{\alpha}}(t) &= Y^\T(x(t), t)B^\T(x(t), t)f(x(t), t) + \frac{\partial \left[Y^\T(x(t), t)B^\T(x(t), t)\right]}{\partial t}x(t) \:.
\end{align*}
Here, $\psi(x)$ is a function that satisfies
\begin{equation*}
    \frac{\partial \psi(x)}{\partial x_i} = \frac{\partial \left[Y(x, t)^\T B(x, t)^\T\right]}{\partial x_i}x \:,
\end{equation*}
i.e., $\frac{\partial \left[Y(x, t)^\T B(x, t)^\T\right]}{\partial x_i}x$ must be the gradient of some auxiliary function $\psi(x)$. In general, this is a strong requirement that may not be satisfied by the system.
\section{Discrete-Time Stability of Zero-Order Hold Closed-Loop Systems}
\label{sec:app:c2d}

In this section, we study under what conditions 
the stability behavior of a continuous-time system
is preserved under discrete sampling. 
In particular, we consider the following continuous-time system $f(x, u, t)$
with a continuous-time feedback law $\pi(x, t)$:
\begin{align*}
    \dot{x}(t) = f(x(t), \pi(x(t), t), t) \:.
\end{align*}
We are interested in understanding the effect of a discrete implementation for the control law $\pi$ via
a zero-order hold at resolution $\tau$, specifically:
\begin{align*}
    \dot{x}(t) = f(x(t), \pi(x(\floor{t/\tau}\tau), \floor{t/\tau}\tau), t) \:.
\end{align*}
We will view this zero-order hold as inducing an associated discrete-time system.
Let the flow map $\Phi(x, s, t)$ denote the solution $\xi(t)$
of the dynamics 
\begin{align*}
   \dot{\xi}(t) = f(\xi(t), \pi(x, s), t) \:, \:\: \xi(s) = x \:.
\end{align*}
For simplicity, we assume in this section that the solution $\Phi(x, s, t)$ exists and is unique.
The closed-loop discrete-time system we consider is
\begin{align*}
    x_{t+1} = g(x_t, t) := \Phi(x_t, \tau t, \tau (t+1)) \:.
\end{align*}
We address two specific questions.
First, if $Q(x, t)$ is a Lyapunov function for $f(x, \pi(x, t), t)$,
when does $(x, t) \mapsto Q(x, \tau t)$ remain a discrete-time Lyapunov 
function for $g(x, t)$?
Similarly, if $M(x, t)$ is a contraction metric for $f(x, \pi(x, t), t)$,
when does $(x, t) \mapsto M(x, \tau t)$
remain a discrete-time contraction metric
for $g(x, t)$? To do so, we will derive upper bounds on the sampling rate to ensure preservation of these stability properties. For simplicity, we perform our analysis at fixed resolution, but irregularly sampled time points may also be used so long as they satisfy our restrictions.

Before we begin our analysis, 
we start with a regularity assumption on both the dynamics $f$ and the
policy $\pi$.
\begin{definition}
\label{def:regular_dynamics}
Let $f(x, u, t)$ and $\pi(x, t)$ be a dynamics and a policy.
We say that $(f, \pi)$ is $(L_f, L_\pi)$-regular if
$f \in C^2$, $\pi \in C^0$, and the following conditions hold:
\begin{enumerate}
    \item $f(0, 0, t) = 0$ \ for all $t$.
    \item $\pi(0, t) = 0$  \ for all $t$.
    \item $\max\left\{ \bignorm{ \frac{\partial f}{\partial x}(x, u, t) }, 
    \bignorm{ \frac{\partial f}{\partial u}(x, u, t) },
    \bignorm{ \frac{\partial^2 f}{\partial x \partial t}(x, u, t) },
    \bignorm{ \frac{\partial^2 f}{\partial u \partial t}(x, u, t) },
    \bignorm{ \frac{\partial^2 f}{\partial x^2}(x, u, t)}
    \right\} \leq L_f$ for all $x, u, t$.
    \item $\bignorm{ \frac{\partial \pi}{\partial x}(x, t) } \leq L_\pi$ for all $x, t$.
\end{enumerate}
\end{definition}

Our first proposition bounds how far the solution
$\Phi(x, s, s + \tau)$ deviates from the initial condition $x$ over a time period $\tau$.
Roughly speaking, the proposition states that the deviation is a constant factor of $\norm{x}$
as long as $\tau$ is on the order of $1/L_f$. Note that for notational simplicity a common bound $L_f$ is used Definition~\ref{def:regular_dynamics}, although our results extends immediately to finer individual bounds.
\begin{proposition}
\label{prop:flow_deviation}
Let $(f, \pi)$ be $(L_f, L_\pi)$-regular. 
Let the flow map $\Phi(x, s, t)$ denote the solution $\xi(t)$
of the dynamics 
\begin{align*}
   \dot{\xi}(t) = f(\xi(t), \pi(x, s), t) \:, \:\: \xi(s) = x \:.
\end{align*}
We have that for any $\tau > 0$:
\begin{align*}
    \norm{\Phi(x, s, s + \tau) - x} \leq (1 + 3L_\pi)( e^{L_f\tau} - 1) \norm{x} \:.
\end{align*}
As a consequence, we have:
\begin{align*}
    \norm{\Phi(x, s, s + \tau)} \leq (e^{L_f \tau} + 3 L_\pi (e^{L_f \tau} - 1))\norm{x} \leq (1 + 3 L_\pi) e^{L_f \tau} \norm{x} \:.
\end{align*}
\end{proposition}
\begin{proof}
The proof follows by a direct application of the comparison lemma.
We use the Lipschitz properties of both $f$ and $\pi$, which
are implied by the regularity assumptions,
to establish the necessary differential inequality.
Let $v(t) := \norm{\xi(t) - x}$.
We note for any signal $z(t)$, we have $\frac{d}{dt} \norm{z(t)} \leq \norm{\dot{z}}$.
Therefore, setting $\xi = \xi(t)$ to simplify the notation:
\begin{align*}
    \frac{d}{dt} v(t) &\leq \norm{ \dot{\xi}(t) } \\
    &= \norm{ f(\xi, \pi(x, s), t) } \\
    &= \norm{ f(\xi, \pi(x, s), t) - f(x, \pi(x, t), t) + f(x, \pi(x, t), t) - f(0, 0, t) } \\
    &\leq \norm{ f(\xi, \pi(x, s), t) - f(x, \pi(x, t), t) } + \norm{ f(x, \pi(x, t), t) - f(0, 0, t) } \\
    &=: T_1 + T_2 \:.
\end{align*}
Next,
\begin{align*}
    T_1 &= \norm{ f(\xi, \pi(x, s), t) - f(x, \pi(x, t), t) } \\
    &= \norm{ f(\xi, \pi(x, s), t) - f(x, \pi(x, s), t) + f(x, \pi(x, s), t) - f(x, \pi(x, t), t) } \\
    &\leq L_f \norm{\xi - x} + L_f \norm{\pi(x, s) - \pi(x, t)} \\
    &= L_f \norm{\xi - x} + L_f \norm{\pi(x, s) - \pi(0, s) + \pi(0, t) - \pi(x, t)} \\
    &\leq L_f \norm{\xi - x} + 2 L_f L_\pi \norm{x} \:.
\end{align*}
Also,
\begin{align*}
    T_2 &= \norm{ f(x, \pi(x, t), t) - f(0, 0, t) } \\
    &= \norm{ f(x, \pi(x, t), t) - f(0, \pi(x, t), t) + f(0, \pi(x, t), t) - f(0, 0, t) } \\
    &\leq L_f \norm{x} + L_f \norm{\pi(x, t)} \\
    &= L_f \norm{x} + L_f \norm{\pi(x, t) - \pi(0, t)} \\
    &\leq L_f (1 + L_\pi) \norm{x} \:.
\end{align*}
Therefore we have the following differential inequality:
\begin{align*}
    \frac{d}{dt} v(t) \leq L_f v(t) + L_f (1 + 3 L_\pi) \norm{x} \:.
\end{align*}
The claim now follows by the comparison lemma.
\end{proof}
The next proposition shows that the error of the
forward Euler approximation of the flow map $\Phi(x, s, s + \tau)$
and also its derivative $\frac{\partial \Phi}{\partial x}(x, s, s + \tau)$ 
scales as $O(\tau^2)$.
\begin{proposition}
\label{prop:forward_euler_approx}
Let $(f, \pi)$ be $(L_f, L_\pi)$-regular,
with $\min\{L_f, L_\pi\} \geq 1$.
Let $\Phi(x, s, t)$ be the solution $\xi(t)$ for the 
dynamics 
\begin{align*}
    \dot{\xi}(t) = f(\xi(t), \pi(x, s), t) \:, \:\: \xi(s) = x \:.
\end{align*}
Fix any $\tau > 0$.
We have that:
\begin{align}
    \norm{\Phi(x, s, s + \tau) - (x + \tau f(x, \pi(x, s), s))} \leq 5 \tau^2 L_f^2 L_\pi e^{L_f \tau} \norm{x} \:. \label{eq:forward_euler_approx}
\end{align}
We also have:
\begin{align}
    \bignorm{ \frac{\partial \Phi}{\partial x}(x, s, s + \tau) - \left(I + \tau \frac{\partial f}{\partial x}(x, \pi(x, s), s)\right)} \leq \frac{7\tau^2}{2} L_f^2 L_\pi e^{2L_f \tau} \max\{1, \norm{x}\} \:. \label{eq:forward_euler_variation_approx}
\end{align}
\end{proposition}
\begin{proof}
We first differentiate $\Phi(x, s, t)$ w.r.t.\ $t$ twice:
\begin{align*}
    \frac{\partial \Phi}{\partial t}(x, s, t) &= f(\xi(t), \pi(x, s), t) \:, \\
    \frac{\partial^2 \Phi}{\partial t^2}(x, s, t) &= \frac{df}{dt}(\xi(t), \pi(x, s), t) = \frac{\partial f}{\partial x}(\xi(t), \pi(x, s), t) f(\xi(t), \pi(x, s), t) + \frac{\partial f}{\partial t}(\xi(t), \pi(x, s), t) \:.
\end{align*}
By Taylor's theorem, there exists some $\iota \in [s, s + \tau]$ such that:
\begin{align*}
    &\Phi(x, s, s + \tau) = \Phi(x, s, s) + \frac{\partial \Phi}{\partial t} (x, s, s) \tau + \frac{\tau^2}{2} \frac{\partial^2 \Phi }{\partial t^2} (x, s, \iota) \\
    &= x + \tau f(x, \pi(x, s), s) + \frac{\tau^2}{2} \left(\frac{\partial f}{\partial x}(\xi(\iota), \pi(x, s), \iota) f(\xi(\iota), \pi(x, s), \iota) + \frac{\partial f}{\partial t}(\xi(\iota), \pi(x, s), \iota) \right)\:.
\end{align*}
In order to bound the error term above, we make a few intermediate calculations.
We use Proposition~\ref{prop:flow_deviation} to bound:
\begin{align*}
    \norm{ f(\xi(\iota), \pi(x, s), \iota) } &= \norm{ f(\xi(\iota), \pi(x, s), \iota) - f(0, \pi(x, s), \iota) + f(0, \pi(x, s), \iota) - f(0, 0, \iota) } \\
    &\leq L_f \norm{\xi(\iota)} + L_f \norm{\pi(x, s)} \\
    &= L_f \norm{\xi(\iota)} + L_f \norm{\pi(x, s) - \pi(0, s)} \\
    &\leq L_f \norm{\xi(\iota)} + L_f L_\pi \norm{x} \\
    &\leq L_f (1+3L_\pi) e^{L_f \tau} \norm{x} + L_f L_\pi \norm{x} \\
    &\leq L_f (1 + 4 L_\pi) e^{L_f \tau} \norm{x} \\
    &\leq 5 L_f L_\pi e^{L_f \tau} \norm{x} \:.
\end{align*}
Again we use Proposition~\ref{prop:flow_deviation}, along with the
fact that $\frac{\partial f}{\partial t}(0, 0, t) = 0$ for all $t$ due to the regularity assumptions on $f$,
to bound:
\begin{align*}
    \bignorm{\frac{\partial f}{\partial t}(\xi(\iota), \pi(x, s), \iota)} &= \bignorm{
    \frac{\partial f}{\partial t}(\xi(\iota), \pi(x, s), \iota)
    - \frac{\partial f}{\partial t}(0, \pi(x, s), \iota) 
    + \frac{\partial f}{\partial t}(0, \pi(x, s), \iota) - \frac{\partial f}{\partial t}(0, 0, \iota) } \\
    &\leq L_f \norm{\xi(\iota)} + L_f \norm{\pi(x, s)} \\
    &\leq L_f \norm{\xi(\iota)} + L_f L_\pi \norm{x} \\
    &\leq 5 L_f L_\pi e^{L_f \tau} \norm{x} \:.
\end{align*}
Therefore:
\begin{align*}
    &\norm{\Phi(x, s, s + \tau) - (x + \tau f(x, \pi(x, s), s))} \\
    &\leq \bignorm{ \frac{\tau^2}{2} \left(\frac{\partial f}{\partial x}(\xi(\iota), \pi(x, s), \iota) f(\xi(\iota), \pi(x, s), \iota) + \frac{\partial f}{\partial t}(\xi(\iota), \pi(x, s), \iota) \right) } \\
    &\leq 5 \tau^2 L_f^2 L_\pi e^{L_f \tau} \norm{x} \:.
\end{align*}
This establishes \eqref{eq:forward_euler_approx}.

Next, let $\Psi(x, s, t)$ be the solution $\Xi(t)$ for the matrix-valued dynamics:
\begin{align*}
    \dot{\Xi}(t) = \frac{\partial f}{\partial x}(\xi(t), \pi(x, s), t) \Xi(t) \:, \:\: \Xi(s) = I \:.
\end{align*}
A standard result in the theory of ordinary differential equations
states that $\frac{\partial \Phi}{\partial x}(x, s, t) = \Psi(x, s, t)$.
We can bound the norm $\norm{\Psi(x, s, t)}$ as follows:
\begin{align*}
    \norm{\Psi(x, s, t)} &= \bignorm{\exp\left( \int_s^t \frac{\partial f}{\partial x}(\xi(\tau), \pi(x, s), \tau) \: d\tau \right)} \\
    &\leq \exp\left( \int_s^t \bignorm{ \frac{\partial f}{\partial x}(\xi(\tau), \pi(x, s), \tau) } \: d\tau \right) \leq \exp( L_f (t - s) ) \:.
\end{align*}
Furthermore, differentiating $\Psi$ w.r.t. $t$ twice:
\begin{align*}
    \frac{\partial \Psi}{\partial t}(x, s, t) &=  \frac{\partial f}{\partial x}(\xi(t), \pi(x, s), t) \Xi(t) \:, \\
    \frac{\partial^2 \Psi}{\partial t^2}(x, s, t) &= \frac{d}{dt}\left(\frac{\partial f}{\partial x}(\xi(t), \pi(x, s), t) \Xi(t)\right)\\
    &= \frac{\partial f}{\partial x}(\xi(t), \pi(x, s), t) \frac{\partial f}{\partial x}(\xi(t), \pi(x, s), t) \Xi(t) \\
    &\qquad + \left(  \frac{\partial^2 f}{\partial x^2}(\xi(t), \pi(x, s), t) f(\xi(t), \pi(x, s), t) + \frac{\partial^2 f}{\partial t \partial x} (\xi(t), \pi(x, s), t) \right)\Xi(t) \:.
\end{align*}
By Taylor's theorem, there exists an $\iota \in [s, s + \tau]$ such that:
\begin{align*}
    \Psi(x, s, s + \tau) &= \Psi(x, s, s) + \tau \frac{\partial \Psi}{\partial t}(x, s, s) + \frac{\tau^2}{2} \frac{\partial^2 \Psi}{\partial t^2}(x, s, \iota) \\
    &= I + \tau \frac{\partial f}{\partial x}(x, \pi(x, s), s) + \frac{\tau^2}{2} \frac{\partial^2 \Psi}{\partial t^2}(x, s, \iota) \:.
\end{align*}
Using the estimate on $\norm{\Psi(x, s, t)}$ above, we bound:
\begin{align*}
    \bignorm{ \frac{\partial f}{\partial x}(\xi(\iota), \pi(x, s), \iota) \frac{\partial f}{\partial x}(\xi(\iota), \pi(x, s), \iota) \Xi(\iota) } \leq L_f^2 e^{L_f \tau} \:. 
\end{align*}
Furthermore by the estimates on $\norm{\Psi(x, s, t)}$ and $\norm{f(\xi(t), \pi(x, s), t)}$,
\begin{align*}
    &\bignorm{ \left(  \frac{\partial^2 f}{\partial x^2}(\xi(\iota), \pi(x, s), \iota) f(\xi(\iota), \pi(x, s), \iota) + \frac{\partial^2 f}{\partial t \partial x} (\xi(\iota), \pi(x, s), \iota) \right)\Xi(\iota) } \\
    &\leq L_f (1 + \norm{ f(\xi(\iota), \pi(x, s), \iota)}) e^{L_f \tau} \\
    &\leq L_f (1 + 5 L_f L_\pi e^{L_f \tau} \norm{x}) e^{L_f \tau} \\
    &\leq 6 L_f^2 L_\pi e^{2L_f \tau} \max\{1, \norm{x}\} \:.
\end{align*}
Therefore:
\begin{align*}
    \bignorm{ \Psi(x, s, s + \tau) - \left(I + \tau \frac{\partial f}{\partial x}(x, \pi(x, s), s)\right)} 
    &\leq \bignorm{ \frac{\tau^2}{2} \frac{\partial^2 \Psi}{\partial t^2}(x, s, \iota) } \\
    &\leq \frac{\tau^2}{2} \left[ L_f^2 e^{L_f \tau} + 6 L_f^2 L_\pi e^{2L_f \tau} \max\{1, \norm{x}\} \right] \\
    &\leq \frac{7\tau^2}{2} L_f^2 L_\pi e^{2L_f \tau} \max\{1, \norm{x}\} \:.
\end{align*}
This establishes \eqref{eq:forward_euler_variation_approx}.
\end{proof}

Our first main result gives conditions on $\tau$
for which Lyapunov stability is preserved with zero-order holds.
\begin{theorem}
\label{thm:stability_zoh}
Let $(f, \pi)$ be $(L_f, L_\pi)$-regular,
with $\min\{L_f, L_\pi\} \geq 1$.
Let $\Phi(x, s, t)$ be the solution $\xi(t)$ for the 
dynamics 
\begin{align*}
    \dot{\xi}(t) = f(\xi(t), \pi(x, s), t) \:, \:\: \xi(s) = x \:.
\end{align*}
Let $Q(x, t) \in C^2$ be a Lyapunov function that satisfies, for positive $\mu, \rho$ and
$L_Q \geq 1$, the conditions:
\begin{enumerate}
    \item $Q(0, t) = 0$ for all $t$.
    \item $Q(x, t) \geq \mu \norm{x}^2$ for all $x, t$.
    \item $\ip{\nabla_x Q(x, t)}{f(x, \pi(x, t), t)} + \frac{\partial Q}{\partial t}(x, t) \leq -\rho Q(x, t)$ for all $x, t$.
    \item $\bignorm{\frac{\partial^2 Q}{\partial x^2}(x, t)} \leq L_Q$ for all $x, t$.
    \item $\bignorm{ \frac{\partial^2 Q}{\partial t \partial x}(x, t) - \frac{\partial^2 Q}{\partial t \partial x}(y, t) } \leq L_Q \norm{x-y}$ for all $x,y,t$.
    \item $\bigabs{\frac{\partial^2 Q}{\partial t^2}(x, t)} \leq L_Q \norm{x}^2$ for all $x, t$.
\end{enumerate}
Fix a $\gamma \in (0, 1)$ and $\tau > 0$.
Define the discrete-time system $g(x, t) := \Phi(x, \tau t, \tau (t+1))$.
As long as $\tau$ satisfies:
\begin{align*}
    \tau \leq \min\left\{ \frac{1}{L_f}, \frac{1}{\gamma\rho},  \frac{2(1-\gamma) \rho \mu}{895 L_Q L_f^2 L_\pi^2}   \right\} \:,
\end{align*}
then the function $V(x, t) := Q(x, \tau t)$ 
is a valid Lyapunov function for $g(x, t)$ with rate $(1-\gamma\tau\rho)$, i.e.,\ for all $x, t$:
\begin{align}
    V(g(x, t), t+1) \leq (1 - \gamma \tau \rho) V(x, t) \:. \label{eq:stability_zoh_lyapunov}
\end{align}
\end{theorem}
\begin{proof}
We define the function $h(t) := Q(\Phi(x, s, t), t)$.
Differentiating $h$ twice,
\begin{align*}
    \frac{\partial h}{\partial t}(t) &= \frac{\partial Q}{\partial x}(\Phi(x, s, t), t)f(\xi(t), \pi(x, s), t) + \frac{\partial Q}{\partial t}(\Phi(x, s, t), t) \:, \\
    \frac{\partial^2 h}{\partial t^2}(t) &= \frac{\partial Q}{\partial x}(\Phi(x, s, t), t)\left( \frac{\partial f}{\partial x}(\xi(t), \pi(x, s), t) f(\xi(t), \pi(x, s), t) + \frac{\partial f}{\partial t}(\xi(t), \pi(x, s), t) \right) \\
    &\qquad + \left( \frac{\partial^2 Q}{\partial x^2}(\Phi(x, s, t), t) f(\xi(t), \pi(x, s), t) + \frac{\partial^2 Q}{\partial t \partial x} (\Phi(x, s, t), t)  \right) f(\xi(t), \pi(x, s), t) \\
    &\qquad + \frac{\partial^2 Q}{\partial x \partial t}(\Phi(x, s, t), t) f(\xi(t), \pi(x, s), t) + \frac{\partial^2 Q}{\partial t^2}(\Phi(x, s, t), t) \:.
\end{align*}
By Taylor's theorem, there exists an $\iota \in [s, s + \tau]$ such that:
\begin{align}
    h(s + \tau) &= h(s) + \tau \frac{\partial h}{\partial t}(s) + \frac{\tau^2}{2} \frac{\partial^2 h}{\partial t^2}(\iota) \nonumber \\
    &= Q(x, s) + \tau \left[ \ip{\nabla_x Q(x, s)}{f(x, \pi(x, s), s)} + \frac{\partial Q}{\partial t}(x, s) \right] + \frac{\tau^2}{2} \frac{\partial^2 h}{\partial t^2}(\iota) \nonumber \\
    &\leq Q(x, s) - \tau \rho Q(x, s) + \frac{\tau^2}{2} \frac{\partial^2 h}{\partial t^2}(\iota) \nonumber \\
    &= (1-\tau \rho \gamma) Q(x, s) - \tau \rho (1-\gamma) Q(x, s) + \frac{\tau^2}{2} \frac{\partial^2 h}{\partial t^2}(\iota) \nonumber \\
    &\leq (1-\tau \rho \gamma) Q(x, s) - \tau \rho (1-\gamma) \mu \norm{x}^2 + \frac{\tau^2}{2} \frac{\partial^2 h}{\partial t^2}(\iota) \:. \label{eq:stability_zoh_taylor}
\end{align}
Above, the first inequality follows from the continuous-time Lyapunov condition.
The remainder of the proof
focuses on estimating a bound for $\bigabs{\frac{\partial^2 h}{\partial t^2}(t)}$.
First, we collect a few useful facts.
Since zero is a global minimum of $x \mapsto Q(x, t)$ for every $t$,
we have that $\frac{\partial Q}{\partial x}(0, t) = 0$ for every $t$. Therefore:
\begin{align*}
    \bignorm{ \frac{\partial Q}{\partial x}(\xi(\iota), \iota) } &= \bignorm{ \frac{\partial Q}{\partial x}(\xi(\iota), \iota) - \frac{\partial Q}{\partial x}(0, \iota) } \leq L_Q \norm{\xi(\iota)} \\
    &\leq L_Q (1 + 3 L_\pi) e^{L_f \tau} \norm{x} \leq 4 L_Q L_\pi e^{L_f \tau} \norm{x} \:.
\end{align*}
Above, the second to last inequality follows
from Proposition~\ref{prop:flow_deviation}.
Next, the proof of Proposition~\ref{prop:forward_euler_approx}
derives the following estimates:
\begin{align*}
    \max\left\{\norm{f(\xi(\iota), \pi(x, s), \iota)}, \bignorm{\frac{\partial f}{\partial t}(\xi(t), \pi(x, s), t)}\right\} \leq 5 L_f L_\pi e^{L_f \tau} \norm{x} \:.
\end{align*}
Using these estimates, we can bound:
\begin{align*}
    &\bigabs{\frac{\partial Q}{\partial x}(\xi(\iota), \iota)\left( \frac{\partial f}{\partial x}(\xi(\iota), \pi(x, s), \iota) f(\xi(\iota), \pi(x, s), \iota) + \frac{\partial f}{\partial t}(\xi(\iota), \pi(x, s), \iota) \right) } \\
    &\leq 4 L_Q L_\pi e^{L_f \tau} \norm{x} \cdot \left[ 5 L_f^2 L_\pi e^{L_f \tau} \norm{x} + 5 L_f L_\pi e^{L_f \tau} \norm{x} \right] \\
    &\leq 40 L_Q L_f^2 L_\pi^2 e^{2 L_f \tau} \norm{x}^2 \:.
\end{align*}
Next, we observe that $\frac{\partial^2 Q}{\partial t \partial x}(0, t) = 0$ for all $t$,
which allows us to bound:
\begin{align*}
    &\bigabs{  \left( \frac{\partial^2 Q}{\partial x^2}(\xi(\iota), \iota) f(\xi(\iota), \pi(x, s), \iota) + \frac{\partial^2 Q}{\partial t \partial x} (\xi(\iota), \iota)  \right) f(\xi(\iota), \pi(x, s), \iota)  } \\
    &\leq L_Q \norm{f(\xi(\iota), \pi(x, s), \iota)}^2 + \bignorm{ \frac{\partial^2 Q}{\partial t \partial x} (\xi(\iota), \iota) }\norm{f(\xi(\iota), \pi(x, s), \iota)} \\
    &\leq L_Q \norm{f(\xi(\iota), \pi(x, s), \iota)}^2 + L_Q \norm{\xi(\iota)} \norm{f(\xi(\iota), \pi(x, s), \iota)} \\
    &\leq 25 L_Q L_f^2 L_\pi^2 e^{2 L_f \tau} \norm{x}^2 + L_Q (1+3L_\pi) e^{L_f\tau} \norm{x} \cdot 5 L_f L_\pi e^{L_f \tau} \norm{x} \\
    &\leq 45 L_Q L_f^2 L_\pi^2 e^{2L_f \tau} \norm{x}^2 \:.
\end{align*}
Finally, we bound:
\begin{align*}
    &\bigabs{ \frac{\partial^2 Q}{\partial x \partial t}(\xi(\iota), \iota) f(\xi(\iota), \pi(x, s), \iota) + \frac{\partial^2 Q}{\partial t^2}(\xi(\iota), \iota) } \\
    &\leq \bignorm{\frac{\partial^2 Q}{\partial x \partial t}(\xi(\iota), \iota)} \norm{ f(\xi(\iota), \pi(x, s), \iota)  } + L_Q \norm{\xi(\iota)}^2 \\
    &\leq L_Q \norm{\xi(\iota)} \norm{ f(\xi(\iota), \pi(x, s), \iota)  } + L_Q \norm{\xi(\iota)}^2 \\
    &\leq L_Q (1+3L_\pi) e^{L_f\tau} \norm{x} \cdot 5 L_f L_\pi e^{L_f \tau} \norm{x} + L_Q (1+3L_\pi)^2 e^{2L_f\tau} \norm{x}^2 \\
    &\leq 36 L_Q L_f L_\pi^2 e^{2 L_f \tau} \norm{x}^2 \:.
\end{align*}
Combining these estimates:
\begin{align*}
    \bigabs{ \frac{\partial^2 h}{\partial t^2}(\iota) } &\leq 40 L_Q L_f^2 L_\pi^2 e^{2 L_f \tau} \norm{x}^2 + 45 L_Q L_f^2 L_\pi^2 e^{2L_f \tau} \norm{x}^2 + 36 L_Q L_f L_\pi^2 e^{2 L_f \tau} \norm{x}^2 \\
    &\leq 121 L_Q L_f^2 L_\pi^2 e^{2L_f \tau} \norm{x}^2 \leq 121 e^2 L_Q L_f^2 L_\pi^2 \norm{x}^2 \:,
\end{align*}
where the last inequality follows since we assume $\tau \leq 1/L_f$.
Continuing from \eqref{eq:stability_zoh_taylor},
\begin{align*}
    h(s + \tau) &\leq (1-\tau \rho \gamma) Q(x, s) - \tau \rho (1-\gamma) \mu \norm{x}^2 + \frac{\tau^2}{2} \frac{\partial^2 h}{\partial t^2}(\iota) \\
    &\leq (1-\tau \rho \gamma) Q(x, s) - \tau \rho (1-\gamma) \mu \norm{x}^2 + \tau^2 \frac{121}{2} e^2 L_Q L_f^2 L_\pi^2 \norm{x}^2 \\
    &= (1-\tau \rho \gamma) Q(x, s) + \left[ -\rho(1-\gamma)\mu + \tau \frac{121}{2} e^2 L_Q L_f^2 L_\pi^2 \right] \tau \norm{x}^2 \:.
\end{align*}
Hence as long as 
\begin{align*}
    -\rho(1-\gamma)\mu + \tau \frac{121}{2} e^2 L_Q L_f^2 L_\pi^2 \leq 0 \:,
\end{align*}
then \eqref{eq:stability_zoh_lyapunov} holds.
It is straightforward to check that the following condition suffices:
\begin{align*}
    \tau \leq \frac{2 \rho (1-\gamma) \mu}{895 L_Q L_f^2 L_\pi^2} \:.
\end{align*}
The claim now follows.
\end{proof}
Before we proceed, we briefly describe the condition
$\bigabs{\frac{\partial^2 Q}{\partial t^2}(x, t)} \leq L_Q \norm{x}^2$
in Theorem~\ref{thm:stability_zoh}.
Let us suppose that $Q \in C^4$,
and define the function $\psi(x, t) := \frac{\partial^2 Q}{\partial t^2}(x, t)$.
By Taylor's theorem, there exists a $\tilde{x}$ 
satisfying $\norm{\tilde{x}} \leq \norm{x}$ such that:
\begin{align*}
    \psi(x, t) = \psi(0, t) + \frac{\partial \psi}{\partial x}(0, t) x + \frac{1}{2} x^\T \frac{\partial^2 \psi}{\partial x^2}(\tilde{x}, t) x\:.
\end{align*}
First, since $Q(0, t) = 0$ for all $t$, we know that
$\frac{\partial Q}{\partial t}(0, t) = 0$ for all $t$.
Repeating this argument yields that $\frac{\partial^2 Q}{\partial t^2}(0, t) = 0$ for all $t$.
Next, we know that $\frac{\partial Q}{\partial x}(0, t) = 0$ for all $t$
because $x=0$ is a global minima of the function $x \mapsto Q(x, t)$ for all $t$.
This means that $\frac{\partial^2 Q}{\partial t \partial x}(0, t) = 0$ for all $t$.
Repeating this argument yields $\frac{\partial^3 Q}{\partial t^2 \partial x}(0, t) = 0$
for all $t$. Swapping the order of differentiation yields that
$\psi(0, t) = \frac{\partial}{\partial x} \frac{\partial^2 Q}{\partial t^2}(0, t) = \frac{\partial^3 Q}{\partial t^2 \partial x}(0, t) = 0$.
Hence:
\begin{align*}
    \bigabs{\psi(x, t)} \leq \frac{1}{2} \bignorm{\frac{\partial^2 \psi}{\partial x^2}(\tilde{x}, t)} \norm{x}^2 \:.
\end{align*}
Therefore if $\frac{\partial^2 \psi}{\partial x^2}$ is uniformly bounded, then this condition
holds.

Our next main result gives conditions on $\tau$ for which
contraction is preserved with zero-order holds.
\begin{theorem}
\label{thm:contraction_zoh}
Let $(f, \pi)$ be $(L_f, L_\pi)$-regular,
with $\min\{L_f, L_\pi\} \geq 1$.
Let $\Phi(x, s, t)$ be the solution $\xi(t)$ for the 
dynamics 
\begin{align*}
    \dot{\xi}(t) = f(\xi(t), \pi(x, s), t) \:, \:\: \xi(s) = x \:.
\end{align*}
Let $M(x, t) \in C^2$ be a positive definite metric 
that satisfies, for positive $\mu,\lambda$ and $\min\{ L, L_M\} \geq 1$,
the conditions:
\begin{enumerate}
    \item $\mu I \preceq M(x, t) \preceq L I$ for all $x, t$.
    \item $\frac{\partial f}{\partial x}(x, \pi(x, t), t)^\T M(x, t) + M(x, t) \frac{\partial f}{\partial x}(x, \pi(x, t), t) + \dot{M}(x, t) \preceq - 2 \lambda M(x, t)$ for all $x, t$.
    \item $\max\left\{ \bignorm{\frac{\partial M}{\partial x}(x, t)}, \bignorm{\frac{\partial M}{\partial t}(x, t)},
    \bignorm{\frac{\partial^2 M}{\partial^2 x}(x, t)},
    \bignorm{\frac{\partial^2 M}{\partial x \partial t}(x, t)},
    \bignorm{\frac{\partial^2 M}{\partial t^2}(x, t)}
    \right\} \leq L_M$ for all $x, t$.
\end{enumerate}
Pick a $\gamma \in (0, 1)$, $\tau > 0$, and $D \geq 1$.
Define the discrete system $g(x, t) := \Phi(x, \tau t, \tau (t+1))$.
As long as $\tau$ satisfies:
\begin{align*}
    \tau \leq \min\left\{ \frac{1}{L_f}, \frac{1}{2\lambda \gamma}, \frac{2 \lambda (1-\gamma) \mu}{ 1463 D^2 L L_M L_f^2 L_\pi^2 } \right\} \:,
\end{align*}
then for any $x$ satisfying $\norm{x} \leq D$ and for any $t$, we have
that $g(x, t)$ is contracting in the metric $V(x, t) := M(x, \tau t)$
with rate $(1-2\lambda \gamma \tau)$, i.e.,
\begin{align}
    \frac{\partial g}{\partial x}(x, t)^\T V(g(x, t), t+1)  \frac{\partial g}{\partial x}(x, t) \preceq (1 - 2 \lambda \gamma \tau) V(x, t) \:. \label{eq:contraction_zoh}
\end{align}
\end{theorem}
\begin{proof}
We fix an $x$ satisfying $\norm{x} \leq D$.
Let $\Delta(x, s, \tau) := \frac{\partial \Phi}{\partial x}(x, s, s + \tau) - \left( I + \tau \frac{\partial f}{\partial x}(x, \pi(x, s), s)\right)$
denote the error of the forward Euler approximation to the variational dynamics.
From Proposition~\ref{prop:forward_euler_approx}, we have the bound:
\begin{align*}
    \norm{\Delta(x, s, \tau)} \leq \frac{7\tau^2}{2} L_f^2 L_\pi e^{2L_f \tau} \max\{1, \norm{x} \} \leq \frac{7 e^2}{2} \tau^2 D L_f^2 L_\pi  \:.
\end{align*}
The last inequality follows from our assumption that $\tau \leq 1/L_f$ and $\norm{x} \leq D$.
Therefore we can expand out the LHS of \eqref{eq:contraction_zoh} as follows:
\begin{align*}
    &\frac{\partial \Phi}{\partial x}(x, s, s + \tau)^\T M( \Phi(x, s, s + \tau), s + \tau ) \frac{\partial \Phi}{\partial x}(x, s, s + \tau) \\
    &= \left( I + \tau \frac{\partial f}{\partial x}(x, \pi(x, s), s) + \Delta(x, s, \tau) \right)^\T  M( \Phi(x, s, s + \tau), s + \tau ) \left( I + \tau \frac{\partial f}{\partial x}(x, \pi(x, s), s) + \Delta(x, s, \tau) \right) \\
    &= \left( I + \tau \frac{\partial f}{\partial x}(x, \pi(x, s), s) \right)^\T  M( \Phi(x, s, s + \tau), s + \tau ) \left( I + \tau \frac{\partial f}{\partial x}(x, \pi(x, s), s) \right) \\
    &\qquad + \left( I + \tau \frac{\partial f}{\partial x}(x, \pi(x, s), s) \right)^\T  M( \Phi(x, s, s + \tau), s + \tau ) \Delta(x, s, \tau) \\
    &\qquad + \Delta(x, s, \tau)^\T  M( \Phi(x, s, s + \tau), s + \tau ) \left( I + \tau \frac{\partial f}{\partial x}(x, \pi(x, s), s) \right) \\
    &\qquad + \Delta(x, s, \tau)^\T M( \Phi(x, s, s + \tau), s + \tau ) \Delta(x, s,  \tau) \\
    &=: T_1 + T_2 + T_3 + T_4 \:.
\end{align*}
We first bound $T_2$, $T_3$, and $T_4$
using our estimate on $\norm{\Delta(x, s, \tau)}$
and the assumption that $\tau \leq 1/L_f$:
\begin{align*}
    \max\{\norm{T_2}, \norm{T_3}\} &\leq (1 + \tau L_f) L \norm{\Delta(x, s, \tau)} \leq (1 + \tau L_f) L  \cdot \frac{7\tau^2}{2} D L_f^2 L_\pi e^{2} \leq 7 e^2 \tau^2 D L L_f^2 L_\pi \:, \\
    \norm{T_4} &\leq L \norm{\Delta(x, s, \tau)}^2 \leq \frac{49 e^4}{4} \tau^4 D^2 L L_f^4 L_\pi^2 \leq \frac{49e^4}{4} \tau^2 D^2 L L_f^2 L_\pi^2 \:.
\end{align*}
It remains to bound $T_1$.
To do this, we define
$H(t) := M(\Phi(x, s, t), t)$,
and compute its first and second derivatives:
\begin{align*}
    \frac{\partial H}{\partial t}(t) &= \frac{\partial M}{\partial x}(\Phi(x, s, t), t) f(\xi(t), \pi(x, s), t) + \frac{\partial M}{\partial t}(\Phi(x, s, t), t) \:, \\
    \frac{\partial^2 H}{\partial t^2}(t) &= \frac{\partial M}{\partial x}(\Phi(x, s, t), t)\left( \frac{\partial f}{\partial x}(\xi(t), \pi(x, s), t) f(\xi(t), \pi(x, s), t) + \frac{\partial f}{\partial t}(\xi(t), \pi(x, s), t) \right)  \\
    &\qquad + \left( \frac{\partial^2 M}{\partial x^2}(\Phi(x, s, t), t) f(\xi(t), \pi(x, s), t)  + \frac{\partial^2 M}{\partial t \partial x}(\Phi(x, s, t), t) \right) f(\xi(t), \pi(x, s), t) \\
    &\qquad + \frac{\partial^2 M}{\partial x \partial t}(\Phi(x, s, t), t) f(\xi(t), \pi(x, s), t) + \frac{\partial^2 M}{\partial t^2}(\Phi(x, s, t), t) \:.
\end{align*}
By Taylor's theorem, there exists an $\iota \in [s, s + \tau]$ such that
\begin{align*}
     H(s + \tau) &= H(s) + \frac{\partial H}{\partial t}(s) \tau + \frac{\tau^2}{2} \frac{\partial^2 H}{\partial t^2}(\iota) \\
     &= M(x, s) + \tau \left( \frac{\partial M}{\partial x}(x, s) f(x, \pi(x, s), s) + \frac{\partial M}{\partial t}(x, s) \right) + \frac{\tau^2}{2} \frac{\partial^2 H}{\partial t^2}(\iota) \\
     &= M(x, s) + \tau \dot{M}(x, s) + \frac{\tau^2}{2} \frac{\partial^2 H}{\partial t^2}(\iota) \:.
\end{align*}
The proof of Proposition~\ref{prop:forward_euler_approx}
derives the following estimates:
\begin{align*}
    \max\left\{\norm{f(\xi(\iota), \pi(x, s), \iota)}, \bignorm{\frac{\partial f}{\partial t}(\xi(t), \pi(x, s), t)}\right\} \leq 5 D L_f L_\pi e^{L_f \tau} \:.
\end{align*}
Therefore:
\begin{align*}
    \bignorm{\frac{\partial^2 H}{\partial t^2}(\iota)} &\leq 10 D L_M L_f^2 L_\pi e^{L_f \tau} + 30 D^2 L_M L_f^2 L_\pi^2 e^{2L_f \tau} + 5 D L_M L_f L_\pi e^{L_f \tau} + L_M  \\
    &\leq 46 D^2 L_M L_f^2 L_\pi^2 e^{2L_f \tau} \:.
\end{align*}
Defining $\Delta_M(x, s, \tau) := M(\Phi(x, s, s + \tau), s + \tau) - (M(x, s) + \tau \dot{M}(x, s))$,
we have shown that:
\begin{align*}
    \norm{\Delta_M(x, s, \tau)} \leq \bignorm{\frac{\tau^2}{2} \frac{\partial^2 H}{\partial t^2}(\iota)} \leq 23 \tau^2 D^2 L_M L_f^2 L_\pi^2 e^{2L_f \tau} \leq 23 e^2 \tau^2 D^2 L_M L_f^2 L_\pi^2 \:,
\end{align*}
where the last inequality uses our assumption that $\tau \leq 1/L_f$.
We can now expand $T_1$ as follows:
\begin{align*}
    T_1 &= \left( I + \tau \frac{\partial f}{\partial x}(x, \pi(x, s), s) \right)^\T  M( \Phi(x, s, s + \tau), s + \tau ) \left( I + \tau \frac{\partial f}{\partial x}(x, \pi(x, s), s) \right) \\
    &= \left( I + \tau \frac{\partial f}{\partial x}(x, \pi(x, s), s) \right)^\T (M(x, s) + \tau \dot{M}(x, s)) \left( I + \tau \frac{\partial f}{\partial x}(x, \pi(x, s), s) \right) \\
    &\qquad + \left( I + \tau \frac{\partial f}{\partial x}(x, \pi(x, s), s) \right)^\T \Delta_M(x, s, \tau) \left( I + \tau \frac{\partial f}{\partial x}(x, \pi(x, s), s) \right) \\
    &=: T_{1,1} + T_{1,2} \:.
\end{align*}
We can bound $T_{1,2}$ by using our estimate on $\norm{\Delta_M(x, s, t)}$ and
the assumption that $\tau \leq 1/L_f$:
\begin{align*}
    \norm{T_{1,2}} \leq (1 + \tau L_f)^2 \norm{\Delta_M(x, s, \tau)} \leq 92 e^2 \tau^2 D^2 L_M L_f^2 L_\pi^2 \:.
\end{align*}
Next, we expand $T_{1,1}$ as follows:
\begin{align*}
    T_{1,1} &= \left( I + \tau \frac{\partial f}{\partial x}(x, \pi(x, s), s) \right)^\T (M(x, s) + \tau \dot{M}(x, s)) \left( I + \tau \frac{\partial f}{\partial x}(x, \pi(x, s), s) \right) \\
    &= M(x, s) + \tau \left(\frac{\partial f}{\partial x}(x, \pi(x, s), s)^\T M(x, s) + M(x, s) \frac{\partial f}{\partial x}(x, \pi(x, s), s) + \dot{M}(x, s) \right) \\
    &\qquad + \tau^2 \left(\frac{\partial f}{\partial x}(x, \pi(x, s), s)^\T \dot{M}(x, s) + \dot{M}(x, s)\frac{\partial f}{\partial x}(x, \pi(x, s), s) \right) \\
    &\qquad + \tau^2 \frac{\partial f}{\partial x}(x, \pi(x, s), s)^\T (M(x, s) + \tau \dot{M}(x, s)) \frac{\partial f}{\partial x}(x, \pi(x, s), s) \\
    &\preceq M(x, s) - 2 \lambda \tau M(x, s) \\
    &\qquad + \tau^2 \left(\frac{\partial f}{\partial x}(x, \pi(x, s), s)^\T \dot{M}(x, s) + \dot{M}(x, s)\frac{\partial f}{\partial x}(x, \pi(x, s), s) \right) \\
    &\qquad + \tau^2 \frac{\partial f}{\partial x}(x, \pi(x, s), s)^\T (M(x, s) + \tau \dot{M}(x, s)) \frac{\partial f}{\partial x}(x, \pi(x, s), s) \\
    &=: T_{1,1,1} + T_{1,1,2} + T_{1,1,3} \:.
\end{align*}
Above, the semidefinite inequality uses the continuous-time contraction inequality. 
Next, we bound $T_{1,1,1}$ as follows:
\begin{align*}
    T_{1,1,1} &= M(x, s) - 2\lambda \tau M(x, s) = (1-2\lambda \tau \gamma) M(x, s) - 2 \lambda \tau (1-\gamma) M(x, s) \\
    &\preceq (1 - 2\lambda \tau \gamma) M(x, s) - 2 \lambda \tau (1-\gamma) \mu I \:.
\end{align*}
To bound $T_{1,1,2}$ and $T_{1,1,3}$, we first estimate a bound on $\dot{M}(x, s)$ as follows:
\begin{align*}
    \norm{\dot{M}(x, s)} &= \bignorm{\frac{\partial M}{\partial x}(x, s) f(x, \pi(x, s), s) + \frac{\partial M}{\partial t}(x, s)} \\
    &\leq L_M \norm{f(x, \pi(x, s), s)} + L_M \\
    &\leq 2 L_M L_f L_\pi \norm{x} + L_M \\
    &\leq 3 D L_M L_f L_\pi \:.
\end{align*}
This estimate allows us to bound:
\begin{align*}
    \norm{T_{1,1,2}} &\leq 6 \tau^2 D L_M L_f^2 L_\pi \:, \\
    \norm{T_{1,1,3}} &\leq 4 \tau^2 D L_M L_f^2 L_\pi \:.
\end{align*}
We are now in a position to establish \eqref{eq:contraction_zoh}.
Combining our bounds above,
\begin{align*}
    &\frac{\partial \Phi}{\partial x}(x, s, s + \tau)^\T M( \Phi(x, s, s + \tau), s + \tau ) \frac{\partial \Phi}{\partial x}(x, s, s + \tau) \\
    &\preceq T_1 + T_2 + T_3 + T_4 \\
    &\preceq T_{1,1,1} + T_{1,1,2} + T_{1,1,3} + T_{1,2} + T_2 + T_3 + T_4 \\
    &\preceq (1-2\lambda \tau \gamma) M(x, s) - 2\lambda \tau (1-\gamma) \mu I + 10 \tau^2 D L_M L_f^2 L_\pi I +  92 e^2 \tau^2 D^2 L_M L_f^2 L_\pi^2 I \\
    &\qquad + 14e^2 \tau^2 D L L_f^2 L_\pi I + \frac{49e^4}{4} \tau^2 D^2 L L_f^2 L_\pi^2 I \:.
\end{align*}
Observe that as long as
\begin{align}
    -2\lambda(1-\gamma)\mu + 10 \tau D L_M L_f^2 L_\pi + 92e^2 \tau D^2 L_M L_f^2 L_\pi^2 + 14e^2 \tau D L L_f^2 L_\pi + \frac{49e^4}{4} \tau D^2 L L_f^2 L_\pi^2 \leq 0 \:, \label{eq:contraction_zoh_negative}
\end{align}
then
\begin{align*}
    \frac{\partial \Phi}{\partial x}(x, s, s + \tau)^\T M( \Phi(x, s, s + \tau), s + \tau ) \frac{\partial \Phi}{\partial x}(x, s, s + \tau) \preceq (1-2\lambda \tau \gamma) M(x, s) \:,
\end{align*}
which is precisely \eqref{eq:contraction_zoh}.
A straightforward calculation shows that the following condition ensures that \eqref{eq:contraction_zoh_negative} holds:
\begin{align*}
    \tau \leq \frac{2 \lambda (1-\gamma) \mu}{ 1463 D^2 L L_M L_f^2 L_\pi^2 } \:.
\end{align*}
The claim now follows.
\end{proof}
\section{Omitted Proofs for Velocity Gradient Results}
\label{sec:app:speed_gradient}

We first state a technical lemma which will be used in the proof
of Theorem~\ref{thm:speed_grad_data_dependent_regret}.
\begin{proposition}[cf.\ Lemma 3.5 of \cite{auer02adaptive}]
\label{prop:weighted_grad_bounds}
For any sequence $\{ g_t \}_{t=1}^{T}$, let $A_t = \sum_{i=1}^{t} g_i^2$.
We have that:
\begin{align*}
    \sqrt{A_T} \leq \sum_{t=1}^{T} \frac{g_t^2}{\sqrt{A_t}} \leq 2 \sqrt{A_T} \:.
\end{align*}
\end{proposition}
\begin{proof}
The lower bound is trivial since $A_t$ is increasing in $t$ and hence:
\begin{align*}
    \sum_{t=1}^{T} \frac{g_t^2}{\sqrt{A_t}} \geq \frac{1}{\sqrt{A_T}} \sum_{t=1}^{T} g_t^2 = \frac{A_T}{\sqrt{A_T}} = \sqrt{A_T} \:.
\end{align*}

We now proceed to the upper bound.
The proof is by induction. Assume w.l.o.g.\ that $g_t$ is a non-negative sequence.
First, for $T=1$, if $g_1 = 0$ there is nothing to prove.
Otherwise, the claim states that $g_1^2 / \sqrt{g_1^2} \leq 2 \sqrt{g_1^2}$ which trivially holds.

Now we assume the claim holds for $T$. If $g_{T+1} = 0$ then there is nothing to prove.
Now assume $g_{T+1} \neq 0$.
Observe that:
\begin{align*}
    \sum_{t=1}^{T+1}  \frac{g_t^2}{\sqrt{A_t}} &= \sum_{t=1}^{T}  \frac{g_t^2}{\sqrt{A_t}} + \frac{g_{T+1}^2}{\sqrt{A_{T+1}}} 
    \stackrel{(a)}{\leq} 2 \sqrt{A_T} +  \frac{g_{T+1}^2}{\sqrt{A_{T+1}}} = \frac{2 \sqrt{A_T}\sqrt{A_{T+1}} + g_{T+1}^2}{\sqrt{A_{T+1}}} \\
    &\stackrel{(b)}{\leq} \frac{ A_T + A_{T+1} + g_{T+1}^2}{\sqrt{A_{T+1}}} = \frac{2 A_{T+1}}{\sqrt{A_{T+1}}} = 2 \sqrt{A_{T+1}} \:.
\end{align*}
Above, (a) follows from the inductive hypothesis and (b) follows
from the inequality $2ab \leq a^2 + b^2$ valid for any $a,b \in \R$.
The claim now follows.
\end{proof}

We now restate and prove Theorem~\ref{thm:speed_grad_data_dependent_regret}.
\speedgradientmain*
\begin{proof}
Observe that by $\mu$-strong convexity of $Q$, we have that
for any $x, t, d$,
\begin{align*}
    Q(f(x, t), t+1) \geq Q(f(x, t) + d, t+1) - \ip{\nabla Q(f(x, t) + d, t+1)}{d} + \frac{\mu}{2} \norm{d}^2 \:. 
\end{align*}
Re-arranging,
\begin{align}
    Q(f(x, t) + d, t+1) \leq Q(f(x, t), t+1) + \ip{\nabla Q(f(x, t) + d, t+1)}{d} - \frac{\mu}{2}\norm{d}^2 \:. \label{eq:speed_grad_strong_cvx}
\end{align}
Define $\eta_{-1} := \frac{D}{\sqrt{\lambda}}$, and consider the Lyapunov-like function
$V_t := Q(x_t^a, t) + \frac{1}{2\eta_{t-1}} \norm{\tilde{\alpha}_t}^2$. Then,
\begin{align*}
    V_{t+1} &= Q(x_{t+1}^a, t+1) + \frac{1}{2\eta_t} \norm{\tilde{\alpha}_{t+1}}^2 \\
    &\stackrel{(a)}{\leq} Q(x_{t+1}^a, t+1) + \frac{1}{2\eta_t} [\norm{\tilde{\alpha}_t}^2 + \eta_t^2 \norm{Y_t^\T B_t^\T \nabla Q(x_{t+1}^a, t+1)}^2 - 2 \eta_t \tilde{\alpha}_t^\T Y_t^\T B_t^\T \nabla Q(x_{t+1}^a, t+1) ] \\
    &= Q(f(x_t^a, t) + B_t Y_t \tilde{\alpha}_t, t+1) - \tilde{\alpha}_t^\T Y_t^\T B_t^\T \nabla Q(x_{t+1}^a, t+1) + \frac{1}{2\eta_t} \norm{\tilde{\alpha}_t}^2 + \frac{\eta_t}{2} \norm{Y_t^\T B_t^\T \nabla Q (x_{t+1}^a, t+1)}^2 \\
    &\stackrel{(b)}{\leq} Q(f(x_t^a, t), t+1) - \frac{\mu}{2} \norm{B_t Y_t \tilde{\alpha}_t}^2 + \frac{1}{2\eta_t} \norm{\tilde{\alpha}_t}^2 + \frac{\eta_t}{2} \norm{Y_t^\T B_t^\T \nabla Q (x_{t+1}^a, t+1)}^2 \\
    &\stackrel{(c)}{\leq} Q(x_t^a, t) - \rho\norm{x_t^a}^2 - \frac{\mu}{2} \norm{B_t Y_t \tilde{\alpha}_t}^2 + \frac{1}{2\eta_t} \norm{\tilde{\alpha}_t}^2 + \frac{\eta_t}{2} \norm{Y_t^\T B_t^\T \nabla Q (x_{t+1}^a, t+1)}^2 \\
    &= V_t + \frac{1}{2}\left( \frac{1}{\eta_t} - \frac{1}{\eta_{t-1}} \right) \norm{\tilde{\alpha}_t}^2  - \rho\norm{x_t^a}^2 - \frac{\mu}{2} \norm{B_t Y_t \tilde{\alpha}_t}^2 + \frac{\eta_t}{2} \norm{Y_t^\T B_t^\T \nabla Q (x_{t+1}^a, t+1)}^2 \\
    &\stackrel{(d)}{\leq} V_t + \left( \frac{1}{\eta_t} - \frac{1}{\eta_{t-1}} \right)2D^2  - \rho\norm{x_t^a}^2 - \frac{\mu}{2} \norm{B_t Y_t \tilde{\alpha}_t}^2 + \frac{\eta_t}{2} \norm{Y_t^\T B_t^\T \nabla Q (x_{t+1}^a, t+1)}^2 \:,
\end{align*}
where (a) holds by the
Pythagorean theorem,
(b) uses the inequality \eqref{eq:speed_grad_strong_cvx}
with $x = x_{t+1}^a$ and $d = B_t Y_t \tilde{\alpha}_t$,
(c) uses the Lyapunov stability assumption \eqref{eq:lyapunov_stability},
and (d) holds after noting that $\eta_t \leq \eta_{t-1}$.
Unrolling this relation,
\begin{align*}
    V_T &\leq V_0 + 2D^2 \sum_{t=0}^{T-1} \left(\frac{1}{\eta_t} - \frac{1}{\eta_{t-1}}\right) - \rho \sum_{t=0}^{T-1} \norm{x_t^a}^2 - \frac{\mu}{2} \sum_{t=0}^{T-1} \norm{B_t Y_t \tilde{\alpha}_t}^2 + \frac{1}{2} \sum_{t=0}^{T-1} \eta_t \norm{Y_t^\T B_t^\T \nabla Q(x_{t+1}^a, t+1)}^2 \\
    &= Q(x_0, 0) + \frac{\sqrt{\lambda}}{2D} \norm{\tilde{\alpha}_0}^2 + 2D^2\left( \frac{1}{\eta_{T-1}} - \frac{1}{\eta_{-1}}\right) - \rho \sum_{t=0}^{T-1} \norm{x_t^a}^2 
     - \frac{\mu}{2} \sum_{t=0}^{T-1} \norm{B_t Y_t \tilde{\alpha}_t}^2 \\
     &\qquad + \frac{1}{2} \sum_{t=0}^{T-1} \eta_t \norm{Y_t^\T B_t^\T \nabla Q(x_{t+1}^a, t+1)}^2 \\
    &\leq Q(x_0, 0) + 2 \sqrt{\lambda} D + \frac{2D^2}{\eta_{T-1}} - \rho \sum_{t=0}^{T-1} \norm{x_t^a}^2 
     - \frac{\mu}{2} \sum_{t=0}^{T-1} \norm{B_t Y_t \tilde{\alpha}_t}^2 + \frac{1}{2} \sum_{t=0}^{T-1} \eta_t \norm{Y_t^\T B_t^\T \nabla Q(x_{t+1}^a, t+1)}^2 \:.
\end{align*}
Using the fact that $V_T \geq 0$ and re-arranging the inequality above,
\begin{align}
    \sum_{t=0}^{T-1} \norm{x_t^a}^2 + \frac{\mu}{2\rho} \sum_{t=0}^{T-1} \norm{B_t Y_t \tilde{\alpha}_t}^2 \leq \frac{Q(x_0, 0)}{\rho} + \frac{2 \sqrt{\lambda} D}{\rho} + \frac{2D^2}{\rho \eta_{T-1}} + \frac{1}{2\rho} \sum_{t=0}^{T-1} \eta_t \norm{Y_t^\T B_t^\T \nabla Q(x_{t+1}^a, t+1)}^2 \:. \label{eq:speed_grad_ineq_1}
\end{align}
Now we apply Proposition~\ref{prop:weighted_grad_bounds}
to the sequence $\{g_t\}_{t=0}^{T}$ defined as 
$g_0 = \sqrt{\lambda}$
and $g_i = \norm{Y_{i-1}^\T B_{i-1}^\T \nabla Q(x_i^a, i)}$
for $i=1, ..., T$ to conclude that
\begin{align*}
    \sum_{t=0}^{T-1} \eta_t \norm{Y_t^\T B_t^\T \nabla Q(x_{t+1}^a, t+1)}^2 \leq 2 D \sqrt{ \lambda + \sum_{t=0}^{T-1}
 \norm{Y_t^\T B_t^\T \nabla Q(x_{t+1}^a, t+1)}^2 } \:.
\end{align*}
Plugging the above inequality into \eqref{eq:speed_grad_ineq_1}:
\begin{align*}
    \sum_{t=0}^{T-1} \norm{x_t^a}^2 + \frac{\mu}{2\rho} \sum_{t=0}^{T-1} \norm{B_t Y_t \tilde{\alpha}_t}^2 &\leq \frac{Q(x_0, 0)}{\rho} + \frac{2 \sqrt{\lambda} D}{\rho} + \frac{3D}{\rho} \sqrt{ \lambda + \sum_{t=0}^{T-1}
 \norm{Y_t^\T B_t^\T \nabla Q(x_{t+1}^a, t+1)}^2 } \\
 &\leq \frac{Q(x_0, 0)}{\rho} + \frac{5\sqrt{\lambda}D}{\rho} + \frac{3D}{\rho} \sqrt{ \sum_{t=0}^{T-1}
 \norm{Y_t^\T B_t^\T \nabla Q(x_{t+1}^a, t+1)}^2 } \:.
\end{align*}
\end{proof}

We now restate and prove Theorem~\ref{thm:speed_grad_lipschitz}.
\speedgradlipschitz*
\begin{proof}
Using our assumptions and the inequality $(a+b)^2 \leq 2 a^2 + 2 b^2$, we have:
\begin{align*}
    &\sum_{t=0}^{T-1} \norm{Y_t^\T B_t^\T \nabla Q(x_{t+1}^a, t+1)}^2 \leq M^4 \sum_{t=0}^{T-1} \norm{\nabla Q(x_{t+1}^a, t+1)}^2 
    \leq M^4 L_Q^2 \sum_{t=0}^{T-1} \norm{x_{t+1}^a}^2 \\
    &= M^4 L_Q^2 \sum_{t=0}^{T-1} \norm{f(x_t^a, t) + B_t Y_t \tilde{\alpha}_t}^2 
    \leq 2 M^4 L_Q^2 \sum_{t=0}^{T-1} (\norm{f(x_t^a, t)}^2 + \norm{B_t Y_t \tilde{\alpha}_t}^2) \\
    &\leq 2 M^4 L_Q^2 \sum_{t=0}^{T-1} (L_f^2\norm{x_t^a}^2 + \norm{B_t Y_t \tilde{\alpha}_t}^2) 
    \leq 2 M^4 L_Q^2 \max\left\{ L_f^2, \frac{2\rho}{\mu} \right\} \sum_{t=0}^{T-1} (\norm{x_t^a}^2 + \frac{\mu}{2\rho} \norm{B_t Y_t \tilde{\alpha}_t}^2 ) \:.
\end{align*}
Define $R :=  \frac{\mu}{2\rho} \sum_{t=0}^{T-1} \norm{B_t Y_t \tilde{\alpha}_t}^2 + \sum_{t=0}^{T-1} \norm{x_t^a}^2$.
From Theorem~\ref{thm:speed_grad_data_dependent_regret} we have,
\begin{align*}
    R &\leq \frac{Q(x_0, 0)}{\rho} + \frac{5\sqrt{\lambda}D}{\rho} + \frac{3D}{\rho} \sqrt{ \sum_{t=0}^{T-1} \norm{Y_t^\T B_t^\T \nabla Q(x_{t+1}^a, t+1)}^2 } \\
    &\leq \frac{Q(x_0, 0)}{\rho} + \frac{5\sqrt{\lambda}D}{\rho} + \frac{3\sqrt{2} D}{\rho} M^2 L_Q \max\left\{L_f, \sqrt{\frac{2\rho}{\mu}}\right\} \sqrt{R} \:.
\end{align*}
This is an inequality of the form $R \leq A + B\sqrt{R}$.
Any positive solution to this inequality can be upper bounded
as $R \leq \frac{3}{2} (A + B^2)$.
From this we conclude:
\begin{align*}
    R \leq \frac{3}{2}\left( \frac{Q(x_0, 0)}{\rho} + \frac{5\sqrt{\lambda}D}{\rho}  \right) + \frac{27 D^2}{\rho^2} M^4 L_Q^2 \max\left\{L_f^2, \frac{2\rho}{\mu}\right\} \:.
\end{align*}
\end{proof}

\section{Contraction implies Incremental Stability}
\label{sec:app:contraction}

In this section, we prove Proposition~\ref{prop:contraction_implies_e_delta_iss}
and Proposition~\ref{prop:contraction_with_noise}.
For completeness, we first state and prove a few well-known
technical lemmas in contraction theory. For a Riemannian metric $M(x)$,
we denote the geodesic distance $d_{M}(x, y)$ as:
\begin{align*}
    d_{M}(x, y) :=  \inf_{\gamma \in \Gamma(x, y)} \sqrt{\int_0^1 \frac{\partial \gamma}{\partial s}(s)^\T M(\gamma(s)) \frac{\partial \gamma}{\partial s}(s) \: ds} \:,
\end{align*}
where $\Gamma(x, y)$ is the set of all smooth curves $\gamma$ with
boundary conditions $\gamma(0) = x$ and $\gamma(1) = y$.

\begin{proposition}[cf.\ Lemma 1 of~\cite{pham08discrete}]
\label{prop:discrete_time_contraction_decrease}
Let $f(x, t)$ be contracting with rate $\gamma$
in the metric $M(x, t)$.
Then for all $x, y, t$:
\begin{align*}
    d_{M_{t+1}}^2(f(x, t), f(y, t)) \leq \gamma d_{M_t}^2(x, y) \:.
\end{align*}
Here, $d_{M_t}$ is the geodesic distance associated
with $M(x, t)$.
\end{proposition}
\begin{proof}
Let $\gamma$ denote the geodesic curve under $M_t$
with $\gamma(0) = x$
and $\gamma(1) = y$.
By differentiability of $f(x, t)$,
we have that $\zeta(s) := f(\gamma(s), t)$
is a smooth curve between $f(x, t)$ and $f(y, t)$.
Furthermore:
\begin{align*}
    \frac{\partial \zeta}{\partial s}(s) = \frac{\partial f}{\partial x}(\gamma(s), t) \frac{\partial \gamma}{\partial s}(s) \:.
\end{align*}
Therefore, noting that the geodesic length between
$f(x, t)$ and $f(y, t)$ under $M_{t+1}$ must be less than the
curve length of $\zeta(\cdot)$ under $M_{t+1}$,
\begin{align*}
    d_{M_{t+1}}^2(f(x, t), f(y, t)) &\leq \int_0^1 \frac{\partial \zeta}{\partial s}(s)^\T M(\zeta(s), t+1) \frac{\partial \zeta}{\partial s}(s) \: ds \\
    &= \int_0^1 \frac{\partial \gamma}{\partial s}(s)^\T \frac{\partial f}{\partial x}(\gamma(s), t)^\T M(f(\gamma(s), t), t+1) \frac{\partial f}{\partial x}(\gamma(s), t) \frac{\partial \gamma}{\partial s}(s) \: ds \\
    &\leq \gamma \int_0^1 \frac{\partial \gamma}{\partial s}(s)^\T M(\gamma(s), t) \frac{\partial \gamma}{\partial s}(s) \: ds \\
    &= \gamma d_{M_t}^2(x, y) \:.
\end{align*}
\end{proof}

\begin{proposition}
\label{prop:metric_upper_lower_bounds}
Let the metric $M(x)$ satisfy $\mu I \preceq M(x) \preceq L I$ for all $x$. Then for all $x, y$:
\begin{align*}
    \sqrt{\mu} \norm{x-y} \leq d_M(x, y) \leq \sqrt{L} \norm{x-y} \:.
\end{align*}
\end{proposition}
\begin{proof}
We first prove the upper bound.
Let $\gamma$ denote a straight line between $x, y$. Then:
\begin{align*}
    d_M^2(x, y) \leq \int_0^1 \frac{\partial \gamma}{\partial s}(s)^\T M(\gamma(s)) \frac{\partial \gamma}{\partial s}(s) \: ds \leq L \int_0^1 \bignorm{\frac{\partial \gamma}{\partial s}(s)}^2 \: ds = L \norm{x-y}^2 \:.
\end{align*}
Taking square roots on both sides yields the result.
For the lower bound, let $\gamma$ denote the geodesic curve between $x$ and $y$ under $M$.
Then:
\begin{align*}
    \mu \norm{x-y}^2 \leq \int_0^1 \frac{\partial \gamma}{\partial s}(s)^\T (\mu I) \frac{\partial \gamma}{\partial s}(s) \: ds \leq \int_0^1 \frac{\partial \gamma}{\partial s}(s)^\T M(\gamma(s)) \frac{\partial \gamma}{\partial s}(s) \: ds = d_M^2(x, y) \:.
\end{align*}
Taking square roots on both sides yields the result.
\end{proof}

We now restate and prove Proposition~\ref{prop:contraction_implies_e_delta_iss}.
\contractionincrstability*
\begin{proof}
Let $u_t$ be an arbitrary signal and consider the two systems:
\begin{align*}
    x_{t+1} &= f(x_t, t) + u_t \:, \\
    y_{t+1} &= f(y_t, t) \:.
\end{align*}
We have for all $t \geq 0$:
\begin{align*}
    d_{M_{t+1}}(y_{t+1}, x_{t+1}) &= d_{M_{t+1}}(y_{t+1}, f(x_t, t) + u_t) \\
    &\leq d_{M_{t+1}}(y_{t+1}, f(x_t, t)) + d_{M_{t+1}}(f(x_t, t), f(x_t, t) + u_t) \\
    &= d_{M_{t+1}}(f(y_t, t), f(x_t, t)) + d_{M_{t+1}}(f(x_t, t), f(x_t, t) + u_t) \\
    &\leq \sqrt{\gamma} d_{M_t}(y_t, x_t) + \sqrt{L} \norm{u_t} \:.
\end{align*}
Above, the first inequality follows by the triangle inequality and the
last inequality follows
from Propositions~\ref{prop:discrete_time_contraction_decrease} and~\ref{prop:metric_upper_lower_bounds}.
Unrolling this recursion
and using Proposition~\ref{prop:metric_upper_lower_bounds} again:
\begin{align*}
    \sqrt{\mu}\norm{x_t - y_t} &\leq  d_{M_t}(y_t, x_t) \\
    &\leq \gamma^{t/2} d_{M_0}(x_0, y_0) + \sqrt{L} \sum_{k=0}^{t-1} \gamma^{(t-1-k)/2} \norm{u_k} \\
    &\leq \sqrt{L} \gamma^{t/2} \norm{x_0 - y_0} + \sqrt{L} \sum_{k=0}^{t-1} \gamma^{(t-1-k)/2} \norm{u_k} \:.
\end{align*}
\end{proof}

Next, we restate and prove Proposition~\ref{prop:contraction_with_noise}.
\contractionwithnoise*
\begin{proof}
Observe that $\frac{\partial g}{\partial x}(x, t) = \frac{\partial f}{\partial x}(x, t)$.
Then for any $x, t$:
\begin{align*}
    &\frac{\partial g}{\partial x}(x, t)^\T M(g(x, t), t+1) \frac{\partial g}{\partial x}(x, t) \\
    &= \frac{\partial f}{\partial x}(x, t)^\T M(f(x, t) + w_t, t+1) \frac{\partial f}{\partial x}(x, t) \\
    &= \frac{\partial f}{\partial x}(x, t)^\T M(f(x, t), t+1) \frac{\partial f}{\partial x}(x, t) \\
    &\qquad+ \frac{\partial f}{\partial x}(x, t)^\T (M(f(x, t) + w_t, t+1) - M(f(x, t), t+1)) \frac{\partial f}{\partial x}(x, t) \\
    &\preceq \gamma M(x, t) + \bignorm{\frac{\partial f}{\partial x}(x, t)}^2 \norm{ M(f(x, t) + w_t, t+1) - M(f(x, t), t+1) } I \\
    &\preceq \gamma M(x, t) + L_f^2 L_M W I \\
    &\preceq \left( \gamma + \frac{L_f^2 L_M W}{\mu} \right) M(x, t) \:.
\end{align*}
\end{proof}

\section{Review of Regret Bounds in Online Convex Optimization}

For completeness, we review basic results in online convex optimization (OCO)
specialized to the case of online least-squares.
A reader who is already familiar with OCO may freely skip this section.
See \cite{hazan16oco} for a more complete treatment of the subject.

In particular, we consider the sequence of functions:
\begin{align*}
    f_t(\hat{\alpha}) := \frac{1}{2} \norm{M_t \hat{\alpha} - y_t}^2 \:, \:\: t = 1, 2, ..., T \:,
\end{align*}
where $\hat{\alpha} \in \R^p$ is constrained to lie in
the set $\calC := \{ \hat{\alpha} \in \R^p : \norm{\hat{\alpha}} \leq D \}$.
All algorithms are initialized with an arbitrary
$\hat{\alpha}_1 \in \calC$.
We define the prediction regret as:
\begin{align*}
    \mathsf{PredictionRegret}(T) := \sup_{\alpha \in \calC} \sum_{t=1}^{T} f_t(\hat{\alpha}_t) - f_t(\alpha) \:. 
\end{align*}
For what follows,
we will assume that $\norm{M_t} \leq M$ and $\norm{Y_t} \leq Y$,
so that $\norm{\nabla f_t(\hat{\alpha})} \leq G := M(DM + Y)$.

\subsection{Online Gradient Descent}

The online gradient descent update is:
\begin{align*}
    \hat{\alpha}_{t+1} = \Pi_{\calC}[ \hat{\alpha}_t - \eta_t \nabla f_t(\hat{\alpha}_t) ] \:.
\end{align*}

The following proposition shows that online gradient descent achieves
$\sqrt{T}$ regret.
\begin{proposition}[cf.\ Theorem 3.1 of \cite{hazan16oco}]
\label{prop:online_gd_bounded_case}
Suppose we run the online gradient descent update with 
$\eta_t := \frac{D}{G\sqrt{t}}$.
We have:
\begin{align*}
    \sup_{\alpha \in \calC} \sum_{t=1}^{T} f_t(\hat{\alpha}_t) - f_t(\alpha) \leq 3 GD \sqrt{T} \:.
\end{align*}
\end{proposition}
\begin{proof}
Fix any $\alpha \in \calC$
and define $\tilde{\alpha}_t := \hat{\alpha}_t - \alpha$.
We abbreviate $\nabla_t := \nabla f_t(\hat{\alpha}_t)$.
First, using the Pythagorean theorem, we perform the following expansion for $t \geq 1$:
\begin{align*}
    \norm{\tilde{\alpha}_{t+1}}^2 \leq \norm{\tilde{\alpha}_t}^2 - 2 \eta_t \ip{\tilde{\alpha}_t}{\nabla_t} + \eta_t^2 \norm{\nabla_t}^2 \:.
\end{align*}
Re-arranging the above inequality yields:
\begin{align*}
    \ip{\tilde{\alpha}_t}{\nabla_t} \leq \frac{1}{2\eta_t} (\norm{\tilde{\alpha}_t}^2 - \norm{\tilde{\alpha}_{t+1}}^2) + \frac{\eta_t}{2} \norm{\nabla_t}^2 \:.
\end{align*}
Therefore by convexity of the $f_t$'s:
\begin{align*}
    \sum_{t=1}^{T} f_t(\hat{\alpha}_t) - f_t(\alpha) &\leq \sum_{t=1}^{T} \ip{\tilde{\alpha}_t}{\nabla_t} \leq \sum_{t=1}^{T} \frac{1}{2\eta_t} (\norm{\tilde{\alpha}_t}^2 - \norm{\tilde{\alpha}_{t+1}}^2) + \frac{\eta_t}{2} \norm{\nabla_t}^2 \\
    &\leq \frac{1}{2}\left( \frac{\norm{\tilde{\alpha}_1}^2}{\eta_1} - \frac{\norm{\tilde{\alpha}_{T+1}}^2}{\eta_T} \right) + \frac{1}{2} \sum_{t=2}^{T} \norm{\tilde{\alpha}_t}^2 \left(\frac{1}{\eta_t} - \frac{1}{\eta_{t-1}}\right) + \frac{1}{2} \sum_{t=1}^{T} \eta_t \norm{\nabla_t}^2 \\
    &\leq \frac{2D^2}{\eta_1} + 2 D^2 \sum_{t=2}^{T} \left(\frac{1}{\eta_t} - \frac{1}{\eta_{t-1}}\right) + \frac{1}{2} \sum_{t=1}^{T} \eta_t \norm{\nabla_t}^2 \\
    &= \frac{2D^2}{\eta_1} + 2 D^2 \left( \frac{1}{\eta_T} - \frac{1}{\eta_1} \right) + \frac{1}{2} \sum_{t=1}^{T} \eta_t \norm{\nabla_t}^2 \\
    &= \frac{2D^2}{\eta_T} + \frac{1}{2} \sum_{t=1}^{T} \eta_t \norm{\nabla_t}^2 
    \leq \frac{2D^2}{\eta_T} + \frac{G^2}{2} \sum_{t=1}^{T} \eta_t \\
    &= 2GD \sqrt{T} + \frac{GD}{2} \sum_{t=1}^{T} \frac{1}{\sqrt{t}} \leq 3 GD \sqrt{T} \:.
\end{align*}
\end{proof}
\subsection{Online Newton Method}
The online Newton algorithm we consider is:
\begin{align*}
    \hat{\alpha}_{t+1} = \Pi_{\mathcal{C},t}[ \hat{\alpha}_t - \eta A_t^{-1} \nabla f_t(\hat{\alpha}_t)] \:, \:\: A_t = \lambda I + \sum_{i=1}^{t} M_i^\T M_i \:.
\end{align*}
Here, $\Pi_{\mathcal{C},t}$ is a generalized projection with respect to the $A_t$ norm:
\begin{align*}
    \Pi_{\calC, t}[x] = \arg\min_{y \in \calC} \norm{ x - y }_{A_t} \:.
\end{align*}
The following result is the regret bound for the online Newton method,
specialized to the least-squares setting rather than the more general
exp-concave setting handled in~\citet{hazan16oco}.
\begin{proposition}[cf.\ Theorem 4.4 of \cite{hazan16oco}]
\label{prop:online_newton_bounded_case}
Suppose we run the online Newton update with any $\lambda > 0$ and $\eta \geq 1$.
Then we have:
\begin{align*}
    \sup_{\alpha \in \calC} \sum_{t=1}^{T} f_t(\hat{\alpha}_t) - f_t(\alpha) \leq \frac{2 D^2}{\eta} (\lambda + M^2) + \frac{\eta p}{2} (DM + Y)^2 \log(1 + M^2 T/\lambda) \:.
\end{align*}
\end{proposition}
\begin{proof}
Let $\alpha$ be any fixed point in $\calC$.
By the Pythagorean theorem,
for any $\hat{\alpha}$, we have that $\norm{\Pi_{\calC,t}(\hat{\alpha}) - \alpha}_{A_t} \leq \norm{\hat{\alpha} - \alpha}_{A_t}$.
Therefore, defining $\tilde{\alpha}_t := \hat{\alpha}_t - \alpha$ and abbreviating $\nabla_t := \nabla f_t(\hat{\alpha}_t)$,
for any $t \geq 1$:
\begin{align*}
    \norm{\tilde{\alpha}_{t+1}}^2_{A_t} \leq \norm{\tilde{\alpha}_t}^2_{A_t} + \eta^2 \norm{\nabla_t}^2_{A_t^{-1}} - 2 \eta \ip{\tilde{\alpha}_t}{\nabla_t} \:.
\end{align*}
Re-arranging the above inequality yields,
\begin{align*}
    \ip{\tilde{\alpha}_t}{\nabla_t} \leq \frac{1}{2\eta} ( \norm{\tilde{\alpha}_t}^2_{A_t} - \norm{\tilde{\alpha}_{t+1}}^2_{A_t}) + \frac{\eta}{2}\norm{\nabla_t}^2_{A_t^{-1}} \:.
\end{align*}
Because $f_t$ is quadratic, its second order Taylor expansion yields the identity:
\begin{align*}
    f_t(\hat{\alpha}_t) - f_t(\alpha) = \ip{\nabla_t}{\tilde{\alpha}_t} - \frac{1}{2} \norm{\tilde{\alpha}_t}^2_{M_t^\T M_t} \:.
\end{align*}
Therefore,
\begin{align*}
    \sum_{t=1}^{T} f_t(\hat{\alpha}_t) - f_t(\alpha) &= \sum_{t=1}^{T} \ip{\nabla_t}{\tilde{\alpha}_t} - \frac{1}{2} \norm{\tilde{\alpha}_t}^2_{M_t^\T M_t} \\
    &\leq \sum_{t=1}^{T}  \frac{1}{2\eta} ( \norm{\tilde{\alpha}_t}^2_{A_t} - \norm{\tilde{\alpha}_{t+1}}^2_{A_t}) + \frac{\eta}{2}\norm{\nabla_t}^2_{A_t^{-1}} - \frac{1}{2} \norm{\tilde{\alpha}_t}^2_{M_t^\T M_t} \:.
\end{align*}
Next, we observe that:
\begin{align*}
     \sum_{t=1}^{T} ( \norm{\tilde{\alpha}_t}^2_{A_t} - \norm{\tilde{\alpha}_{t+1}}^2_{A_t}) \leq \norm{\tilde{\alpha}_1}^2_{A_1} + \sum_{t=2}^{T} \tilde{\alpha}_t^\T( A_t - A_{t-1} ) \tilde{\alpha}_t 
     = \norm{\tilde{\alpha}_1}^2_{A_1} + \sum_{t=2}^{T} \norm{\tilde{\alpha}_t}^2_{M_t^\T M_t} \:.
\end{align*}
Therefore as long as $\eta \geq 1$,
\begin{align*}
    \sum_{t=1}^{T} f_t(\hat{\alpha}_t) - f_t(\alpha) &\leq \frac{1}{2\eta} \norm{\tilde{\alpha}_1}^2_{A_1} + \frac{\eta}{2} \sum_{t=1}^{T} \norm{\nabla_t}^2_{A_t^{-1}} \:.
\end{align*}
Let $\nabla_t = M_t^\T r_t$ with $r_t := M_t \hat{\alpha}_t - y_t$.
With this notation:
\begin{align*}
    \norm{\nabla_t}^2_{A_t^{-1}} &= \Tr( \nabla_t^\T A_t^{-1} \nabla_t ) 
    = \Tr( M_t A_t^{-1} M_t^\T r_tr_t^\T ) 
    \leq \norm{r_t}^2 \Tr(A_t^{-1} M_t^\T M_t) \\
    &= \norm{r_t}^2 \Tr(A_t^{-1}(A_t - A_{t-1})) 
    \leq \norm{r_t}^2 \log\frac{\det{A_t}}{\det{A_{t-1}}} \:.
\end{align*}
Above, the last inequality follows
from Lemma 4.6 of \cite{hazan16oco}.
Therefore:
\begin{align*}
    \sum_{t=1}^{T} f_t(\hat{\alpha}_t) - f_t(\alpha) &\leq \frac{1}{2\eta} \norm{\tilde{\alpha}_1}^2_{A_1} + \frac{\eta}{2}\sum_{t=1}^{T}  \norm{M_t \hat{\alpha}_t - y_t}^2 \log\frac{\det(A_t)}{\det(A_{t-1})} \\
    &\leq  \frac{1}{2\eta} \norm{\tilde{\alpha}_1}^2_{A_1} + \frac{\eta}{2} \max_{t=1,...,T} \norm{M_t \hat{\alpha}_t - y_t}^2 \sum_{t=1}^{T} \log\frac{\det(A_t)}{\det(A_{t-1})} \\
    &= \frac{1}{2\eta} \norm{\tilde{\alpha}_1}^2_{A_1} + \frac{\eta}{2} \max_{t=1,...,T} \norm{M_t \hat{\alpha}_t - y_t}^2 \log \frac{\det(A_T)}{\det(A_0)} \\
    &\leq \frac{1}{2\eta} 4 D^2 (\lambda + M^2) + \frac{\eta}{2} (DM + Y)^2 p \log(1 + M^2 T/\lambda) \:.
\end{align*}
\end{proof}

\section{Omitted Proofs for Online Least-Squares Results}
\label{sec:app:online_ls}

We first restate and prove Theorem~\ref{thm:regret_ls_reduction}
\regretlsreduction*
\begin{proof}
Fix a realization $\{w_t\}$.
We first compare the two trajectories:
\begin{align*}
    x^c_{t+1} &= f(x^c_t, t) + w_t \:, \:\: x^c_0 = x_0 \:, \\
    y_{t+1} &= f(y_t, t) \:, \:\: y_0 = 0 \:.
\end{align*}
Since the zero trajectory is a valid trajectory for $f(x, t)$, 
and since the system $f(x, t)$ is $(\beta, \rho, \gamma)$-E-$\delta$ISS,
we have by \eqref{eq:e_delta_iss_ineq} for all $t \geq 0$:
\begin{align*}
    \norm{x_t^c} \leq \beta \rho^t \norm{x_0} + \gamma \sum_{k=0}^{t} \rho^{t-1-k} \norm{w_k} \leq \beta \norm{x_0} + \frac{W \gamma}{1-\rho} \:. 
\end{align*}
Next, we compare the two trajectories:
\begin{align*}
    x^a_{t+1} &= f(x^a_t, t) + B_t Y_t \tilde{\alpha}_t + w_t \:, \:\: x^a_0 = x_0 \:, \\
    x^c_{t+1} &= f(x^c_t, t) + w_t \:, \:\: x^c_0 = x_0 \:.
\end{align*}
Since the system $g(x, t)$ is also $(\beta, \rho, \gamma)$-E-$\delta$ISS,
we have by \eqref{eq:e_delta_iss_ineq} for all $t \geq 0$:
\begin{align*}
    \norm{x_t^a - x_t^c} \leq \gamma \sum_{k=0}^{t-1} \rho^{t-1-k} \norm{B_k Y_k \tilde{\alpha}_k} \:.
\end{align*}
We can upper bound the RHS of the above inequality by $\frac{2 M^2 D \gamma}{1-\rho}$.
Therefore, we have for all $t \geq 0$:
\begin{align*}
    \max\{\norm{x_t^a}, \norm{x_t^c}\} \leq \beta \norm{x_0} + \frac{(2M^2 D + W)\gamma}{1-\rho} = B_x \:.
\end{align*}
We now write:
\begin{align*}
    \sum_{t=0}^{T-1} \norm{x_t^a}^2 - \norm{x_t^c}^2 &\leq \sum_{t=0}^{T-1} (\norm{x_t^a} + \norm{x_t^c}) \norm{x_t^a - x_t^c} \leq 2 B_x \gamma \sum_{t=0}^{T-1} \sum_{k=0}^{t-1} \rho^{t-1-k} \norm{B_k Y_k \tilde{\alpha}_k} \\
    &\leq \frac{2 B_x \gamma}{1-\rho} \sum_{t=0}^{T-1} \norm{B_t Y_t \tilde{\alpha}_t} \leq \frac{2 B_x \gamma}{1-\rho} \sqrt{T} \sqrt{\sum_{t=0}^{T-1} \norm{B_t Y_t \tilde{\alpha}_t}^2} \:. 
\end{align*}
The first inequality follows by factorization and the reverse triangle inequality, while the last follows by Cauchy-Schwarz.
The above inequality holds for every realization $\{w_t\}$.
Therefore, taking an expectation and using Jensen's inequality
to move the expectation under the square root:
\begin{align*}
    \E\left[\sum_{t=0}^{T-1} \norm{x_t^a}^2 - \norm{x_t^c}^2 \right] \leq \frac{2 B_x \gamma}{1-\rho} \sqrt{T} \sqrt{\sum_{t=0}^{T-1} \E\norm{B_t Y_t \tilde{\alpha}_t}^2} \:.
\end{align*}
\end{proof}

We first prove a result analogous to Theorem~\ref{thm:regret_ls_reduction}
for the $k$ timestep delayed system~\eqref{eq:delayed_system}.
\begin{lemma}
\label{lemma:delayed_decomposition}
Consider the $k$ timestep delayed system~\eqref{eq:delayed_system}.
Suppose that Assumption~\ref{assumption:online_ls} holds.
We have for every realization $\{w_t\}$ satisfying $\sup_{t} \norm{w_t} \leq W$
and every $T \geq k$:
\begin{align*}
    \sum_{t=0}^{T-1} \norm{x_t^a}^2 - \norm{x_t^c}^2 &\leq k B_x^2 + \frac{2B_x M^2 D \gamma}{(1-\rho)^2} + \frac{2B_x\gamma}{1-\rho} \left(\sum_{t=0}^{T-1} \norm{B_t Y_t \tilde{\alpha}_t} + \sum_{s=k}^{T-2} \norm{B_s Y_s (\hat{\alpha}_s - \hat{\alpha}_{s-k})} \right)\:.
\end{align*}
\end{lemma}
\begin{proof}
Fix a realization $\{w_t\}$.
We compare the two dynamical systems:
\begin{align*}
    x_{t+1}^a &= f(x_t^a, t) + B(x_t^a, t) (\xi_t - Y_t \alpha) + w_t \:, \:\: x_0^a = x_0 \:, \\
    x_{t+1}^c &= f(x_t^c, t) + w_t \:, \:\: x_0^c = x_0 \:.
\end{align*}
Let $\hat{\alpha}_{t} = 0$ for all $t < 0$.
Because $f(x, t) + w_t$ is $(\beta,\rho,\gamma)$-E-$\delta$ISS, then
for all $t \geq 0$ we have:
\begin{align*}
    \norm{x_t^a - x_t^c} &\leq \gamma \sum_{s=0}^{t-1} \rho^{t-1-s} \norm{B_s(\xi_s - Y_s \alpha)} \\
    &= \gamma \sum_{s=0}^{t-1} \rho^{t-1-s} \norm{B_s(u_{s-k} - Y_s \alpha)} \\
    &= \gamma \sum_{s=0}^{t-1} \rho^{t-1-s} \norm{ B_s Y_s \tilde{\alpha}_{s-k}} \\
    &\leq \gamma \sum_{s=0}^{t-1} \rho^{t-1-s} \norm{ B_s Y_s \tilde{\alpha}_{s}} + \gamma \sum_{s=0}^{t-1} \rho^{t-1-s} \norm{B_s Y_s (\hat{\alpha}_{s} - \hat{\alpha}_{s-k})} \:.
\end{align*}
Therefore:
\begin{align*}
    \sum_{t=k}^{T-1} \norm{x_t^a - x_t^c} &\leq \gamma \sum_{t=k}^{T-1} \sum_{s=0}^{t-1} \rho^{t-1-s} \norm{ B_s Y_s \tilde{\alpha}_{s}} + \gamma \sum_{t=k}^{T-1} \sum_{s=0}^{t-1} \rho^{t-1-s} \norm{B_s Y_s (\hat{\alpha}_{s} - \hat{\alpha}_{s-k})} \\
    &\leq \frac{\gamma}{1-\rho} \sum_{t=0}^{T-1} \norm{B_t Y_t \tilde{\alpha}_t} + M^2 D \gamma \sum_{t=k}^{T-1} \sum_{s=0}^{k-1} \rho^{t-1-s} +  \gamma \sum_{t=k}^{T-1} \sum_{s=k}^{t-1} \rho^{t-1-s} \norm{B_s Y_s (\hat{\alpha}_s-\hat{\alpha}_{s-k})} \\
    &\leq \frac{\gamma}{1-\rho} \sum_{t=0}^{T-1} \norm{B_t Y_t \tilde{\alpha}_t} + M^2 D \gamma \frac{(1-\rho^{k})(1-\rho^{T-k})}{(1-\rho)^2} + \frac{\gamma}{1-\rho} \sum_{s=k}^{T-2} \norm{B_s Y_s (\hat{\alpha}_s - \hat{\alpha}_{s-k})} \\
    &\leq \frac{\gamma}{1-\rho} \sum_{t=0}^{T-1} \norm{B_t Y_t \tilde{\alpha}_t} + M^2 D \frac{\gamma}{(1-\rho)^2} + \frac{\gamma}{1-\rho} \sum_{s=k}^{T-2} \norm{B_s Y_s (\hat{\alpha}_s - \hat{\alpha}_{s-k})} \:.
\end{align*}
By an identical argument as in Theorem~\ref{thm:regret_ls_reduction},
we can bound:
\begin{align*}
    \max\{\norm{x_t^a}, \norm{x_t^c}\} \leq \beta \norm{x_0} + \frac{(2M^2 D + W)\gamma}{1-\rho} = B_x \:.
\end{align*}
Hence:
\begin{align*}
    &\sum_{t=0}^{T-1} \norm{x_t^a}^2 - \norm{x_t^c}^2 \\
    &= \sum_{t=0}^{k-1} \norm{x_t^a}^2 - \norm{x_t^c}^2 + \sum_{t=k}^{T-1} \norm{x_t^a}^2 - \norm{x_t^c}^2 \\
    &\leq k B_x^2 + \sum_{t=k}^{T-1} (\norm{x_t^a} + \norm{x_t^c}) \norm{x_t^a - x_t^c} \\
    &\leq k B_x^2 + 2 B_x \left( \frac{\gamma}{1-\rho} \sum_{t=0}^{T-1} \norm{B_t Y_t \tilde{\alpha}_t} + M^2 D \frac{\gamma}{(1-\rho)^2} + \frac{\gamma}{1-\rho} \sum_{s=k}^{T-2} \norm{B_s Y_s (\hat{\alpha}_s - \hat{\alpha}_{s-k})} \right) \:.
\end{align*}
\end{proof}

Lemma~\ref{lemma:delayed_decomposition}
shows that the extra work needed to bound the control regret
in the delayed setting is to control the drift error
$\sum_{s=k}^{T-2} \norm{B_s Y_s (\hat{\alpha}_s - \hat{\alpha}_{s-k})}$. We have two proof strategies for bounding this term, one for each of online gradient descent and online Newton.

We first focus on the proof of Theorem~\ref{thm:regret_ls_delay},
which is the result for online gradient descent.
Towards this goal, we require a proposition
that bounds the drift of the parameters $\hat{\alpha}_t$.
While this type of result is standard in the online
learning community, we replicate its proof for completeness.
\begin{proposition}
\label{prop:iterate_shift}
Consider the online gradient descent update
\eqref{eq:online_gd}. Suppose that
$\sup_{x,t} \norm{B(x,t)} \leq M$
and $\sup_{x,t} \norm{Y(x,t)} \leq M$.
Put $G = M^2 (2 D M^2 + W)$
and let $\eta_t = \frac{D}{G\sqrt{t+1}}$.
Then we have for any $t \geq 0$ and $k \geq 1$:
\begin{align*}
    \norm{\tilde{\alpha}_{t+k} - \tilde{\alpha}_t} \leq \frac{Dk}{\sqrt{t+1}} \:.
\end{align*}
\end{proposition}
\begin{proof}
First, we observe that by the Pythagorean theorem:
\begin{align*}
    \norm{\tilde{\alpha}_{t+1} - \tilde{\alpha}_t} 
    &= \norm{\hat{\alpha}_{t+1} - \hat{\alpha}_t} 
    = \norm{\Pi_{\calC}[\hat{\alpha}_t - \eta_t Y_t^\T B_t^\T (B_t Y_t \tilde{\alpha}_t + w_t)] - \hat{\alpha}_t} \\
    &\leq \eta_t \norm{ Y_t^\T B_t^\T (B_t Y_t \tilde{\alpha}_t + w_t)} 
    \leq \eta_t M^2( 2DM^2 + W) = \eta_t G \:.
\end{align*}
Therefore for any $k \geq 1$:
\begin{align*}
    \norm{\tilde{\alpha}_{t+k} - \tilde{\alpha}_t} &= \bignorm{\sum_{i=0}^{k-1} (\hat{\alpha}_{t+i+1} - \hat{\alpha}_{t+i})} 
    \leq \sum_{i=0}^{k-1} \norm{\hat{\alpha}_{t+i+1} - \hat{\alpha}_{t+i}} 
    \leq G \sum_{i=0}^{k-1} \eta_{t+i} \\
    &= D \sum_{i=0}^{k-1} \frac{1}{\sqrt{t+i+1}} 
    = \frac{D}{\sqrt{t+1}} + D \sum_{i=1}^{k-1} \frac{1}{\sqrt{t+i+1}} \\
    &\leq \frac{D}{\sqrt{t+1}} + D \int_0^{k-1} \frac{1}{\sqrt{t+x+1}} \: dx 
    = \frac{D}{\sqrt{t+1}} + 2D ( \sqrt{t + k} - \sqrt{t + 1} ) \\
    &\leq \frac{D}{\sqrt{t+1}} + \frac{D(k-1)}{\sqrt{t+1}} 
    = \frac{D k}{\sqrt{t+1}} \:.
\end{align*}
\end{proof}

We now restate and prove Theorem~\ref{thm:regret_ls_delay}.
\delayresult*
\begin{proof}
By Proposition~\ref{prop:iterate_shift}, we bound:
\begin{align*}
    \sum_{s=k}^{T-2} \norm{B_s Y_s (\hat{\alpha}_s - \hat{\alpha}_{s-k})} &\leq M^2 \sum_{s=k}^{T-2} \norm{\hat{\alpha}_s - \hat{\alpha}_{s-k}} \leq M^2 D k \sum_{s=k}^{T-2} \frac{1}{\sqrt{s-k+1}} \\
    &\leq M^2 D k \left(1 + 2 \sqrt{T-2-k+1} - 2 \right) \leq 2 M^2 D k \sqrt{T} \:.
\end{align*}
Hence by Lemma~\ref{lemma:delayed_decomposition}:
\begin{align*}
    \sum_{t=0}^{T-1} \norm{x_t^a}^2 - \norm{x_t^c}^2 
&\leq k B_x^2 + \frac{2B_x M^2 D \gamma}{(1-\rho)^2} + \frac{2B_x\gamma}{1-\rho} \left(\sum_{t=0}^{T-1} \norm{B_t Y_t \tilde{\alpha}_t} + \sum_{s=k}^{T-2} \norm{B_s Y_s (\hat{\alpha}_s - \hat{\alpha}_{s-k})} \right) \\
&\leq k B_x^2 + \frac{2B_x M^2 D \gamma}{(1-\rho)^2} + \frac{2B_x\gamma}{1-\rho} \left(\sqrt{T} \sqrt{\sum_{t=0}^{T-1} \norm{B_t Y_t \tilde{\alpha}_t}^2} + 2 M^2 D k \sqrt{T} \right)
\end{align*}
Taking expectations and using Jensen's inequality followed by 
Proposition~\ref{prop:online_gd_bounded_case}:
\begin{align*}
    \E\left[\sum_{t=0}^{T-1} \norm{x_t^a}^2 - \norm{x_t^c}^2\right] &\leq k B_x^2 + \frac{2B_x M^2 D \gamma}{(1-\rho)^2} + \frac{2B_x\gamma}{1-\rho} \left(\sqrt{T} \sqrt{\sum_{t=0}^{T-1} \E\norm{B_t Y_t \tilde{\alpha}_t}^2} + 2 M^2 D k \sqrt{T} \right) \\
    &\leq k B_x^2 + \frac{2B_x M^2 D \gamma}{(1-\rho)^2} + \frac{2B_x\gamma}{1-\rho} \left(\sqrt{T} \sqrt{ 6 G D T^{1/2}} + 2 M^2 D k \sqrt{T} \right) \\
    &= k B_x^2 + \frac{2B_x M^2 D \gamma}{(1-\rho)^2} + \frac{2 \sqrt{6} B_x \gamma \sqrt{GD}}{1-\rho} T^{3/4} + \frac{4 B_x \gamma M^2 D}{1-\rho} k \sqrt{T} \:.
\end{align*}
\end{proof}

Next, we turn to proving the result for online Newton's method.
The following two propositions will allow us to bound the drift
error.
\begin{proposition}
\label{prop:online_newton_one_step_diff}
For the online Newton update~\eqref{eq:online_newton}, 
we have for every $t \geq 0$:
\begin{align*}
    \norm{\hat{\alpha}_{t+1} - \hat{\alpha}_t}_{A_t} \leq \eta \norm{\nabla f_t(\hat{\alpha}_t)}_{A_t^{-1}} \:.
\end{align*}
\end{proposition}
\begin{proof}
Since $\Pi_{\calC,t}[\cdot]$ is the orthogonal projection
onto $\calC$ in the $\norm{\cdot}_{A_t}$-norm,
by the Pythagorean theorem:
\begin{align*}
    \norm{\hat{\alpha}_{t+1} - \hat{\alpha}_t}_{A_t} = \norm{ \Pi_{\calC,t}[\hat{\alpha}_t - \eta A_t^{-1} \nabla f_t(\hat{\alpha}_t)] - \hat{\alpha}_t}_{A_t} \leq \eta \norm{A_t^{-1} \nabla f_t(\hat{\alpha}_t)}_{A_t} = \eta \norm{\nabla f_t(\hat{\alpha}_t)}_{A_t^{-1}} \:.
\end{align*}
\end{proof}

\begin{proposition}
\label{prop:online_newton_k_step_diff}
Consider the online Newton update~\eqref{eq:online_newton}.
Suppose that $\sup_{x,t} \norm{B(x,t)} \leq M$ and
$\sup_{x,t} \norm{Y(x,t)} \leq M$.
For any $1 \leq k \leq s$, we have:
\begin{align*}
    \norm{B_s Y_s(\hat{\alpha}_s - \hat{\alpha}_{s-k})} \leq \frac{M^2 \eta}{\sqrt{\lambda}} \sum_{\ell=1}^{k} \norm{\nabla f_{s-\ell}(\hat{\alpha}_{s-\ell})}_{A^{-1}_{s-\ell}} \:.
\end{align*}
\end{proposition}
\begin{proof}
First, by definition of $A_t$, we have that
$A_t \succeq \lambda I$ for every $t \geq 0$.
Therefore, for any $t \geq 0$:
\begin{align*}
    (B_s Y_s)^\T (B_s Y_s) \preceq M^4 I \preceq \frac{M^4}{\lambda} A_t \:.
\end{align*}
Therefore by Proposition~\ref{prop:online_newton_one_step_diff}:
\begin{align*}
    \norm{B_s Y_s (\hat{\alpha}_s - \hat{\alpha}_{s-k})} &= \bignorm{ B_s Y_s \left( \sum_{\ell=0}^{k-1} \hat{\alpha}_{s-\ell} - \hat{\alpha}_{s-\ell-1} \right)} \leq \sum_{\ell=0}^{k-1} \norm{B_s Y_s (\hat{\alpha}_{s-\ell} - \hat{\alpha}_{s-\ell-1})} \\
    &= \sum_{\ell=0}^{k-1} \sqrt{ (\hat{\alpha}_{s-\ell} - \hat{\alpha}_{s-\ell-1})^\T (B_s Y_s)^\T (B_s Y_s) (\hat{\alpha}_{s-\ell} - \hat{\alpha}_{s-\ell-1}) } \\
    &\leq \sum_{\ell=0}^{k-1} \sqrt{ (\hat{\alpha}_{s-\ell} - \hat{\alpha}_{s-\ell-1})^\T \left(\frac{M^4}{\lambda} A_{s-\ell-1}\right) (\hat{\alpha}_{s-\ell} - \hat{\alpha}_{s-\ell-1}) } \\
    &= \frac{M^2}{\sqrt{\lambda}} \sum_{\ell=0}^{k-1} \norm{\hat{\alpha}_{s-\ell} - \hat{\alpha}_{s-\ell-1}}_{A_{s-\ell-1}} \leq \frac{M^2 \eta}{\sqrt{\lambda}} \sum_{\ell=0}^{k-1} \norm{\nabla f_{s-\ell-1}(\hat{\alpha}_{s-\ell-1})}_{A_{s-\ell-1}^{-1}} \:.
\end{align*}
\end{proof}

We now restate and prove Theorem~\ref{thm:regret_ls_delay_newton}.
\delayresultnewton*
\begin{proof}
By Proposition~\ref{prop:online_newton_k_step_diff}, we bound:
\begin{align*}
    \sum_{s=k}^{T-2} \norm{B_s Y_s (\hat{\alpha}_s - \hat{\alpha}_{s-k})} &\leq \frac{M^2 }{\sqrt{\lambda}} \sum_{s=k}^{T-2} \sum_{\ell=1}^{k} \norm{\nabla f_{s-\ell}(\hat{\alpha}_{s-\ell})}_{A_{s-\ell}^{-1}} \leq \frac{M^2  k}{\sqrt{\lambda}} \sum_{t=0}^{T-1} \norm{\nabla f_t(\hat{\alpha}_t)}_{A_t^{-1}} \\
    &\leq \frac{M^2  k}{\sqrt{\lambda}} \sqrt{T} \sqrt{\sum_{t=0}^{T-1} \norm{\nabla f_t(\hat{\alpha}_t)}^2_{A_t^{-1}}} \\
    &\stackrel{(a)}{\leq} \frac{M^2  k}{\sqrt{\lambda}} \sqrt{T} \sqrt{ (2 D M^2 + W)^2 \sum_{t=0}^{T-1} \log\frac{\det{A_t}}{\det{A_{t-1}}} } \\
    &\leq M^2 (2 D M^2 + W)  k \sqrt{\frac{p T}{\lambda} \log(1 + M^2 T/\lambda)} \\
    &= G k \sqrt{\frac{p T}{\lambda} \log(1 + M^2 T/\lambda)} \:.
\end{align*}
Above, (a)
follows from Lemma 4.6 of \cite{hazan16oco}
(cf.\ the analysis in Proposition~\ref{prop:online_newton_bounded_case}).
Therefore by Lemma~\ref{lemma:delayed_decomposition}:
\begin{align*}
    &\sum_{t=0}^{T-1} \norm{x_t^a}^2 - \norm{x_t^c}^2 \\
    &\leq k B_x^2 + \frac{2B_x M^2 D \gamma}{(1-\rho)^2} + \frac{2B_x\gamma}{1-\rho} \left(\sum_{t=0}^{T-1} \norm{B_t Y_t \tilde{\alpha}_t} + \sum_{s=k}^{T-2} \norm{B_s Y_s (\hat{\alpha}_s - \hat{\alpha}_{s-k})} \right) \\
    &\leq k B_x^2 + \frac{2B_x M^2 D \gamma}{(1-\rho)^2} + \frac{2B_x\gamma}{1-\rho} \left( \sqrt{T}\sqrt{\sum_{t=0}^{T-1} \norm{B_t Y_t \tilde{\alpha}_t}^2} + G k \sqrt{\frac{p T}{\lambda} \log(1 + M^2 T/\lambda)} \right) \:.
\end{align*}
Taking expectations and using Jensen's inequality
combined with Proposition~\ref{prop:online_newton_bounded_case}:
\begin{align*}
    &\E\left[\sum_{t=0}^{T-1} \norm{x_t^a}^2 - \norm{x_t^c}^2\right] \\
    &\leq k B_x^2 + \frac{2B_x M^2 D \gamma}{(1-\rho)^2} + \frac{2B_x\gamma}{1-\rho} \left( \sqrt{T}\sqrt{\sum_{t=0}^{T-1} \E\norm{B_t Y_t \tilde{\alpha}_t}^2} + G k \sqrt{\frac{p T}{\lambda} \log(1 + M^2 T/\lambda)} \right) \\
    &\leq k B_x^2 + \frac{2B_x M^2 D \gamma}{(1-\rho)^2} + \frac{2B_x\gamma}{1-\rho} \sqrt{T} \sqrt{4D^2(\lambda + M^4) + p G^2 \log(1 + M^4 T/\lambda) } \\
    &\qquad + \frac{2 B_x \gamma G k}{1-\rho} \sqrt{\frac{p T}{\lambda} \log(1 + M^2 T/\lambda)} \:.
\end{align*}

\end{proof}
\section{From Stability to Incremental Stability}
\label{sec:appendix:stability}

In this section, we study the relationship between
stability and incremental stability and the consequences of this relationship
for control regret bounds. We first start with the definition of stability
we will consider here.

\begin{definition}
\label{def:e_iss}
Let $\beta, \gamma$ be positive and $\rho \in (0, 1)$.
The discrete-time dynamical system $f(x, t)$ is
called $(\beta, \rho, \gamma)$-\emph{exponentially-input-to-state-stable} (E-ISS) for 
an initial condition $x_0$ and a signal $u_t$ (which is possibly adapted to the history $\{x_s\}_{s \leq t}$)
if the trajectory
$x_{t+1} = f(x_t, t) + u_t$
satisfies for
all $t \geq 0$:
\begin{align}
    \norm{x_t} \leq \beta \rho^{t} \norm{x_0} + \gamma \sum_{k=0}^{t-1} \rho^{t-1-k} \norm{u_{k}} \:. \label{eq:e_iss_ineq}
\end{align}
A system is called $(\beta,\rho,\gamma)$-E-ISS if it
is $(\beta,\rho,\gamma)$-E-ISS for all initial conditions $x_0$ and signals $u_t$.
\end{definition}

The following proposition shows that 
Definition~\ref{def:e_iss} is satisfied by an exponentially stable system
with a well-behaved Lyapunov function.
It is analogous to how Proposition~\ref{prop:contraction_implies_e_delta_iss}
demonstrates that contraction implies E-$\delta$ISS.
\begin{proposition}
Consider a dynamical system $f(x, t)$ with $f(0, t) = 0$ for all $t$.
Suppose $Q(x, t)$ is a Lyapunov function satisfying for some positive $\mu, L, L_Q$ and $\rho \in (0, 1)$:
\begin{enumerate}
    \item $\mu \norm{x}^2 \leq Q(x, t) \leq L \norm{x}^2$ for all $x, t$.
    \item $Q(f(x, t), t+1) \leq \rho Q(x, t)$ for all $x, t$.
    \item $x \mapsto \nabla Q(x, t)$ is $L_Q$-Lipschitz for all $t$. 
\end{enumerate}
Then the system $f(x, t)$ is $(\sqrt{L/\mu}, \sqrt{\rho}, L_Q/(2\mu))$-E-ISS.
\end{proposition}
\begin{proof}
Fix any $x, t$. We have:
\begin{align*}
    \nabla V(x, t) = \frac{1}{2 \sqrt{Q(x, t)}} \nabla Q(x, t) \:.
\end{align*}
Hence since zero is a local minima of the function $x \mapsto Q(x, t)$,
\begin{align*}
    \norm{ \nabla V(x, t) } = \frac{1}{2\sqrt{Q(x, t)}} \norm{\nabla Q(x, t)} \leq \frac{1}{2\sqrt{\mu}\norm{x}} L_Q \norm{x} = \frac{L_Q}{2\sqrt{\mu}} \:.
\end{align*}
Therefore by Taylor's theorem:
\begin{align*}
    \abs{V(f(x, t) + u, t+1) - V(f(x, t), t+1)} \leq \frac{L_Q}{2\sqrt{\mu}} \norm{u} \:.
\end{align*}
Hence:
\begin{align*}
    V(f(x, t) + u, t+1) &\leq V(f(x, t), t+1) + \frac{L_Q}{2\sqrt{\mu}} \norm{u} 
    \leq \sqrt{\rho} V(x, t) +  \frac{L_Q}{2\sqrt{\mu}} \norm{u} \:.
\end{align*}
Now consider the trajectory
\begin{align*}
    x_{t+1} = f(x_t, t) + u_t \:.
\end{align*}
By the inequality above, we have that:
\begin{align*}
    V(x_{t+1}, t+1) \leq \sqrt{\rho} V(x_t, t) + \frac{L_Q}{2\sqrt{\mu}}\norm{u_t} \:.
\end{align*}
Unrolling this recursion,
\begin{align*}
    \sqrt{\mu} \norm{x_t} &\leq V(x_t, t) \leq \rho^{t/2} V(x_0, 0) + \frac{L_Q}{2\sqrt{\mu}} \sum_{k=0}^{t-1} \rho^{(t-k-1)/2} \norm{u_k} \\
    &\leq \sqrt{L} \rho^{t/2} \norm{x_0} + \frac{L_Q}{2\sqrt{\mu}} \sum_{k=0}^{t-1} \rho^{(t-k-1)/2} \norm{u_k} \:.
\end{align*}
Therefore:
\begin{align*}
    \norm{x_t} \leq \sqrt{\frac{L}{\mu}} \rho^{t/2} \norm{x_0} + \frac{L_Q}{2\mu} \sum_{k=0}^{t-1} \rho^{(t-k-1)/2} \norm{u_k} \:.
\end{align*}
\end{proof}

\subsection{Incremental Stability over a Restricted Set}
\label{sec:app:stability:incremental}

In this section, we give a set of sufficient conditions
under which an E-ISS system can also be considered an E-$\delta$ISS
system, when we restrict both the set of initial conditions 
and the admissible inputs.
The results in this section are inspired from the work of \cite{ruffer13convergent}, 
who show that convergent systems can be considered incrementally stable
when restricted to a compact set of initial conditions. Their analysis, however,
does not preserve rates, which we aim to do in this section.

We start off with a basic definition that quantifies the rate
of stability of a discrete-time stable matrix.
\begin{definition}[cf.\ \cite{mania19CE}]
A matrix $A \in \R^{n \times n}$ is $(C,\rho)$ discrete-time stable 
for some $C \geq 1$ and $\rho \in (0, 1)$ if
$\norm{A^t} \leq C \rho^{t}$ for all $t \geq 0$.
\end{definition}

The next proposition shows how we can upper bound the operator norm
of the product of perturbed discrete-time stable matrices.
\begin{proposition}
\label{prop:product_perturbation}
Let $A$ be a $(C,\rho)$ discrete-time stable matrix.
Let $\Delta_1, ..., \Delta_{t}$ be arbitrary perturbations.
We have that for all $t \geq 1$:
\begin{align*}
    \bignorm{\prod_{i=1}^{t} (A + \Delta_i)} \leq C \prod_{i=1}^{t} (\rho + C \norm{\Delta_i}) \:.
\end{align*}
\end{proposition}
\begin{proof}
This proof is inspired by Lemma 5 of \cite{mania19CE}.
The proof works by considering
all $2^t$ terms $\{T_k\}$ of the product on the left-hand side.
Suppose that a term $T_k$ has $\ell$ occurrences of $\Delta_i$ terms,
namely $\Delta_{i_1}, ..., \Delta_{i_\ell}$.
This means there are at most $\ell+1$ slots for the $t-\ell$ $A$'s to appear
consecutively. 
Then since $C \geq 1$, we can bound:
\begin{align*}
    \norm{T_k} \leq C^{\ell+1} \rho^{t-\ell} \norm{\Delta_{i_1}}\cdot ... \cdot \norm{\Delta_{i_\ell}} = C \cdot \rho^{t-\ell} (C\norm{\Delta_{i_1}}) \cdot ... \cdot (C\norm{\Delta_{i_\ell}}) \:.
\end{align*}
Now notice that each term of the form
$\rho^{t-\ell} (C\norm{\Delta_{i_1}}) \cdot ... \cdot (C\norm{\Delta_{i_\ell}})$
can be identified uniquely with a term
in the product $\prod_{i=1}^{t} (\rho + C \norm{\Delta_i})$.
The claim now follows.
\end{proof}

The next lemma is the main result of this section.

\begin{lemma}
\label{lemma:stability_to_incremental_stability}
Consider an autonomous system $f(x)$ with $f(0) = 0$.
Suppose that $f(x)$ is $(\beta,\rho,\gamma)$-E-ISS,
that the linearization $A_0 := \frac{\partial f}{\partial x}(0)$ is
a $(C, \zeta)$ discrete-time stable matrix, and that $\frac{\partial f}{\partial x}$ is
$L$-Lipschitz.
Define the system $g(x_t, t) := f(x_t) + w_t$,
which is the original dynamics $f(x)$ driven by the noise sequence $\{w_t\}$.
Choose any $\psi \in (0, 1-\zeta)$
and suppose that:
\begin{align*}
    \sup_{t \geq 0} \norm{w_t} \leq W := \frac{1-\rho}{CL\gamma}(1-\zeta-\psi) \:.
\end{align*}
Fix a $D > 0$.
Let $h(\psi, B) : (0, 1) \times \R_+ \rightarrow \R_+$ be a function which is monotonically 
increasing in its second argument.
Let $\calD_h(\psi, B)$ denote a family of admissible sequences defined as:
\begin{align}
    \calD_h(\psi, B) := \left\{ \{d_t\}_{t \geq 0} : \sup_{t \geq 0} \norm{d_t} \leq D \:, \:\: \sup_{t \geq 1} \max_{0 \leq k \leq t-1}  \left[ -(t-k) \psi + B \sum_{s=k}^{t-1} \norm{d_s} \right] \leq h(\psi, B) \right\} \:. \label{eq:admissible_inputs}
\end{align}
Then for any initial conditions $(x_0, y_0)$ satisfying $\norm{x_0} \leq B_0$, 
$\norm{y_0} \leq B_0$
and any sequence $\{d_t\} \in \calD_h(\psi/2, \frac{CL\gamma}{1-\rho})$,
we have that $g(x_t, t)$ is
$(\beta', \rho', \gamma')$-E-$\delta$ISS for $(x_0, y_0, \{d_t\})$ with:
\begin{align*}
    \beta' &= \gamma' = C \exp\left( \frac{CL\beta}{1-\rho} \left(\beta B_0 + \frac{\gamma(W+D)}{1-\rho}\right) + h\left(\psi/2, \frac{CL\gamma}{1-\rho}\right)\right) \:, \\
    \rho' &= e^{-\psi/2} \:.
\end{align*}
\end{lemma}
\begin{proof}
By E-ISS \eqref{eq:e_iss_ineq}, we have that for all $t \geq 0$,
for the dynamics $x_{t+1} = f(x_t) + w_t + d_t$:
\begin{align*}
    \norm{x_t} \leq \beta \rho^t \norm{x_0} + \gamma \sum_{s=0}^{t-1} \rho^{t-1-s} \norm{w_s + d_s} \:.
\end{align*}
In particular, this implies that for all $t \geq 0$:
\begin{align*}
    \norm{x_t} \leq \beta \norm{x_0} + \frac{\gamma(W+D)}{1-\rho} \:.
\end{align*}
Define $g_t(x) := f(x) + w_t + d_t$ and for $t \geq 1$:
\begin{align*}
    \Phi_t(x_0, d_0, ..., d_{t-1}) := (g_{t-1} \circ g_{t-2} \circ ... \circ g_0)(x_0) \:.
\end{align*}
Observe that $\frac{\partial g_t}{\partial x}(x) = \frac{\partial f}{\partial x}(x)$.
By the chain rule:
\begin{align*}
    \frac{\partial \Phi_t}{\partial x_0}(x_0, d_0, ..., d_{t-1}) &= \frac{\partial g_{t-1}}{\partial x}(x_{t-1}) \frac{\partial g_{t-2}}{\partial x}(x_{t-2}) \cdots \frac{\partial g_0}{\partial x}(x_0) \\
    &= \frac{\partial f}{\partial x}(x_{t-1}) \frac{\partial f}{\partial x}(x_{t-2}) \cdots \frac{\partial f}{\partial x}(x_0) \\
    &= \left(A_0 + \frac{\partial f}{\partial x}(x_{t-1}) - A_0\right) \left(A_0 + \frac{\partial f}{\partial x}(x_{t-2}) - A_0\right) \cdots \left(A_0 + \frac{\partial f}{\partial x}(x_0) - A_0\right) \:. 
\end{align*}
Define $\Delta_t := \frac{\partial f}{\partial x}(x_t) - A_0$.
By the assumption that $\frac{\partial f}{\partial x}$ is $L$-Lipschitz, we have that
$\norm{\Delta_t} \leq L \norm{x_t}$.
Therefore by Proposition~\ref{prop:product_perturbation}:
\begin{align*}
    \bignorm{ \frac{\partial \Phi_t}{\partial x_0}(x_0, d_0, ..., d_{t-1}) } &\leq C \exp\left( -t(1-\zeta) + CL\sum_{s=0}^{t-1} \norm{x_s}      \right) \\
    &\leq C \exp\left( -t(1-\zeta) + CL \sum_{s=0}^{t-1} \left(\beta \rho^s \norm{x_0} + \gamma \sum_{k=0}^{s-1} \rho^{s-1-k} (W + \norm{d_k}) \right) \right) \\
    &\leq C \exp\left( -t(1-\zeta) + \frac{CL\beta}{1-\rho}\norm{x_0} + \frac{CL \gamma W t}{1-\rho} + \frac{CL\gamma}{1-\rho} \sum_{s=0}^{t-1} \norm{d_s} \right) \\
    &\stackrel{(a)}{\leq} C \exp\left( - t \psi + \frac{CL\beta}{1-\rho}\norm{x_0} + \frac{CL\gamma}{1-\rho} \sum_{s=0}^{t-1} \norm{d_s} \right) \\
    &\stackrel{(b)}{\leq} C \exp\left( - t \psi/2 + \frac{CL\beta}{1-\rho}\norm{x_0} + h\left(\psi/2, \frac{CL\gamma}{1-\rho}\right) \right) \:,
\end{align*}
where (a) follows from our assumption on $W$
and (b) follows from the definition of $\calD_h$.
Now let us look at $\frac{\partial \Phi_t}{\partial d_k}(x_0, d_0, ..., d_{t-1})$
for some $0 \leq k \leq t-1$. Again by the chain rule:
\begin{align*}
    \frac{\partial \Phi_t}{\partial d_k}(x_0, d_0, ..., d_{t-1}) &=  \frac{\partial g_{t-1}}{\partial x}(x_{t-1}) \frac{\partial g_{t-2}}{\partial x}(x_{t-2}) \cdots \frac{\partial g_{k+1}}{\partial x}(x_{k+1}) \\
    &= \frac{\partial f}{\partial x}(x_{t-1}) \frac{\partial f}{\partial x}(x_{t-2}) \cdots \frac{\partial f}{\partial x}(x_{k+1}) \\
    &= \left(A_0 + \frac{\partial f}{\partial x}(x_{t-1}) - A_0\right) \left(A_0 + \frac{\partial f}{\partial x}(x_{t-2}) - A_0\right) \cdots \left(A_0 + \frac{\partial f}{\partial x}(x_{k+1}) - A_0\right) \:. 
\end{align*}
Using Proposition~\ref{prop:product_perturbation} again:
\begin{align*}
    &\bignorm{\frac{\partial \Phi_t}{\partial d_k}(x_0, d_0, ..., d_{t-1})} \leq C \exp\left( -(t-k-1)(1-\zeta) + CL \sum_{s=k+1}^{t-1} \norm{x_s} \right) \\
    &\leq C \exp\left( -(t-k-1)(1-\zeta) + CL \sum_{s=k+1}^{t-1} \left(\beta \rho^{s-(k+1)} \norm{x_{k+1}} + \gamma \sum_{\ell=0}^{s-(k+1)-1} \rho^{s-(k+1)-1-\ell} (W + \norm{d_{k+1+\ell}}) \right) \right) \\
    &\leq C \exp\left( -(t-k-1)(1-\zeta) + \frac{CL\beta}{1-\rho}\norm{x_{k+1}} + \frac{CL\gamma W(t-k-1)}{1-\rho} + \frac{CL\gamma}{1-\rho} \sum_{s=k+1}^{t-1} \norm{d_s} \right) \\
    &\leq C \exp\left( -(t-k-1)\psi + \frac{CL\beta}{1-\rho}\norm{x_{k+1}} + \frac{CL\gamma}{1-\rho} \sum_{s=k+1}^{t-1} \norm{d_s} \right) \\
    &\leq C \exp\left( -(t-k-1)\psi + \frac{CL\beta}{1-\rho}\left(\beta \norm{x_0} + \frac{\gamma(W+D)}{1-\rho}\right) + \frac{CL\gamma}{1-\rho} \sum_{s=k+1}^{t-1} \norm{d_s} \right) \\
    &\leq C \exp\left( -(t-k-1)\psi/2 + \frac{CL\beta}{1-\rho}\left(\beta \norm{x_0} + \frac{\gamma(W+D)}{1-\rho}\right) + h\left(\psi/2, \frac{CL\gamma}{1-\rho}\right)\right) \:.
\end{align*}
Now let $x_0, y_0$ be norm bounded by $B_0$.
Let $(\tilde{z}_0, \tilde{d}_0, ..., \tilde{d}_{t-1})$
be an element along the ray connecting
$(x_0, d_0, ..., d_{t-1})$ with $(y_0, 0, ..., 0)$.
Observe that $\norm{\tilde{z}_0} \leq B_0$
and furthermore $(\tilde{d}_0, ..., \tilde{d}_{t-1}, 0, 0, ...) \in \calD_{h}(\psi/2, \frac{CL\gamma}{1-\rho})$.
Therefore by Taylor's theorem,
\begin{align*}
    &\norm{\Phi_t(x_0, d_0, ..., d_{t-1}) - \Phi_t(y_0, 0, ..., 0)} \\
    &\leq \bignorm{\frac{\partial \Phi_t}{\partial x_0}(\tilde{z}_0, \tilde{d}_0, ..., \tilde{d}_{t-1})} \norm{x_0-y_0} + \sum_{s=0}^{t-1} \bignorm{\frac{\partial \Phi_t}{\partial d_s}(\tilde{z}_0, \tilde{d}_0, ..., \tilde{d}_{t-1})} \norm{d_s} \\
    &\leq C \exp\left( \frac{CL\beta}{1-\rho} \left(\beta B_0 + \frac{\gamma(W+D)}{1-\rho}\right) + g\left(\psi/2, \frac{CL\gamma}{1-\rho}\right)\right) \times \\
    &\qquad \left( e^{-(\psi/2) t} \norm{x_0 - y_0} + \sum_{s=0}^{t-1} e^{-(\psi/2) (t-s-1)} \norm{d_s} \right) \:.
\end{align*}
\end{proof}

\subsection{Admissibility Bounds for Least-Squares}
\label{sec:app:stability:admissibility}

In this section, we show that under a persistence of excitation assumption,
regularized least-squares for estimating the parameters admits an admissible
sequence \eqref{eq:admissible_inputs} with high probability.
The statistical model we consider is the following.
Let $\{M_t\}_{t \geq 1} \subseteq \R^{n \times p}$ be a sequence of matrix-valued covariates adapted to a filtration $\{\calF_t\}_{t \geq 1}$.
Let $\{w_t\}_{t \geq 1} \subseteq \R^n$ be a martingale difference sequence
adapted to $\{\calF_t\}_{t \geq 2}$.
Assume that for all $t$, $w_t$ is conditionally
a $\sigma$-sub-Gaussian random vector:
\begin{align*}
    \forall v \in \R^n \text{ s.t. } \norm{v} = 1, \:\: \E[ \exp(\lambda \ip{v}{w_t}) | \calF_t] \leq \exp\left(\frac{\lambda^2 \sigma^2}{2}\right) \text{ a.s.} \:.
\end{align*}
Let the vector-valued responses $\{y_t\}_{t \geq 1} \subseteq \R^n$ be
given by $y_t = M_t \alpha_\star + w_t$, for an
unknown $\alpha_\star \in \calC$ which we wish to recover.
Fix a $\lambda > 0$.
The estimator we will study is the projected regularized least-squares estimator:
\begin{align*}
    \overline{\alpha}_t &= \arg\min_{\alpha \in \R^p} \frac{1}{2} \sum_{k=1}^{t} \norm{M_t \alpha - y_t}^2 + \frac{\lambda}{2}\norm{\alpha}^2 \:, \\
    \hat{\alpha}_t &= \Pi_{\calC}[\overline{\alpha}_t] \:.
\end{align*}
The closed-form solution for $\overline{\alpha}_t$
is $\overline{\alpha}_t = \left( \sum_{k=1}^{t} M_k^\T M_k + \lambda I\right)^{-1} \sum_{k=1}^{t} M_k^\T y_t$.
The next lemma gives us a high probability bound on the estimation
error $\norm{\hat{\alpha}_t - \alpha_\star}$ under a persistence of excitation condition.
\begin{lemma}
\label{lemma:regression_error}
Let $\{M_t\}$, $\{\calF_t\}$, $\{w_t\}$, $\{y_t\}$, and $\{\hat{\alpha}_t\}$
be as defined previously.
Let $V_t := \sum_{k=1}^{t} M_k^\T M_k + V$,
with $V \in \R^{p \times p}$ a fixed positive definite matrix.
We have with probability at least $1-\delta$, for all $t \geq 1$:
\begin{align}
    \bignorm{\sum_{k=1}^{t} M_i^\T w_k}_{V_t}^2 \leq 2 \sigma^2 \log\left( \frac{1}{\delta} \frac{ \det(V_t)^{1/2}  }{\det(V)^{1/2} } \right) \:. \label{eq:self_normalized_inequality}
\end{align}
Now suppose furthermore that
almost surely for all $t \geq T_0$, the following
persistence of excitation condition holds for some $\mu > 0$:
\begin{align}
    \frac{1}{t} \sum_{k=1}^{t} M_k^\T M_k \succeq \mu I \:. \label{eq:persistence_excitation}
\end{align}
Suppose also that $\norm{M_t} \leq M$ a.s.\ for all $t \geq 1$.
Then with probability at least $1-\delta$, for all $t \geq T_0$:
\begin{align}
    \norm{\hat{\alpha}_t - \alpha_\star} \leq \frac{\sigma}{\sqrt{\lambda + \mu t}} \sqrt{3p\log\left( \frac{1}{\delta}\left(1 + \frac{tM^2}{\lambda}\right) \right)} + \frac{\lambda}{\lambda + \mu t} \norm{\alpha_\star} \:. \label{eq:parameter_recovery}
\end{align}
\end{lemma}
\begin{proof}
The inequality \eqref{eq:self_normalized_inequality} comes
from a straightforward modification of Theorem 3 and Corollary 1 in \cite{abbasiyadkori11regret}
for scalar-valued regression.
In particular, the super-martingale $P_t^\lambda$ in Lemma 1 is replaced with:
\begin{align*}
    P_t^\lambda = \exp\left( \sum_{k=1}^{t} \frac{\ip{\lambda}{M_k^\T w_k}}{\sigma^2} - \frac{1}{2} \norm{M_k \lambda}^2 \right) \:.
\end{align*}
The rest of the proof of Theorem 3 and Corollary 1 proceeds without modification.

Now we turn to \eqref{eq:parameter_recovery}.
We let $V = \lambda I$.
Then we have for any $t \geq 1$:
\begin{align*}
    \overline{\alpha}_t &= V_t^{-1} \sum_{k=1}^{t} M_k^\T (M_k \alpha_\star + w_k) = V_t^{-1} \sum_{k=1}^{t} M_k^\T w_k + V_t^{-1} \sum_{k=1}^{t} M_k^\T M_k \alpha_\star \\
    &= \alpha_\star + V_t^{-1} \sum_{k=1}^{t} M_k^\T w_k - \lambda V_t^{-1} \alpha_\star \:.
\end{align*}
Hence by the Pythagorean theorem:
\begin{align*}
    \norm{\hat{\alpha}_t - \alpha_\star} &\leq \norm{\overline{\alpha}_t - \alpha_\star} \leq \bignorm{V_t^{-1} \sum_{k=1}^{t} M_k^\T w_k} + \lambda \norm{V_t^{-1} \alpha_\star} \\
    &\leq \norm{V_t^{-1/2}} \bignorm{V_t^{-1/2} \sum_{k=1}^{t} M_k^\T w_k} + \lambda \norm{V_t^{-1} \alpha_\star} \\
    &= \norm{V_t^{-1/2}} \bignorm{\sum_{k=1}^{t} M_k^\T w_k}_{V_t^{-1}} + \lambda \norm{V_t^{-1} \alpha_\star} \:.
\end{align*}
Now for $t \geq T_0$, we know that by the persistence of excitation condition:
\begin{align*}
    V_t^{1/2} \succeq \sqrt{\lambda + \mu t} \cdot I \:.
\end{align*}
Hence we have $\norm{V_t^{-1/2}} \leq \frac{1}{\sqrt{\lambda + \mu t}}$. 
Now suppose we are on the event given by \eqref{eq:self_normalized_inequality}.
Then:
\begin{align*}
    \norm{\hat{\alpha}_t - \alpha_\star} &\leq \frac{1}{\sqrt{\lambda + \mu t}} \bignorm{  \sum_{k=1}^{t} M_k^\T w_k}_{V_t} + \frac{\lambda}{\lambda + \mu t} \norm{\alpha_\star} \\
    &\leq \frac{\sigma}{\sqrt{\lambda + \mu t}} \sqrt{3p\log\left( \frac{1}{\delta}\left(1 + \frac{tM^2}{\lambda}\right) \right)} + \frac{\lambda}{\lambda + \mu t} \norm{\alpha_\star} \:.
\end{align*}
\end{proof}

The next proposition is a technical result
which derives an upper bound on the
functional inverse of $t \mapsto \log(c_1 t)/t$.
\begin{proposition}[cf.\ Proposition F.4 of~\cite{krauth19lspi}]
\label{prop:log_t_over_t_inverse}
Fix positive constants $c_1, c_2$.
We have that for any 
\begin{align*}
    t \geq \max\left\{ e/c_1, 1.582 \frac{1}{c_2} \log(c_1/c_2) \right\} \:,
\end{align*}
the following inequality holds:
\begin{align*}
    \frac{\log(c_1 t)}{t} \leq c_2 \:.
\end{align*}
\end{proposition}
\begin{proof}
First, we observe that:
\begin{align*}
    \frac{\log(c_1 t)}{t} \leq c_2 \Longleftrightarrow \frac{\log(c_1 t)}{(c_1 t)} \leq \frac{c_2}{c_1} \:.
\end{align*}
Now we change variables $x \gets c_1 t$, and hence we have the equivalent problem:
\begin{align*}
    \frac{\log{x}}{x} \leq \frac{c_2}{c_1} \:.
\end{align*}
Let $f(x) := \log{x}/x$.
It is straightforward to check that $f'(x) \leq 0$ for all $x \geq e$ and hence the 
function $f(x)$ is decreasing whenever $x \geq e$.

\paragraph{Case $c_2/c_1 > 1/e$.}
In this setting,
$f(e) = 1/e < c_2/c_1$, so for any $x' \geq e$ we have $f(x') \leq c_2/c_1$.
Undoing our change of variables, it suffices to take $t \geq e/c_1$.

\paragraph{Case $c_2/c_1 \leq 1/e$.}
Now we assume $c_2/c_1 \leq 1/e$.
Then $f(x') \leq c_2/c_1$ for any $x' \geq x$ where
$x$ is solution to $f(x) = c_2/c_1$.
Hence it suffices to upper bound the solution $x$.
To do this, we write $x$ in terms of the secondary branch $W_{-1}$ of the Lambert $W$ function.
We claim that $x = \exp(-W_{-1}(-c_2/c_1))$. 
First we note that $-c_2/c_1 \geq -1/e$ by assumption, so $W_{-1}(-c_2/c_1)$ is well-defined.
Next, observe that:
\begin{align*}
    \frac{\log{x}}{x} = \frac{-W_{-1}(-c_2/c_1)}{\exp(-W_{-1}(-c_2/c_1))} = - W_{-1}(-c_2/c_1) e^{W_{-1}(-c_2/c_1)} = c_2/c_1 \:.
\end{align*}
It remains to lower bound $W_{-1}(-c_2/c_1)$.
From Theorem 3.2 of \cite{alzahrani18lambert}, for any $t \geq 0$ we have:
\begin{align}
    W_{-1}(-e^{-t-1}) > - \log(t+1) - t - \alpha \:, \:\: \alpha = 2 - \log(e-1) \:. \label{eq:lambert_bound}
\end{align}
Hence:
\begin{align*}
    W_{-1}(-c_2/c_1) &= W_{-1}(-\exp(\log(c_2/c_1))) = W_{-1}(-\exp(-\log(c_1/c_2))) \\
    &= W_{-1}(-\exp(-(\log(c_1/c_2) - 1) - 1)) \:.
\end{align*}
Since $\log(c_1/c_2) - 1 \geq 0$, we can apply \eqref{eq:lambert_bound} to bound:
\begin{align*}
    W_{-1}(-c_2/c_1) \geq -\log\log(c_1/c_2) - \log(c_1/c_2) + 1 - \alpha \:.
\end{align*}
Therefore:
\begin{align*}
    x &= \exp(-W_{-1}(c_2/c_1)) \leq \exp( \log\log(c_1/c_2) + \log(c_1/c_2) + \alpha - 1) \\
    &= e^{\alpha-1} \frac{c_1}{c_2} \log(c_1/c_2) \leq 1.582 \frac{c_1}{c_2} \log(c_1/c_2) \:.
\end{align*}
Now we undo our change of variables to conclude that the solution to
$\log(c_1 t)/t = c_2$ is upper bounded by
$t \leq 1.582 \frac{1}{c_2} \log(c_1/c_2)$.
\end{proof}

\begin{proposition}
\label{prop:least_squares_admissible}
Let $\{\hat{\alpha}_t\}$ be as defined above.
Suppose the persistence of excitation condition \eqref{eq:persistence_excitation} holds.
Let $d_t := M_t(\hat{\alpha}_t - \alpha_\star)$
and suppose that $\norm{M_t} \leq M$ a.s.\ for all $t$.
Let $M_+ := \max\{M, \sqrt{\lambda}\}$.
With probability at least $1-\delta$,
for all positive $B, \psi$, we have:
\begin{align*}
  \sup_{t \geq 1} \max_{0 \leq k \leq t-1} \left[ -(t-k) \psi + B \sum_{s=k}^{t-1} \norm{d_s} \right] \leq 4BM_+D \max\left\{ T_0, \frac{2\lambda}{\mu\psi} BM_+D, \frac{38\sigma^2 p}{\psi^2 \mu} \log\left( \frac{96 M_+^2 \sigma^2 p}{\delta \lambda \psi^2 \mu}\right)   \right\} \:.
\end{align*}
\end{proposition}
\begin{proof}
Assume that $M^2/\lambda \geq 1$ w.l.o.g.\ (otherwise
take $M \gets \max\{M,\sqrt{\lambda}\}$.
We want to compute a $t_0 \geq T_0$ such that for all $t \geq t_0$,
\begin{align}
    \frac{BM \sigma}{\sqrt{\lambda + \mu t}} \sqrt{3p\log\left( \frac{1}{\delta}\left(1 + \frac{tM^2}{\lambda}\right) \right)} + \frac{\lambda BMD}{\lambda + \mu t} \leq \psi/2 \:. \label{eq:small_enough_d_s}
\end{align}
It suffices to find a $t_0$ such that for all $t \geq t_0$,
both inequalities hold:
\begin{align*}
    \frac{\sigma}{\sqrt{\lambda + \mu t}} \sqrt{3p\log\left( \frac{1}{\delta}\left(1 + \frac{tM^2}{\lambda}\right) \right)} &\leq \psi/4 \:, \:\: \frac{\lambda BMD}{\lambda + \mu t} \leq \psi/4 \:.
\end{align*}
The second inequality is satisfied for 
\begin{align*}
    t_0 \geq \frac{4\lambda}{\mu \psi} BMD \:.
\end{align*}
The first inequality is more involved. It is sufficient to require:
\begin{align*}
    \frac{1}{t} \log\left( \frac{1}{\delta} + \frac{tM^2}{\delta\lambda} \right) \leq \frac{\psi^2 \mu}{48 \sigma^2 p}
\end{align*}
By the assumption that $M^2/\lambda \geq 1$, it suffices to require:
\begin{align*}
        \frac{1}{t} \log\left( \frac{2M^2}{\delta\lambda} t \right) \leq \frac{\psi^2 \mu}{48 \sigma^2 p}
\end{align*}
We are now in a position to invoke Proposition~\ref{prop:log_t_over_t_inverse}
with $c_1 = \frac{2M^2}{\delta \lambda}$
and $c_2 = \frac{\psi^2 \mu}{48 \sigma^2 p}$
The conclusion is that we can take:
\begin{align*}
    t_0 \geq \max\left\{ T_0, \frac{e \delta \lambda}{2M^2}, 1.582 \cdot \frac{48 \sigma^2 p}{\psi^2 \mu} \log\left( \frac{96 M^2 \sigma^2 p}{\delta \lambda \psi^2 \mu}\right) \right\}
\end{align*}
Since $M^2/\lambda \geq 1$ and $\delta \in (0, 1)$,
we have $e\delta\lambda / (2M^2) \leq e/2 \leq 2$.
Hence the final requirement for $t_0$ is:
\begin{align*}
    t_0 \geq \max\left\{ T_0, 2, \frac{4\lambda}{\mu\psi} BMD, \frac{76\sigma^2 p}{\psi^2 \mu} \log\left( \frac{96 M^2 \sigma^2 p}{\delta \lambda \psi^2 \mu}\right)   \right\} \:.
\end{align*}
With these bounds in place, we look at:
\begin{align*}
\sup_{t \geq 1} \max_{0 \leq k \leq t-1} \left[ -(t-k) \psi + B \sum_{s=k}^{t-1} \norm{d_s} \right] \:.
\end{align*}
First suppose that $t \leq t_0$, then we have the trivial bound:
\begin{align*}
    -(t-k) \psi + B \sum_{s=k}^{t-1} \leq 2B M D t_0 \:.
\end{align*}
Now suppose that $t \geq t_0$ but $k \leq t_0$.
Then:
\begin{align*}
    -(t-k)\psi + B \sum_{s=k}^{t-1} \norm{d_s} &= -(t_0-k)\psi + B \sum_{s=k}^{t_0-1} \norm{d_s} + \left[ -(t - t_0)\psi + B \sum_{s=t_0}^{t-1} \norm{d_s} \right]  \\
    &\leq 2 B M D t_0 + \max_{t_0 \leq k \leq t-1} \left[  -(t-k) \psi + B \sum_{s=k}^{t-1} \norm{d_s} \right] \:.
\end{align*}
Hence we can assume that $t_0 \leq k \leq t-1$.
Now assume the event described by \eqref{eq:parameter_recovery} holds.
Then for each $s \geq t_0$,
$B \norm{d_s} \leq \psi/2$ by \eqref{eq:small_enough_d_s},
and hence 
\begin{align*}
    \max_{t_0 \leq k \leq t-1} \left[  -(t-k) \psi + B \sum_{s=k}^{t-1} \norm{d_s} \right] \leq 0 \:.
\end{align*}
\end{proof}

\subsection{Regret Bounds from Stability}

We are now ready to combine the results from
Section~\ref{sec:app:stability:incremental}
and Section~\ref{sec:app:stability:admissibility} into a regret bound.
As is done in Section~\ref{sec:online_ls:results},
we focus on the system \eqref{eq:system}.
Unlike Section~\ref{sec:online_ls:results} however, 
the online parameter estimator we consider is based on regularized least-squares.
We will discuss the issues of using online convex optimization algorithms at 
the end of this section.

We consider the following estimator,
which starts with a fixed $\lambda > 0$ and an arbitrary $\hat{\alpha}_0 \in \calC$ and iterates:
\begin{subequations}
\label{eq:regularized_ls_update}
\begin{align}
    \varphi_t &= f(x_t, t) + B_t u_t - x_{t+1} \:, \\
    \hat{\alpha}_{t+1} &= \Pi_{\calC}\left[ \left( \sum_{k=0}^{t} Y_t^\T B_t^\T B_t Y_t + \lambda I \right)^{-1} \sum_{k=0}^{t} Y_t^\T B_t^\T \varphi_t \right] \:.
\end{align}
\end{subequations}
Observe that $\varphi_t = B_t Y_t \alpha - w_t$, which fits the statistical model
setup of Section~\ref{sec:app:stability:admissibility}.
Letting $V_t :=  \left( \sum_{k=0}^{t} Y_t^\T B_t^\T B_t Y_t + \lambda I \right)^{-1} \sum_{k=0}^{t}$
and $M_t := B_t Y_t$,
we note that by the Woodbury matrix identity
\begin{align*}
    V_{t+1}^{-1} = V_t^{-1} - V_t^{-1} M_{t+1}^\T (I + M_{t+1} V_t^{-1} M_{t+1}^\T )^{-1} M_{t+1} V_t^{-1} \:,
\end{align*}
and hence if $n \ll p$, the quantity $V_t^{-1}$ can be computed efficiently in an online manner.

The next proposition is a simple technical result which will allow us to estimate
the growth of admissible sequences.
\begin{proposition}
\label{prop:log_tail_sum}
Let $c_0, c_1$ be positive constants.
Fix any integers $s, t$ satisfying
$\max\{4, c_0/c_1\} \leq s \leq t$.
We have that:
\begin{align*}
    \sum_{i=s}^{t} \frac{\log(c_0 + c_1 i)}{i} \leq \log(2c_1)(\log(t) - \log(s-1)) + \frac{1}{2} (\log^2(t) - \log^2(s-1)) \:.
\end{align*}
\end{proposition}
\begin{proof}
Whenever $i \geq c_0/c_1$, we have that $c_0 + c_1 i \leq 2 c_i i$.
Hence:
\begin{align*}
    \sum_{i=s}^{t} \frac{\log(c_0 + c_1 i)}{i} \leq \sum_{i=s}^{t} \frac{\log(2c_1 i)}{i} = \log(2c_1) \sum_{i=s}^{t} \frac{1}{i} + \sum_{i=s}^{t} \frac{\log{i}}{i} \:.
\end{align*}
The function $x \mapsto \log{x}/x$ is monotonically decreasing whenever $x \geq e$.
Hence:
\begin{align*}
    \sum_{i=s}^{t} \frac{\log{i}}{i} \leq \int_{s-1}^{t} \frac{\log{x}}{x} \: dx = \frac{1}{2} (\log^2(t) - \log^2(s-1)) \:.
\end{align*}
Similarly:
\begin{align*}
    \sum_{i=s}^{t} \frac{1}{i} \leq \int_{s-1}^{t} \frac{1}{x} \: dx = \log(t) - \log(s-1) \:.
\end{align*}
\end{proof}

We are now in a position to state our main regret bound for E-ISS systems.
\begin{theorem}
\label{thm:stability_regret_bound}
Fix a constant $B_0 > 0$.
Consider the dynamics $f(x)$ with $f(0) = 0$, and suppose that
$f(x)$ is $(\beta,\rho,\gamma)$-E-ISS, 
that the linearization $\frac{\partial f}{\partial x}(0)$ is a $(C,\zeta)$ 
discrete-time stable matrix, and that $\frac{\partial f}{\partial x}$ is $L$-Lipschitz.
Choose any $\psi \in (0, 1-\zeta)$ and define
$W := \frac{1-\rho}{CL\gamma}(1-\zeta-\psi)$.
Consider the regularized least-squares parameter update rule \eqref{eq:regularized_ls_update}.
Suppose that $\sup_{x,t} \norm{B(x,t)} \leq M$ and
$\sup_{x,t} \norm{Y(x,t)} \leq M$.
With constant probability (say $9/10$),
for any initial condition $x_0$ satisfying $\norm{x_0} \leq B_0$
and noise sequence $\{w_t\}$ satisfying $\sup_{t} \norm{w_t} \leq W$,
we have that for all $T \geq 1$:
\begin{align*}
    \sum_{t=0}^{T-1} \norm{x_t^a}^2 - \norm{x_t^c}^2 \leq \exp\left(\mathrm{poly}\left(\frac{1}{1-\rho},\frac{1}{\psi}, \frac{1}{\mu}, \beta, \gamma, B_0, D, M, W, \lambda, \log(1/\lambda), p \right)\right) \sqrt{T}\log{T} \:.
\end{align*}
The explicit form of the leading constant is given in the proof.
\end{theorem}
\begin{proof}
First we establish state bounds on the algorithm $x_t^a$ 
and the comparator $x_t^c$.
Define $B_x := \beta B_0 + \frac{\gamma(W + 2DM^2)}{1-\rho}$.
By E-ISS \eqref{eq:e_iss_ineq},
\begin{align*}
    \norm{x_t^c} \leq \beta \rho^t \norm{x_0} + \gamma \sum_{k=0}^{t-1} \rho^{t-1-k} \norm{w_k} \leq \beta \norm{x_0} + \frac{\gamma W}{1-\rho} \leq B_x \:.
\end{align*}
Similarly:
\begin{align*}
    \norm{x_t^a} \leq \beta \rho^t \norm{x_0} + \gamma \sum_{k=0}^{t-1} \rho^{t-1-k} \norm{w_k + B_k Y_k \tilde{\alpha}_k} \leq \beta \norm{x_0} + \frac{\gamma(W + 2 D M^2)}{1-\rho} \leq B_x \:.
\end{align*}
Hence:
\begin{align*}
    \sum_{t=0}^{T-1} \norm{x_t^a}^2 - \norm{x_t^c}^2 \leq \sum_{t=0}^{T-1} (\norm{x_t^a} + \norm{x_t^c}) \norm{x_t^a - x_t^c} \leq 2 B_x \sum_{t=0}^{T-1} \norm{x_t^a - x_t^c} \:.
\end{align*}
We suppose that the event prescribed by \eqref{eq:parameter_recovery}
holds.
Since the noise $w_t$ is bounded by $W$ a.s., it is a $W$-sub-Gaussian random vector~(see e.g.,\ Chapter 2 of \cite{wainwright19book}).
Put $M_+ = \max\{M^2, \sqrt{\lambda}\}$ and define $h(\psi, B)$ as:
\begin{align*}
    h(\psi, B) := 4B M_+ D \max\left\{ T_0, \frac{2\lambda}{\mu\psi} BM_+D, \frac{38W p}{\psi^2 \mu} \log\left( \frac{96 M_+^2 W p}{\delta \lambda \psi^2 \mu}\right)   \right\} \:.
\end{align*}
Combining
Lemma~\ref{lemma:stability_to_incremental_stability} and
Proposition~\ref{prop:least_squares_admissible},
we have that 
$g(x_t, t) := f(x_t) + w_t$ is $(\beta', \rho', \gamma')$-E-$\delta$ISS for
initial conditions $(x_0, y_0)$ and signal $\{ B_t Y_t \tilde{\alpha}_t \})$
with constants:
\begin{align*}
    \beta' &= \gamma' = C \exp\left( \frac{CL\beta}{1-\rho} \left(\beta B_0 + \frac{\gamma(W+2DM^2)}{1-\rho}\right) + h\left(\psi/2, \frac{CL\gamma}{1-\rho}\right)\right) \:, \\
    \rho' &= e^{-\psi/2} \:.
\end{align*}
By E-$\delta$ISS \eqref{eq:e_delta_iss_ineq}:
\begin{align*}
    \sum_{t=0}^{T-1} \norm{x_t^a}^2 - \norm{x_t^c}^2 
    &\leq 2B_x\sum_{t=0}^{T-1} \norm{x_t^a - x_t^c} 
    \leq 2B_x\gamma' \sum_{t=0}^{T-1} \sum_{k=0}^{t-1} {\rho'}^{t-1-k} \norm{B_k Y_k \tilde{\alpha}_k} \\
    &\leq \frac{2 B_x \gamma'}{1-\rho'} \sum_{t=0}^{T-1} \norm{B_t Y_t \tilde{\alpha}_t} 
    \leq \frac{2 B_x \gamma'}{1-\rho'} \sqrt{T} \sqrt{ \sum_{t=0}^{T-1} \norm{B_t Y_t \tilde{\alpha}_t}^2 } \:.
\end{align*}
We now bound using Proposition~\ref{prop:log_tail_sum}:
\begin{align*}
    &\sum_{t=0}^{T-1} \norm{ B_t Y_t \tilde{\alpha}_t}^2 \leq M^4 \sum_{t=0}^{T-1} \norm{\tilde{\alpha}_t}^2 \leq 4 M^4 D^2 T_0 + M^4\sum_{t=T_0}^{T-1} \norm{\tilde{\alpha}_t}^2 \\
    &\leq 4M^4 D^2 T_0 + \frac{6M^4W p}{\mu} \sum_{t=T_0}^{T-1} \frac{1}{t} \log\left(\frac{1}{\delta} + \frac{tM^4}{\delta\lambda}\right) + 2\lambda^2 M^4 D^2 \sum_{t=T_0}^{T-1} \frac{1}{(\lambda + \mu t)^2} \\
    &\leq 4M^4 D^2 T_0^2 + \frac{6M^4W p}{\mu} \left( \log\left(\frac{2M^4}{\delta\lambda}\right) (\log(T-1) - \log(T_0 - 1)) + \frac{1}{2}(\log^2(T-1) - \log^2(T_0-1)) \right) \\
    &\qquad + \frac{2\lambda^2 M^4 D^2}{\mu} \left( \frac{1}{\lambda + \mu (T_0 - 1)} - \frac{1}{\lambda + \mu (T - 1)} \right)  \\
    &\leq 4 M^4 D^2 T_0^2 + \frac{6M^4W p}{\mu} \log\left(\frac{2M^4}{\delta\lambda}\right) \log{T} + \frac{3 M^4 W p}{\mu} \log^2{T} + \frac{2 \lambda M^4 D^2}{\mu} \:.
\end{align*}
The claim now follows by combining the previous inequalities.
\end{proof}

We conclude this section on a discussion regarding the 
admissibility of online convex optimization algorithms
with respect to \eqref{eq:admissible_inputs}.
In the context of adaptive control,
the sequence $\{d_t\}$ is given by
$d_t = \norm{B_t Y_t \tilde{\alpha}_t}$.
By Cauchy-Schwarz, we can bound:
\begin{align}
 -(t-k)\psi + B \sum_{s=k}^{t-1} \norm{d_s} \leq -(t-k)\psi + B \sqrt{t-k} \sqrt{\sum_{s=k}^{t-1} \norm{B_t Y_t \tilde{\alpha}_t}^2} \:. \label{eq:admissible_eq}
\end{align}
The term 
$\sum_{s=k}^{t-1} \norm{B_t Y_t \tilde{\alpha}_t}^2$
is closely related to the prediction regret of the online
convex optimization algorithm; in particular, we have
$\sum_{s=0}^{t-1} \E\norm{B_t Y_t \tilde{\alpha}_t}^2 \leq \mathsf{PredictionRegret}(T) = o(T)$.
The key difference, however, is that in order for \eqref{eq:admissible_inputs} to be controlled,
we need the \emph{tail regret}
$\sum_{s=k}^{t-1} \E\norm{B_t Y_t \tilde{\alpha}_t}^2 \leq o(T-k)$
for $k = o(T)$.
To the best of our knowledge, such a guarantee is not achieved by
the online algorithms we consider in this paper.
The tail regret is related to a stronger notion of regret
in the literature known as \emph{strongly adaptive regret} (SA-Regret)~\citep{jun17saregret}.
However, the best known bounds for SA-Regret
scale as $\sqrt{(T-k)\log{T}}$~\citep{jun17saregret}, which is not strong enough 
to ensure that \eqref{eq:admissible_eq} remains finite 
when $k = o(T)$ due to the presence of the $\log{T}$ term.
It remains open whether or not an online algorithm is
capable of producing admissible sequences with respect to
\eqref{eq:admissible_inputs} without requiring parameter convergence.

\end{document}